\documentclass[a4paper, 11pt]{article}
\usepackage[english]{babel}
\usepackage[T1]{fontenc}

\usepackage{blindtext}
\usepackage{scrextend}
\usepackage{amsmath}
\usepackage{amsthm}
\usepackage{graphicx}
\usepackage{subfigure}
\usepackage{xcolor}

\usepackage{centernot,cancel}
\usepackage{amsmath}
\usepackage{amssymb}
\usepackage{lineno,hyperref}
\usepackage{authblk}
\usepackage{dsfont}
\usepackage{placeins}

\newtheorem{thrm}{Theorem}
\newtheorem{prop}{Proposition}
\newtheorem {corol}{Corollary}

\newtheorem{rmk}{Remark}
\newtheorem*{ex}{Example}
\newtheorem{lm}{Lemma}

\newcommand{\E}{\mathbb{E}}

\newcommand{\R}{\mathbb{R}}
\newcommand{\N}{\mathbb{N}}

\newcommand{\PP}{\mathbb{P}}
\newcommand{\V}{\mathrm{Var}}

\newcommand{\be}{\begin{eqnarray}}
\newcommand{\ee}{\end{eqnarray}}
\newcommand{\beq}{\begin{eqnarray*}}
\newcommand{\eeq}{\end{eqnarray*}}
\DeclareMathOperator{\Tr}{Tr}

\title{Gaussian field on the symmetric group: prediction and learning}

\usepackage{times}
 \author[1]{F. Bachoc  \thanks{francois.bachoc@math.univ-toulouse.fr}}
\author[2]{B. Broto\thanks{baptiste.broto@cea.fr}}
\author[1]{F. Gamboa\thanks{fabrice.gamboa@math.univ-toulouse.fr}}
\author[1]{J-M. Loubes\thanks{jean-michel.loubes@math.univ-toulouse.fr}}

\affil[1]{Institut de Math\'ematiques de Toulouse, Universit\'e Paul Sabatier,
118 route de Narbonne, F-31062 Toulouse Cedex 9}
\affil[2]{CEA, LIST, Universit\'e Paris-Saclay, F-91120, Palaiseau, France}

\begin{document}
\maketitle

\begin{abstract}
In the framework of the supervised learning of a real function defined on an abstract space $\mathcal X$, Gaussian processes are widely used. The Euclidean case for $\mathcal{X}$ is well known and has been widely studied. In this paper, we explore the less classical case where  $\mathcal X$ is the non commutative finite group of permutations (namely the so-called symmetric group $S_N$). We provide an application to Gaussian process based optimization of Latin Hypercube Designs. We also extend our results to the case of partial rankings.
\end{abstract}

\section{Introduction} \label{s:intro}
The problem of ranking a set of items is a fundamental task in today's data driven world. Analysing observations which are not quantitative variables but rankings has been often studied in social sciences. It  has  also become a popular problem in statistical learning thanks to the generalization of the use of automatic recommendation systems. Rankings are labels that model an order over a finite set $E_N:=\{1,\dots,N\}$. Hence, an observation is a set of preferences between these $N$ points. It is thus a one to one relation $\sigma$ acting from $E_N$ onto $E_N$. In other words, $\sigma$ lies in the finite symmetric group $S_N$ of all permutations of $E_N$. More precisely, assume that we have a finite set $X=\{x_1,\cdots,x_N\}$ and we have to order the elements of $X$. A ranking on $X$ is a statement of the form
\begin{equation}\label{eq:ranking_def}
x_{i_1}\succ x_{i_2} \succ\cdots\succ x_{i_N},
\end{equation}
where all the $i_j, j=1\cdots,N$ are different. 
We can associate to this ranking the permutation $\sigma$ defined by $\sigma(i_k)=k$. Reversely, to a permutation $\sigma$, we can associate the following ranking
\begin{equation}
x_{\sigma^{-1}(1)}\succ x_{\sigma^{-1}(2)} \succ\cdots \succ x_{\sigma^{-1}(N)}.
\end{equation}
 We refer to the works of Douglas E. Critchlow (see for example \cite{critchlow_probability_1991,critchlow_rank_1992,critchlow_ranking_1993}) for an introduction to rankings, together with various results.


Our aim is to predict outputs corresponding to permutations inputs. For instance, the permutation input can correspond to an ordering of tasks, in applications. In a workflow management system, there may be a large number of tasks that may be done in different orders but are all necessary to achieve the goal.  Workflow prediction or optimization problems currently occur in fields such as grid computing \cite{yu2005cost}, and logistics \cite{christopher2016logistics}.

Another example of application is given by the maintenance of machines in a  supply line. Machines in a supply line need to be tuned or monitored in order to optimize the production of a good. The machines can be tuned in different orders, each corresponding to a permutation and these choices have an impact on the quality of the production of the goods, measured by a quantitative variable $Y$, for instance the amount of defects in the produced goods. Hence, the objective of the model will thus be to forecast the outcome of a specific order for the maintenance  of the machines in order to optimize the production.

Another interesting case of output corresponding to a permutation input is of the form $\max_{x \in X} f(\sigma,x)$, where $f$ is a function both acting on the permutation  $\sigma$ and on some external variable $x$. This output corresponds to a worst case for the performance or the cost given by the permutation $\sigma$. Classical examples of this kind of output are the max distance criterion for Latin Hypercube Designs \cite{mckay_comparison_1979,santner_design_2003} and the robust deviation for a tour in the robust traveling salesman problem \cite{montemanni2007robust}. In Section \ref{sec_LHD}, we discuss and address the example of the max distance criterion.\vskip .1in

In this paper, we will be in the framework of Gaussian processes indexed by $S_N$.  Actually, Gaussian process models rely on the definition of a covariance function that characterizes the correlations between values of the process at different observation points.
As the notion of similarity between data points is crucial, \textit{i.e.} close location inputs are likely to have similar target values,  covariance functions  (symmetric positive definite kernels) are the key ingredient in using Gaussian processes for prediction. Indeed, the covariance operator contains nearness or similarity informations. In order to obtain a satisfying model one needs to choose a covariance function (\textit{i.e.} a symmetric positive definite kernel) that respects the structure of the index space of the dataset.\vskip .1in

A large number of applications gave rise to recent researches on ranking including {\it ranking aggregation} \cite{korba_learning_2017}, clustering rankings (see \cite{clemencon_clustering_2011}) or kernels on rankings for supervised learning. Constructing kernels over the set of permutations has been studied following several different ways. In \cite{kondor_group_2008}, Kondor provides results about kernels in non-commutative finite groups and constructs {\it diffusion kernels} (which are positive definite) on $S_N$.  These diffusion kernels are based on a discrete notion of neighbourhood. Notice that the kernels considered therein are quite different from those considered in this paper. Furthermore, the diffusion kernels are not in general covariance functions because of their tricky dependency on permutations. The recent reference \cite{jiao2017kendall} proves that the Kendall and Mallow's kernels are positive definite.  Further, \cite{mania_kernel_2016} extends this study characterizing both the feature spaces and the spectral properties associated with these two kernels. A real data set \cite{data_euro} on rankings is studied in \cite{mania_kernel_2016}. The authors used a kernel regression to predict the age of a participant with his/her order of preference of six sources of news regarding scientific developments: TV, radio, newspapers and magazines, scientific magazines, the internet, school/university.   \vskip .1in

There are applications where not all of the items in \eqref{eq:ranking_def} are ranked. Rather, a partial ranking is given (see for example the "sushi" dataset available at \linebreak {\verb http://www.kamishima.net } or movie datasets). The books \cite{critchlow2012metric} and \cite{marden_analyzing_2014} provide metrics on partial rankings and the papers \cite{kondor2010ranking} and \cite{jiao2017kendall} provide kernels on partial rankings and deal with the complexity reduction of their computation. \vskip .1in

The goal in this paper is threefold: first we define Gaussian processes indexed by $S_N$ by providing a wide class of covariance kernels. We generalize previous results on the Mallow's kernel (see \cite{jiao2017kendall}). Second, we consider the Kriging models (see for instance \cite{stein99interpolation}) that consist in inferring the values of a Gaussian random field given observations at a finite set of observation points. Here, the observations points are permutations. We study the asymptotic properties of the maximum likelihood estimator of the parameters of the covariance function.
We  also prove the asymptotic accuracy 
of the Kriging prediction under the estimated covariance parameters. Further, we provide simulations that illustrate the very good performances of the proposed kernels. 
Finally, we provide an application to Gaussian process based optimization of Latin Hypercube Designs.
Last, we show that the Gaussian process framework may be adapted to the cases of learning with partially observed rankings. We define a class of covariance kernels on partial rankings, for which we show how to reduce the computation complexity. In simulations, we show that our suggested kernels yield more efficient Gaussian process predictions than the kernels given in \cite{jiao2017kendall}.\vskip .1in

The paper falls into the following parts. In Section \ref{s:kernrank}, we recall some facts on  $S_N$  and  provide some covariance kernels on this set. Asymptotic results on the estimation of the covariance function are presented in Section \ref{s:GPrank}. Section \ref{s:GPrank} also contains an application to the optimization of Latin Hypercube Designs. Section \ref{s:kernpartial} provides new covariance kernels for partial rankings with a comparison with the ones given in \cite{jiao2017kendall} in a numerical experiment.  Section \ref{conclu} concludes the paper. The proofs are  all postponed to the appendix.

\section{Covariance model for  rankings} \label{s:kernrank} 

Recall that we define $S_N$ as the set of all permutations on $E_N:=\{1,\dots,N\}$. An element $\sigma$ of $S_N$ is a bijection from $E_N$ to $E_N$. We aim at constructing kernels, or covariance functions, on $S_N$. We will base these kernels on the three following distances on $S_N$ (see \cite{diaconis_group_1988}). For any permutations $\pi$ and $\sigma$ of $S_N$,
\begin{itemize}
\item The Kendall's tau distance is defined by
\begin{equation}
d_\tau(\pi,\sigma):=\sum_{\substack{i,j=1,...,N\\i<j}}\left(\mathds{1}_{\sigma(i)>\sigma(j),\;\pi(i)<\pi(j)}+\mathds{1}_{\sigma(i)<\sigma(j),\;\pi(i)>\pi(j)}\right).
\label{eq:ken}
\end{equation}
This distance counts the number of pairs on which the permutations disagree in ranking.
\item The Hamming distance is defined by
\begin{equation}
d_H(\pi,\sigma):=\sum_{i=1}^N\mathds{1}_{\pi(i)\neq \sigma(i)}.
\label{eq:ham}
\end{equation}
\item The Spearman's footrule distance is defined by
\begin{equation}
d_S(\pi,\sigma):=\sum_{i=1}^N|\pi(i)-\sigma(i)|.
\label{eq:spear}
\end{equation}
\end{itemize}
These three distances are right-invariant. That is, for all $\pi,\; \sigma,\;\tau \in S_N$, $d(\pi,\sigma)=d(\pi \tau,\sigma \tau)$. Other right-invariant distances are discussed in \cite{diaconis_group_1988}.

We aim at defining a Gaussian process indexed by permutations. Notice that, generally speaking, using the abstract Kolmogorov construction (see for example \cite{dacunha2012probability} Chapter 0), the law of a Gaussian random process $(Y_x)_{x\in E}$ indexed by an abstract set $E$ is entirely characterized by its mean and covariance functions
$$
 M : x \mapsto \E (Y_x)
 $$
and 
$$K : (x,y) \mapsto {\rm Cov}(Y_x, Y_y).
$$ 
Of course, here the framework is much simpler as $S_N$ is finite ($|S_N|=N!$),  and the Gaussian distribution is obviously  completely determined by its mean and covariance matrix.
Hence, if we assume that the process is centered, we only have to build a covariance function on $S_N$.
First, we recall the definition of a positive definite kernel on an abstract space $E$. A symmetric map $K:E\times E \rightarrow \R$ is called a {\it positive definite kernel} if for all $n\in \N$ and for all $(x_1,\cdots,x_n)\in E^n$, the matrix $(K(x_i,x_j))_{i,j}$ is positive semi-definite. In this paper, we say that $K$ is a {\it strictly positive definite kernel} if $K$ is symmetric and, for all $n \in \N$ and for all $(x_1,\cdots,x_n) \in E^n$ such that $x_i\neq x_j$ if $i\neq j$, the matrix $(K(x_i,x_j))_{i,j}$ is positive definite. \vskip .1in

These notions are particularly interesting for $S_N$ (and any finite set). Indeed, if $K$ is a strictly positive definite kernel, then for any function $f:S_N\rightarrow \R$, there exists $(a_\sigma)_{\sigma\in S_N}$ such that
\begin{equation}\label{universel}
f=\sum_{\sigma \in S_N}a_\sigma K(.,\sigma),
\end{equation}
and $K$ is of course an {\it universal kernel} (see \cite{micchelli2006universal}).

\begin{rmk}
Since $S_N$ is a finite discrete space, remark that the Reproducible Kernel Hilbert Space (RKHS) of a kernel $K$ is defined by the set of the functions of the form \eqref{universel}, and the universality of the kernel $K$ is equivalent to the equality of its RKHS with the set of the functions from $S_N$ to $\R$. This is, in turn, equivalent to the fact that $K$ is strictly positive definite.
\end{rmk}

We now provide two different parametric families  of  covariance kernels. The members of these families have the general form
\begin{equation}\label{K}
K_{\theta_1,\theta_2}(\sigma,\sigma'):=\theta_2\exp\left(-\theta_1d(\sigma,\sigma')\right),\;\; (\theta_1, \theta_2>0),
\end{equation}
and
\begin{equation}\label{eq:noyau}
K_{\theta_1,\theta_2,\theta_3}(\sigma,\sigma'):=\theta_2\exp\left(-\theta_1d(\sigma,\sigma')^{\theta_3}\right),\;\; (\theta_1, \;\theta_2>0,\;\theta_3\in [0,1]).
\end{equation}
Here, $d$ is one of the three distances defined in (\ref{eq:ken}),
(\ref{eq:ham}) and (\ref{eq:spear}). More precisely, for the Kendall's (resp. Hamming's and Spearman's footrule)  distance  let $K_{\theta_1,\theta_2(,\theta_3)}^\tau$  (resp. $K_{\theta_1,\theta_2(,\theta_3)}^H$, $K_{\theta_1,\theta_2(,\theta_3)}^S$) be the corresponding covariance function.  
For concision, sometimes we will write $K_{\theta_1,\theta_2(,\theta_3)}$ (resp. $d$) for one of these three kernels (resp. distances).

We show in the next proposition that $K_{\theta_1,\theta_2}$ is strictly positive definite.

\begin{prop}\label{defpos}
For all $\theta_1>0$ and $\theta_2>0$, $K_{\theta_1,\theta_2}^\tau$, $K_{\theta_1,\theta_2}^H$, $K_{\theta_1,\theta_2}^S$ are strictly positive definite kernels on $S_N$.
\end{prop}

\begin{rmk}
In \cite{mania_kernel_2016}, the strict positive definiteness of the Mallow's kernel, corresponding to $K_{\theta_1,\theta_2}^\tau$, is also shown. Our proof of Proposition \ref{defpos} seems more direct than the one given in \cite{mania_kernel_2016}.
\end{rmk}

We also have a similar result for $K_{\theta_1,\theta_2,\theta_3}$.
\begin{prop}\label{pos}
For all $\theta_1>0$, $\theta_2\geq 0$ and $\theta_3 \in [0,1]$, the maps $K_{\theta_1,\theta_2,\theta_3}^\tau$, $K_{\theta_1,\theta_2,\theta_3}^H$, $K_{\theta_1,\theta_2,\theta_3}^S$ are positive definite kernels on $S_N$.
\end{prop}

Propositions \ref{defpos} and \ref{pos} enable to define Gaussian processes indexed by permutations.

\begin{rmk}
The authors of \cite{anderes2017isotropic} define strictly positive definite kernels on graphs with Euclidean edges with two different metrics: the geodesic metric and the "resistance metric". The kernels are obtained by applying completely monotonous functions to these metrics (distances).
They provide different classes of such functions: the power exponential functions (which are considered in our work, see \eqref{eq:noyau}), the Mat\'{e}rn functions (with a smoothness parameter $0< \nu \leq 1 \slash 2$), the generalized Cauchy functions and the Dagum functions. One can show that Proposition \ref{pos} remains valid for all these kernels, by remarking as in \cite{anderes2017isotropic} that these kernels are based on completely monotonous functions. Some of the proofs of \cite{anderes2017isotropic} are based on techniques similar to the proof of Proposition \ref{pos}, using Schoenberg's theorems. 

We remark that the finite set of permutations $S_N$ is a graph, when two permutations $\sigma_1$ and $\sigma_2$ are connected if there exists a transposition $\pi$ such that  $\sigma_1=\sigma_2 \pi$. Hence, it is natural to ask if the results of \cite{anderes2017isotropic} can imply or extend some of the results in this paper. The answer however appears to be negative. Indeed, the distances considered in \cite{anderes2017isotropic} are the geodesic or the "resistance" distances, ans the distances in \eqref{eq:ken}, \eqref{eq:ham} and \eqref{eq:spear} do not fall into this category.

One could also consider the set of the permutations as a fully connected weighted graph, where the weight of the edge between $\sigma_1$ and $\sigma_2$ is $d(\sigma_1,\sigma_2)$, and where $d$ is $d_\tau$ or $d_H$ or $d_S$. Nevertheless, also with this graph, the results of \cite{anderes2017isotropic} do not apply, since the graphs addressed by this reference have a particular structure (finite sequential 1-sum of Euclidean cycles and trees).

We finally remark that \cite{anderes2017isotropic} constructs covariance functions not only on finite graphs, but between connected vertices. In contrast, the covariance functions constructed here are defined only on the finite set $S_N$.
\end{rmk}

\section{Gaussian fields on the symmetric group} \label{s:GPrank}
\subsection{Maximum likelihood} \label{s:asymtoresults}
Let us consider a Gaussian process $Y$ indexed by $\sigma \in S_N$, with zero mean and covariance function $K_*$. 
In a parametric setting, a classical assumption is that the covariance function $K_*$ belongs to some parametric set of the form
\begin{equation} \label{eq:parametric_model_modeling}
\{ K_{\theta} \; ;\; \theta \in \Theta \},
\end{equation}
where $\Theta \subset \mathbb{R}^p$ is given and for all $\theta \in \Theta$, $K_{\theta}$ is a covariance function. The parameter $\theta$ is generally called the covariance parameter.
In this framework, $K_*=K_{\theta^*}$ for some parameter $\theta^* \in \Theta$.

The parameter $\theta^*$ is estimated from noisy observations of the values of the Gaussian process at several inputs. Namely, to the observation point $\sigma_i$, we associate the observation $Y(\sigma_i)+\varepsilon_i$, for $i=1,\dots,n$, where $(\varepsilon_i)_i$ is an independent Gaussian white noise. Let us consider a sample of  random permutations $\Sigma=(\sigma_1,\sigma_2,\cdots,\sigma_n) \in S_N$. Assume that we observe $\Sigma$ and a random vector $y=(y_1,y_2,\cdots,y_n)^T$ defined by, for $i \leq N$,
\begin{equation} \label{eq:noisyGP}
y_i=Y(\sigma_i)+\varepsilon_i.
\end{equation}
Here, $Y$ is Gaussian process indexed by $S_N$ and independent of $\Sigma$. We assume that $Y$ is centered with covariance function $K_{\theta_1^*,\theta_2^*}$ (see \eqref{K} in Section~\ref{s:kernrank}) and that $(\varepsilon_i)_{i\leq n} \sim \mathcal{N}(0,\theta_3^* I_n)$. $Y$ is the unknown process to predict and $\varepsilon$ is an additive white noise. 
Notice that $\theta_3$ denotes here the variance of the nugget effect while it is a power in Section \ref{s:kernrank} (see \eqref{eq:noyau}). We keep the same name in order to use the compact notation $\theta$ for the parameter  of the model. 
The Gaussian process $Y$  is stationary in the sense that for all $\sigma_1,\cdots,\sigma_n\in S_N$ and for all $\tau \in S_N$, the finite-dimensional distribution of $Y$ at $\sigma_1,\cdots,\sigma_n$ is the same as the finite-dimensional distribution at $\sigma_1\tau,\cdots,\sigma_n\tau$.

Several techniques have been proposed for constructing an estimator\\ $\hat{\theta} = \hat{\theta}(\sigma_1,y_1,\cdots,\sigma_n,y_n)$ of $\theta^*:=(\theta_1^*,\theta_2^*,\theta_3^*)$: maximum likelihood estimation \cite{white1982maximum}, restricted maximum likelihood \cite{cressie93asymptotic}, leave-one-out estimation \cite{cressie1992statistics,bachoc2013cross}, leave-one-out log probability \cite{sundararajan_predictive_2001}... Here, we shall focus on the maximum likelihood method. It is widely used in practice and has received a lot of theoretical attention. Assume that $\Theta \subset \prod_{i=1}^3 [\theta_{i,min} , \theta_{i,max}]$ for some given $0 < \theta_{i,min} \leq \theta_{i,max} < \infty$ ($i=1,2,3$).  The maximum likelihood estimator is defined as  \begin{equation} \label{eq:MLdel} \widehat{\theta}_{ML}=\widehat{\theta}_n \in {\rm arg}\min_{\theta \in \Theta} L_\theta \end{equation} with
\begin{equation}
L_\theta:=\frac{1}{n}\ln(\det R_\theta)+\frac{1}{n}y^T R_{\theta}^{-1} y,
\end{equation}
where $R_\theta = [ K_{\theta_1,\theta_2}(\sigma_i , \sigma_j)+\theta_3\mathds{1}_{i=j} ]_{1 \leq i,j \leq n}$ is invertible for $\theta \in \Theta$ since $\theta_3>0$. 

\subsection{Asymptotic results} \label{s:asymtoresults2}

When considering the asymptotic behaviour of the maximum likelihood estimator, two different frameworks can be studied: fixed domain and increasing domain asymptotics \cite{stein99interpolation}. Under increasing-domain asymptotics, as $n \to \infty$, the observation points $\sigma_1,\cdots,\sigma_n $ are such that $\min_{i \neq j} d(\sigma_i ,\sigma_j) $ is lower bounded  and $d(\sigma_i , \sigma_j)$  becomes large with $|i-j|$,
(thus we can not keep $N$ fixed as $n\to +\infty$). Under fixed-domain asymptotics, the sequence (or triangular array) of observation points $(\sigma_1,\cdots,\sigma_n,\cdots)$ is dense in a fixed bounded subset. For a Gaussian field on $\R^d$, under increasing-domain asymptotics, the true covariance parameter $\theta^*$ can be estimated consistently by maximum likelihood. Furthermore, the maximum likelihood estimator is asymptotically normal \cite{MarMar1984,cressie93asymptotic,cressie96asymptotics,bachoc_asymptotic_2013}. Moreover, prediction performed using the estimated covariance parameter $\hat{\theta}_n$ is asymptotically as good as the one computed with $\theta^*$  as pointed out in \cite{bachoc_asymptotic_2013}. Finally, note that in the symmetric group, the fixed-domain framework can not be considered (contrary to the input space $\R^d$) since $S_N$ is a finite space.

We will consider hereafter the increasing-domain framework. We thus consider a number of observations $n$ that goes to infinity. Hence, the size $N$ of the permutations can not be fixed, as pointed out above. We thus let the size of the permutations be a function of $n$, that we write $N_n$, with $N_n \to \infty$ as $n\to \infty$. To summarize, we consider a sequence of Gaussian processes $Y_n$ on $S_{N_n}$, with $N_n\underset{n\rightarrow+\infty}{\longrightarrow}+\infty$ and where we consider a triangular array $(\sigma_i^{(n)})_{i\leq n} \subset S_{N_n}$ of observation points. However, for the sake of simplicity, we only write $Y$ and $(\sigma_i)_{i\leq n}$ and the dependency on $n$ is implicit. We observe values of the Gaussian process on the permutations $\Sigma = (\sigma_1,\cdots,\sigma_n)$, that are assumed to fulfill the following assumptions:
\bigskip

\underline{Condition 1:} For $d=d_\tau$ or $d=d_H$ or $d=d_S$, there exists $\beta>0$ such that $\forall i,j$, $d(\sigma_i,\sigma_j)\geq  |i-j|^{\beta}$.

\underline{Condition 2:} For $d=d_\tau$ or $d=d_H$ or $d=d_S$, there exists $c>0$ such that  $\forall i,\; d(\sigma_i,\sigma_{i+1})\leq c$.\\

Here, we recall that $d_\tau$, $d_H$ and $d_S$ are defined in Section \ref{s:kernrank}. Notice  that $\beta$ and $c$ are assumed to be independent on $n$.\bigskip

These conditions are natural under increasing-domain asymptotics. Indeed, Condition 1 provides asymptotic independence for pairs of observations with asymptotically distant indices. It allows to show   that the variance of $L_\theta$ and of its gradient converges to $0$. Condition 2 ensures the asymptotic discrimination of the covariance parameters (see Lemma \ref{lem:extra} in the appendix).
These conditions can be  ensured with particular choices of sampling schemes for $(\sigma_1,\cdots,\sigma_n)$ (using the distances previously discussed).

As an example consider the following setting.
We fix $k\in \N$. For $n\in \N, i\in [1:n]$, we choose $\sigma_i^{(n)}=\sigma_i=\tau_i c_i\in S_{k+n}$ (we have $N_n=k+n)$ with $\tau_i\in S_k\times id_{[k+1:n+k]}:=\{\sigma \in S_{n+k}|\;\sigma_{| [k+1:n+k]}=id\}$ a random permutation such that $(\tau_i)_i$ are independent (we do not make further assumptions on the law of $\tau_i$). Let $c_i=(i+k\;\;i+k-1\;\;\cdots\;\;1)$ the cycle defined by $c_i(1)=i+k$, $c_i(j)=j-1$ if $1<j\leq i+k$ and $c_i(j)=j$ if $j>i+k$. Then, $\sigma_i$ is a permutation such that $\sigma_i(1)=i+k$, $\sigma_i(j)$ is a random variable in $[2:k]$ if $1<j\leq k+1$, $\sigma_i(j)=j-1$ if $k+1<j \leq i+k$ and $\sigma_i(j)=j$ if $j>i+k$. A straightforward computation shows that the Conditions 1 and 2 are satisfied with $\beta=1$ and $c=1+k(k-1)\slash2$ for the Kendall's tau distance, $c=2+k$ for the Hamming distance, $c=2+k^2$ for the Spearman's footrule distance. Indeed, the three distances in $S_k$ are upper-bounded by $k(k-1)\slash2$, $k$ and $k^2$ respectively.

The following theorems give both the consistency and the asymptotic normality  of the estimator when the number of observations increases.
\begin{thrm}\label{consis} 
Let $\widehat{\theta}_{ML}$ be defined as in \eqref{eq:MLdel}, where the distance $d$ used to define the set $\{ K_\theta\; ; \; \theta \in \Theta\}$ is $d_\tau$, $d_H$ or $d_S$. Assume that Conditions 1 and 2 hold with the same choice of the distance $d$. Then,
\begin{equation}
\widehat{\theta}_{ML}\overset{\PP}{\underset{n\rightarrow+\infty}{\longrightarrow}}\theta^*.
\end{equation}
\end{thrm}


\begin{thrm}\label{gaussien}
Under the assumptions of Theorem \ref{consis}, let $M_{ML}$ be the $3 \times 3$ matrix defined by
\begin{equation}\label{eqgaussien}
(M_{ML})_{i,j}=\frac{1}{2n}\Tr\left(R_{\theta^*}^{-1}\frac{\partial R_{\theta^*}}{\partial \theta_i}  R_{\theta^*}^{-1}\frac{\partial R_{\theta^*}}{\partial \theta_j} \right).
\end{equation}
Then
\begin{equation} \label{eq:TCL}
\sqrt{n} M_{ML}^\frac{1}{2}\left(\widehat{\theta}_{ML}-\theta^* \right)\;\overset{\mathcal{L}}{\underset{n\rightarrow +\infty}{\longrightarrow}}\mathcal{N}(0,I_3).
\end{equation}
Furthermore,
\begin{equation}\label{22}
0<\liminf_{n\rightarrow \infty}\lambda_{\min}(M_{ML})\leq \limsup_{n\rightarrow \infty} \lambda_{\max}(M_{ML})<+\infty,
\end{equation}
where $\lambda_{\min}(M_{ML})$ (resp. $\lambda_{\max}(M_{ML})$) is the smallest (resp. largest) eigenvalue of $M_{ML}$.
\end{thrm}

Given the maximum likelihood estimator $\widehat{\theta}_n=\widehat{\theta}_{ML}$, the value $Y(\overline{\sigma}_n)$, for any input $\overline{\sigma}_n \in S_{N_n}$, can be forecasted  by plugging the estimated parameter in the conditional expectation expression for Gaussian processes. Hence  $Y(\overline{\sigma}_n)$ is predicted by  
\begin{equation} \label{eq:pred}
\hat{Y}_{\widehat{\theta}_n}(\overline{\sigma}_n)=r_{\widehat{\theta}_n}^T(\overline{\sigma}_n)R_{\widehat{\theta}_n}^{-1} y
\end{equation}
with
\[
r_{\widehat{\theta}_n}(\overline{\sigma}_n)=\left[\begin{array}{c} K_{\widehat{\theta}_n}(\overline{\sigma}_n,\sigma_1)\\ \vdots\\ K_{\widehat{\theta}_n}(\overline{\sigma}_n,\sigma_n) \end{array}\right].
\]
We point out that $\hat{Y}_{\widehat{\theta}_n}(\overline{\sigma}_n)$ is the conditional expectation of $Y(\overline{\sigma}_n)$ given $y_1,\cdots,y_n$, when assuming that $Y$ is a centered Gaussian process with covariance function $K_{\widehat{\theta}_n}$. 

The following theorem shows that the forecast with the estimated parameter behaves asymptotically as if the true covariance parameter were known.

\begin{thrm}\label{pred}
Under the assumptions of Theorem \ref{consis}, for any fixed sequence $(\overline{\sigma}_n)_{n\in \N}$, with $\overline{\sigma}_n\in S_{N_n}$ for $n\in \N$, we have
\begin{equation}
\left|\widehat{Y}_{\widehat{\theta}_{ML}}(\overline{\sigma}_n)-\widehat{Y}_{\theta^*}(\overline{\sigma}_n)\right|\overset{\PP}{\underset{n\rightarrow+\infty}{\longrightarrow}}0.
\end{equation}
\end{thrm}

\begin{rmk}
Theorem \ref{pred} does not imply that
\begin{equation}\label{eq_false}
    \max_{\sigma \in S_{N_n}}\left|\widehat{Y}_{\widehat{\theta}_{ML}}(\sigma)-\widehat{Y}_{\theta^*}(\sigma)\right|\overset{\PP}{\underset{n\rightarrow+\infty}{\longrightarrow}}0.
\end{equation}
Indeed, letting 
$
\overline{\sigma}_n \in \underset{\sigma \in S_{N_n}}{\rm argmax}\left|\widehat{Y}_{\widehat{\theta}_{ML}}(\sigma)-\widehat{Y}_{\theta^*}(\sigma)\right|,
$
\eqref{eq_false} is equivalent to 
$$
\left|\widehat{Y}_{\widehat{\theta}_{ML}}(\overline{\sigma}_n)-\widehat{Y}_{\theta^*}(\overline{\sigma}_n)\right|\overset{\PP}{\underset{n\rightarrow+\infty}{\longrightarrow}}0,
$$
but where $\overline{\sigma}_n$ is random. Here, Theorem \ref{pred} does not imply \eqref{eq_false} as it holds for deterministic sequences $(\overline{\sigma}_n)_{n\in \N}$. It would be interesting, in future work, to extend Theorem \ref{pred} to show \eqref{eq_false}.
\end{rmk}

The proofs of Theorems \ref{consis}, \ref{gaussien} and \ref{pred} are given in the appendix,  Sections \ref{section_proof_1}, \ref{section_proof_2} and \ref{section_proof_3} respectively. They are based on lemmas stated and proved in Section \ref{section_lemmas}. 
In \cite{bachoc_asymptotic_2013} and \cite{bachoc_gaussian_2017}, similar results for maximum likelihood are given for Gaussian fields indexed on $\R^d$ and on the set of all probability measures on $\R$ (see also \cite{bachoc2019gaussian}).
At the beginning of Appendix \ref{s:AB}, we also discuss the similarities and differences between the proofs of Theorems \ref{consis}, \ref{gaussien} and \ref{pred} and these given in \cite{bachoc_asymptotic_2013} and \cite{bachoc_gaussian_2017}.

\subsection{Numerical experiments}\label{s:appli}
As an illustration of Theorem \ref{consis}, we provide a numerical illustration showing that the maximum likelihood is consistent. 
We generated the observations as discussed  in Section \ref{s:GPrank} with $k=3$. We recall that $N_n=k+n$ and $\sigma_i=\tau_i(i+k\;i+k-1\;\cdots1)\in S_{k+n}$ where $\tau_i\in S_k\times id_{[k+1:k+n]}$ is a random permutation.

For each value of $n$, we estimate the probability $\PP(\|\widehat{\theta}_{n}-\theta^*\|>\varepsilon)$ using a Monte-Carlo method and a sample of 1000 values of $\mathds{1}_{\|\widehat{\theta}_{n}-\theta^*\|>\varepsilon}$. Figure \ref{figconsis} depicts these estimates for $\varepsilon=0.5$, $\theta^*=(0.1,0.8,0.3)$ and $\Theta=[0.02,2] \times [0.3,2] \times [0.1,1]$.

\begin{figure}
\center
\includegraphics[height=4cm,width=4cm]{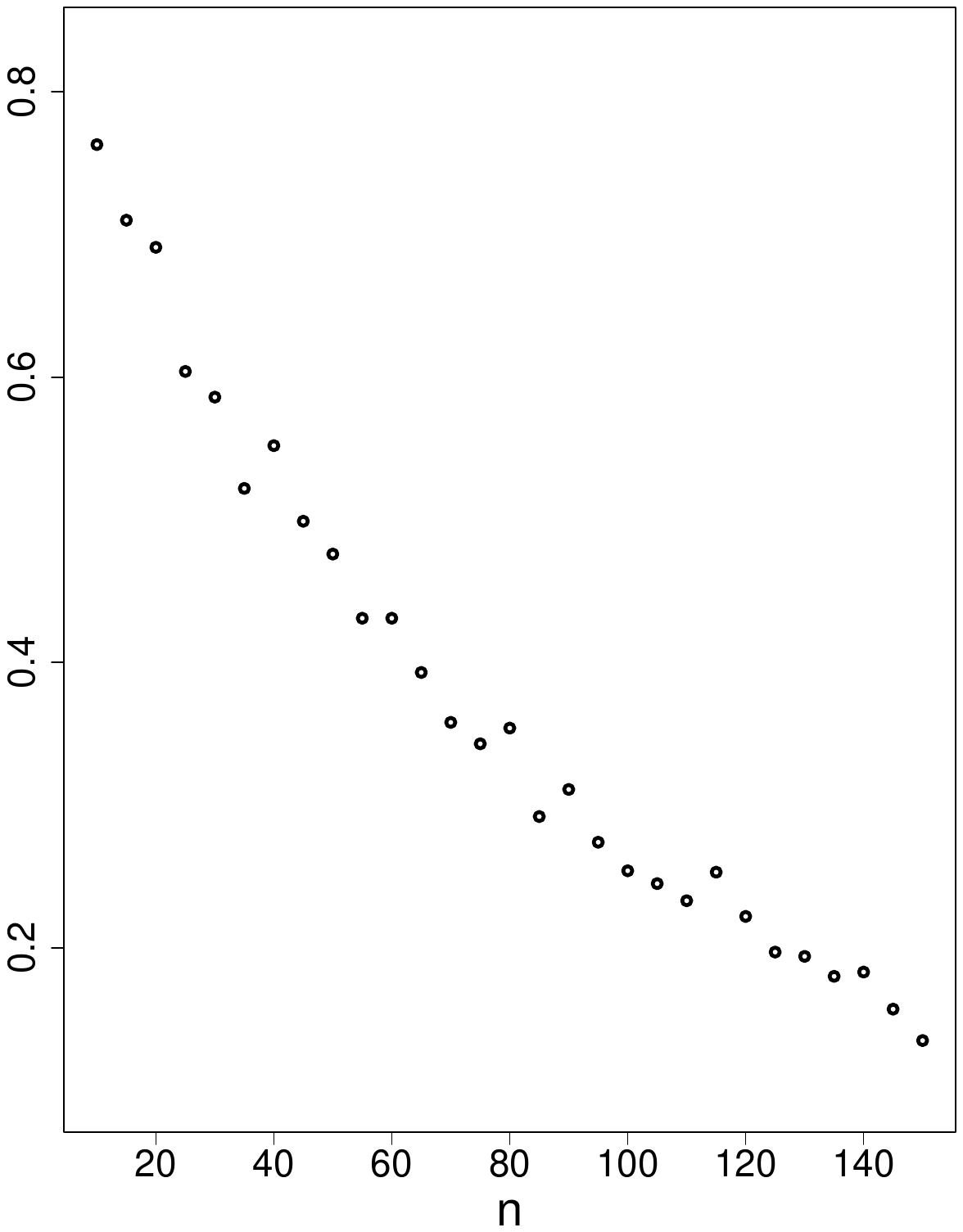}\includegraphics[height=4cm,width=4cm]{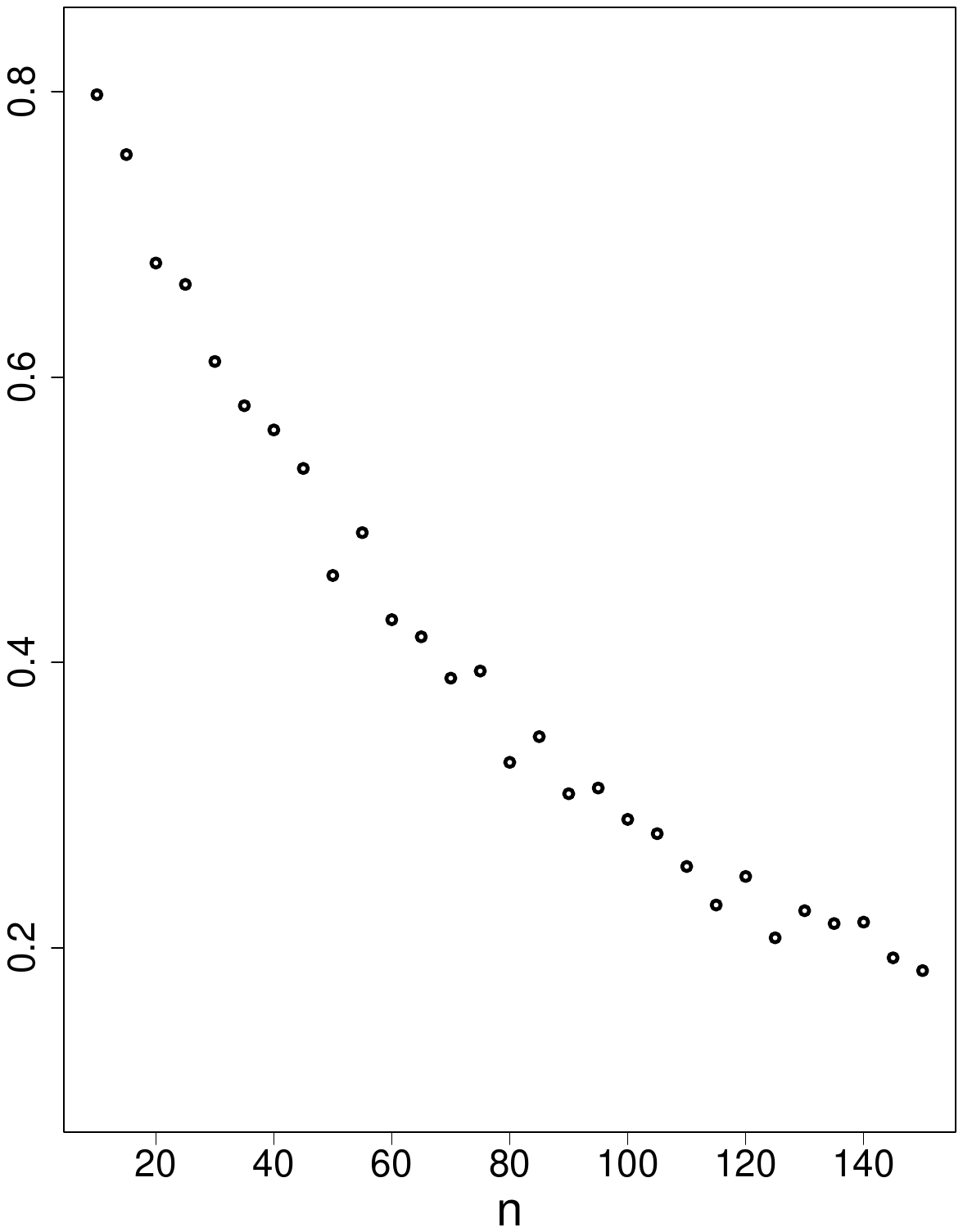}
\includegraphics[height=4cm,width=4cm]{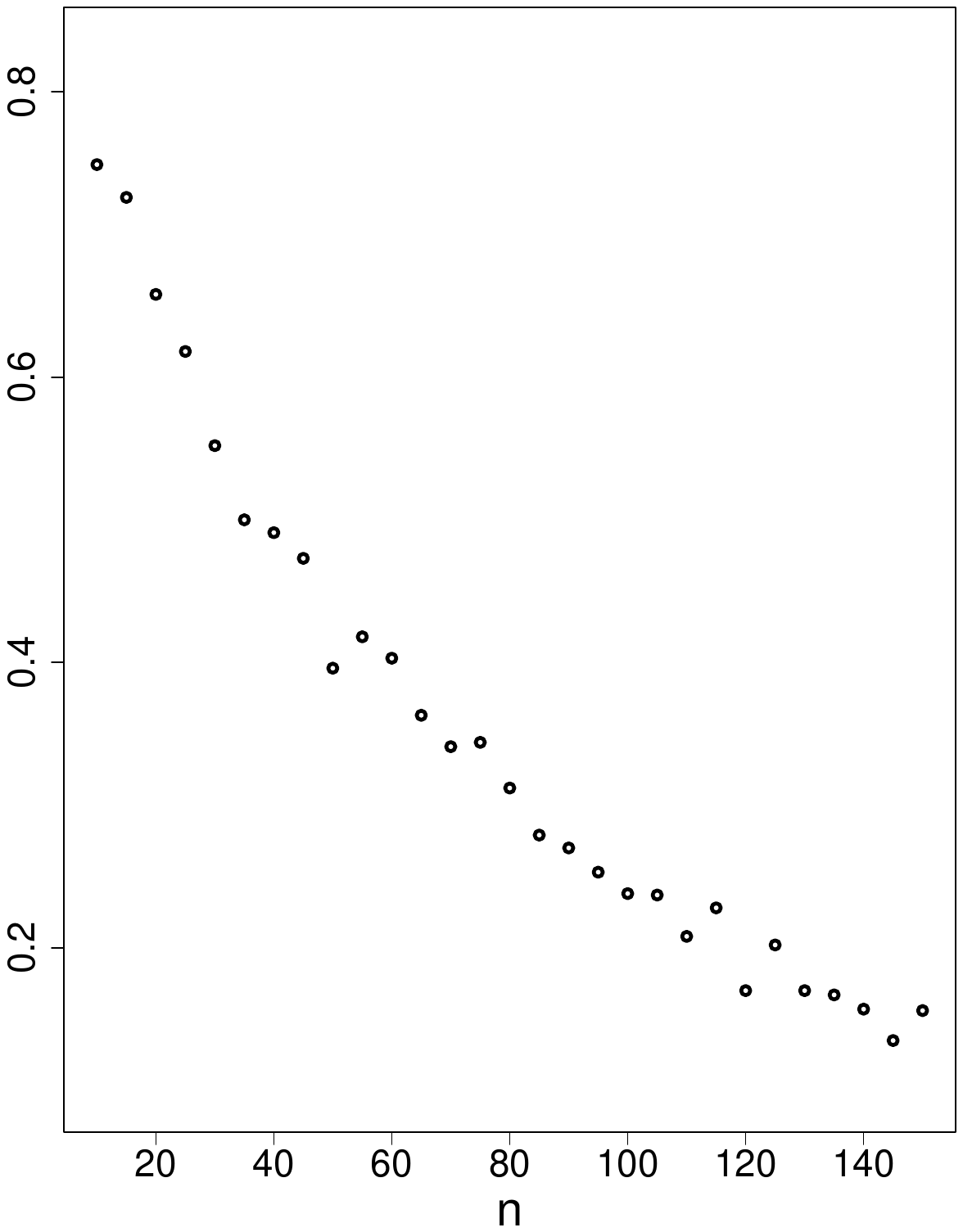}
\caption{Monte Carlo estimates of $\PP(\|\widehat{\theta}_{n}-\theta^*\|>0.5)$ for different values of $n$, the number of observations, with $\theta^*=(0.1,0.8,0.3)$ and Kendall's tau distance, the Hamming distance and the Spearman's footrule distance from left to right.}
\label{figconsis}
\end{figure}

\begin{figure}	
	\centering
	\begin{subfigure}
		\centering
		\includegraphics[height=3.6cm,width=3.6cm]{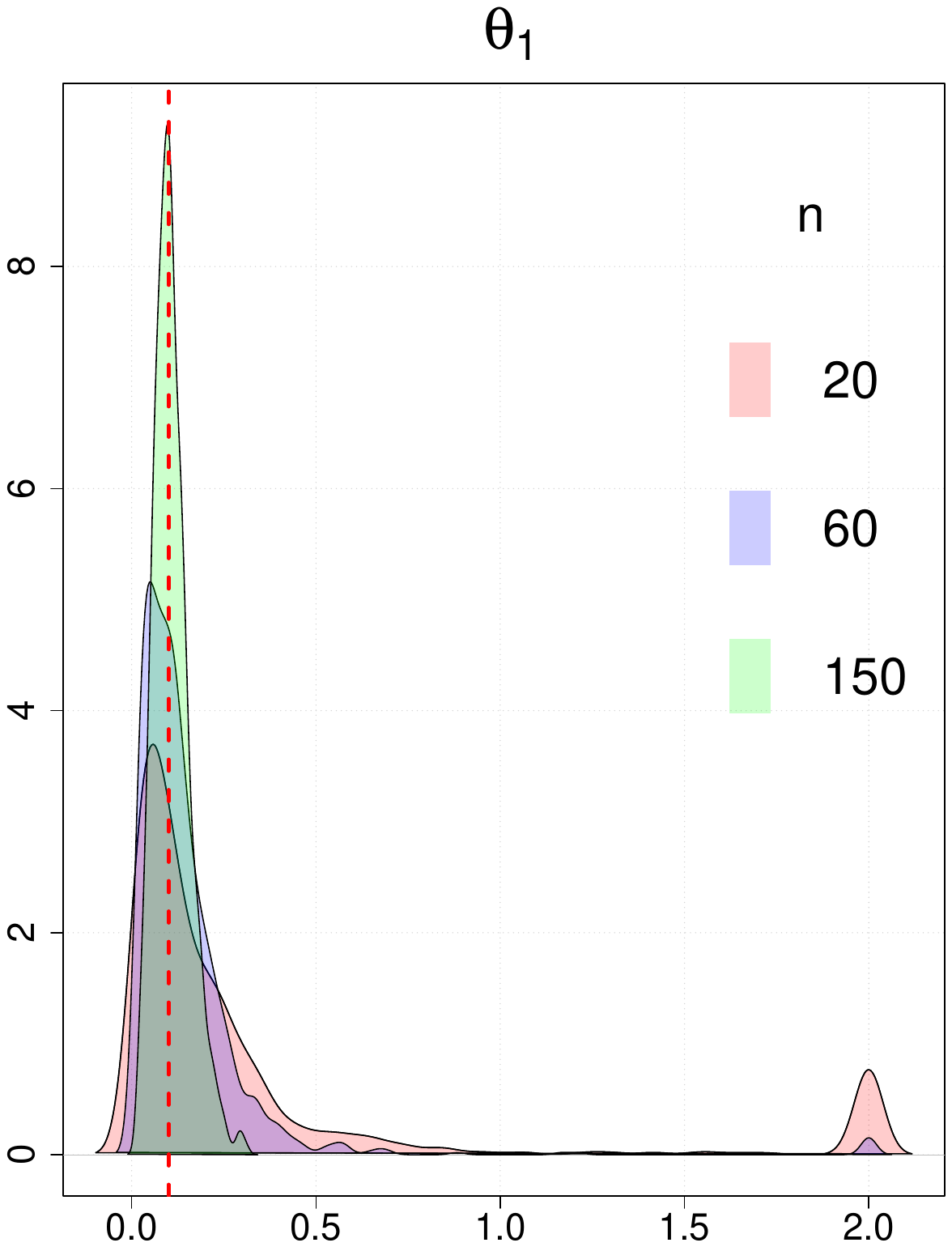}
		\includegraphics[height=3.6cm,width=3.6cm]{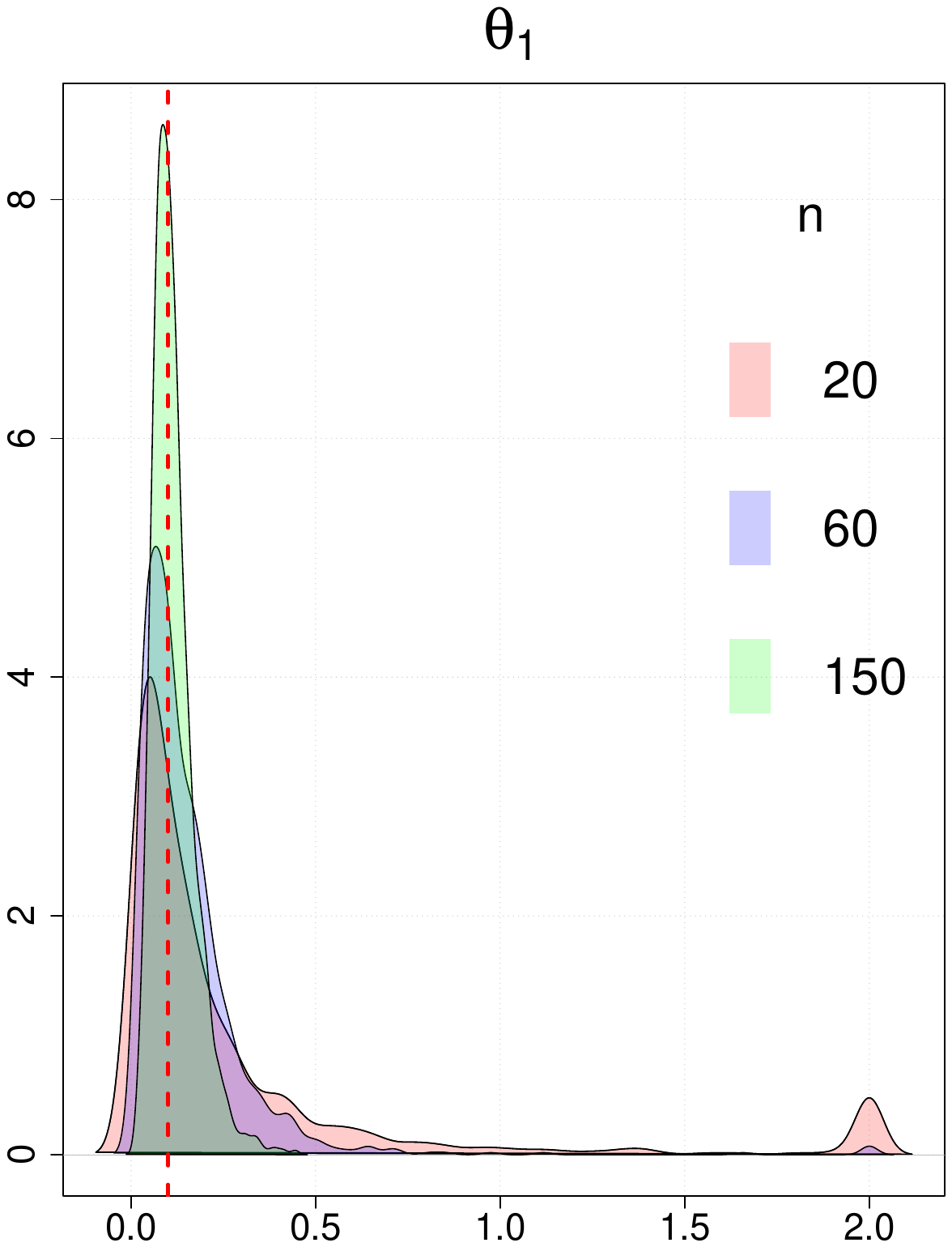}
		\includegraphics[height=3.6cm,width=3.6cm]{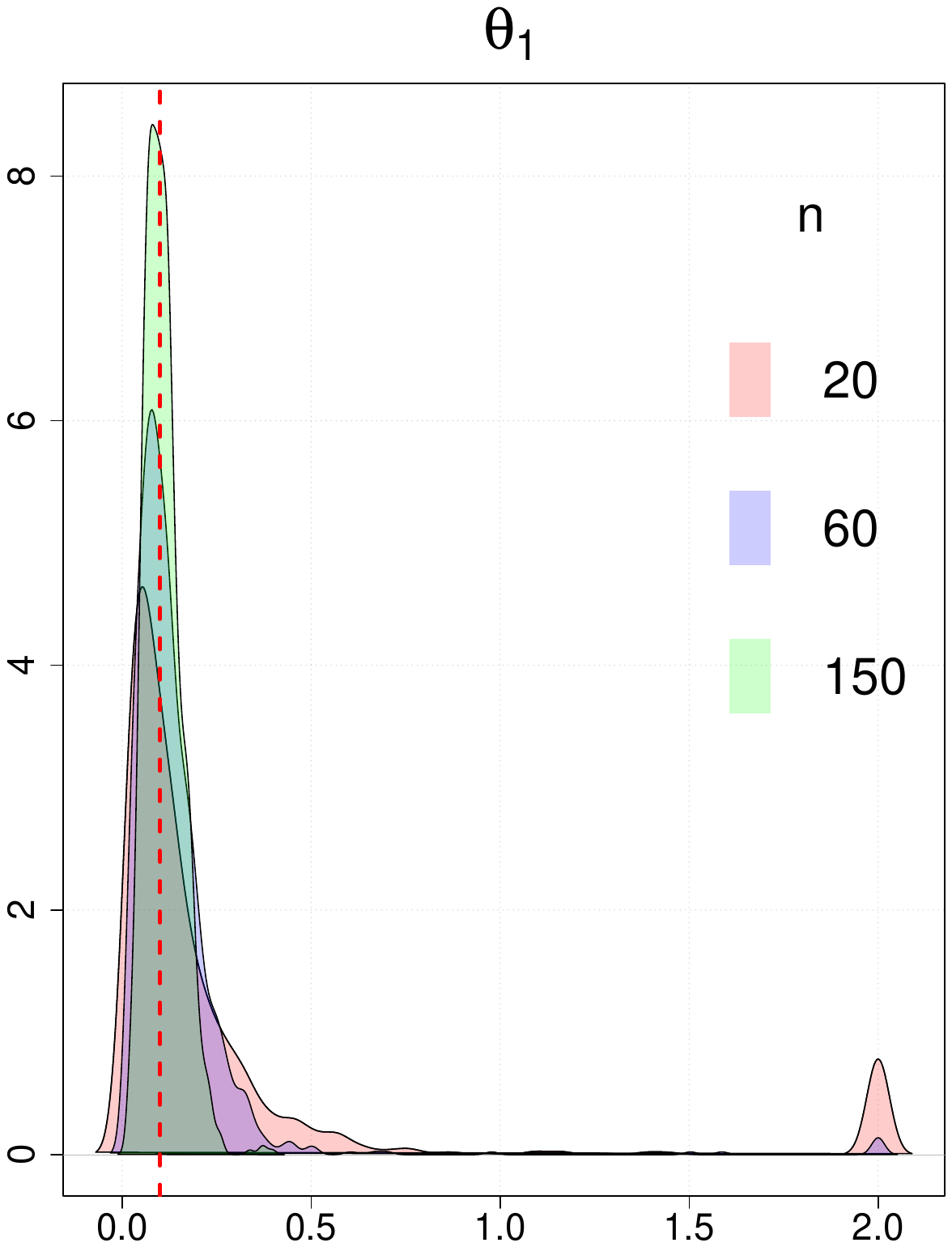}		
	\end{subfigure}
	\quad
	
	\begin{subfigure}
		\centering
		\includegraphics[height=3.6cm,width=3.6cm]{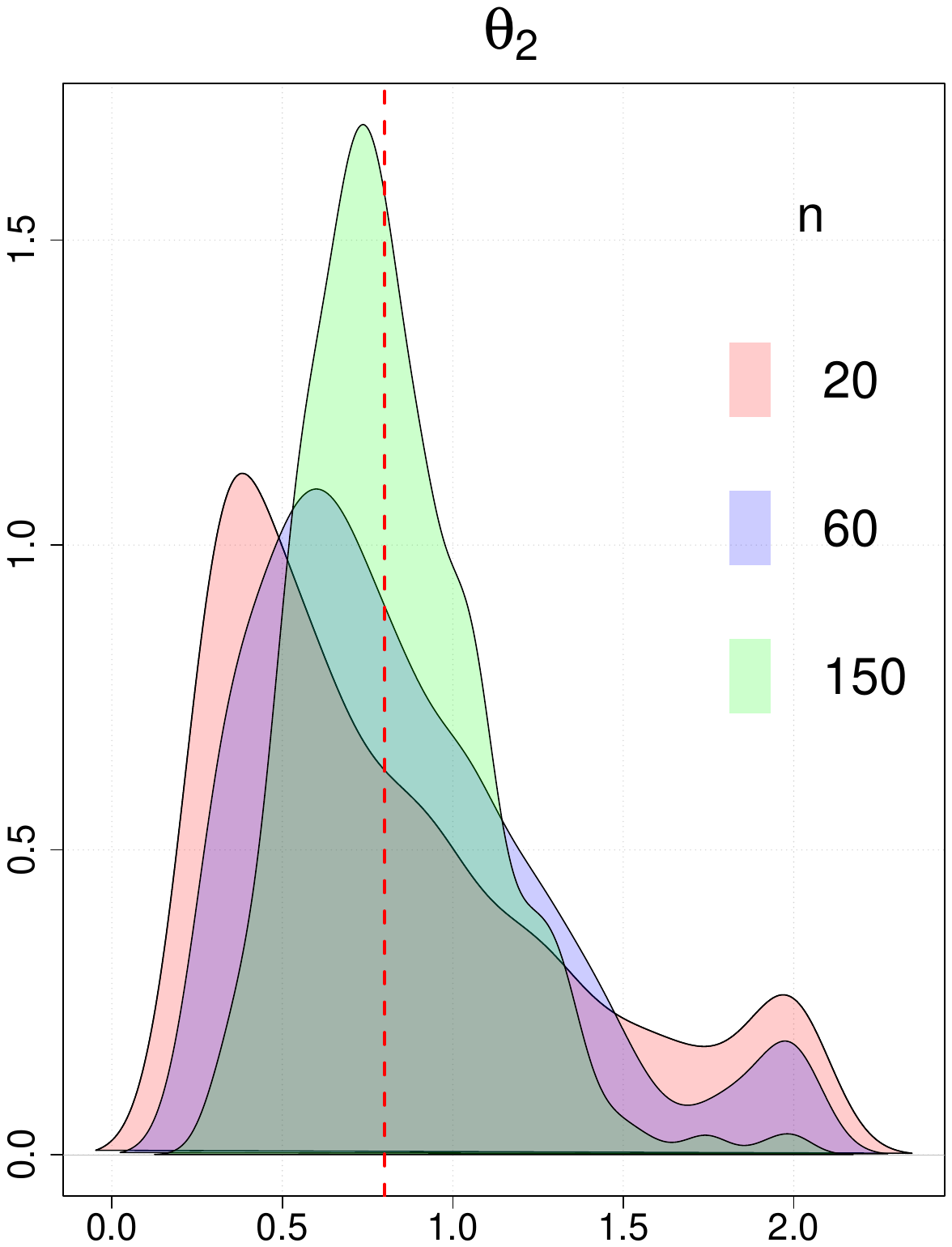}
		\includegraphics[height=3.6cm,width=3.6cm]{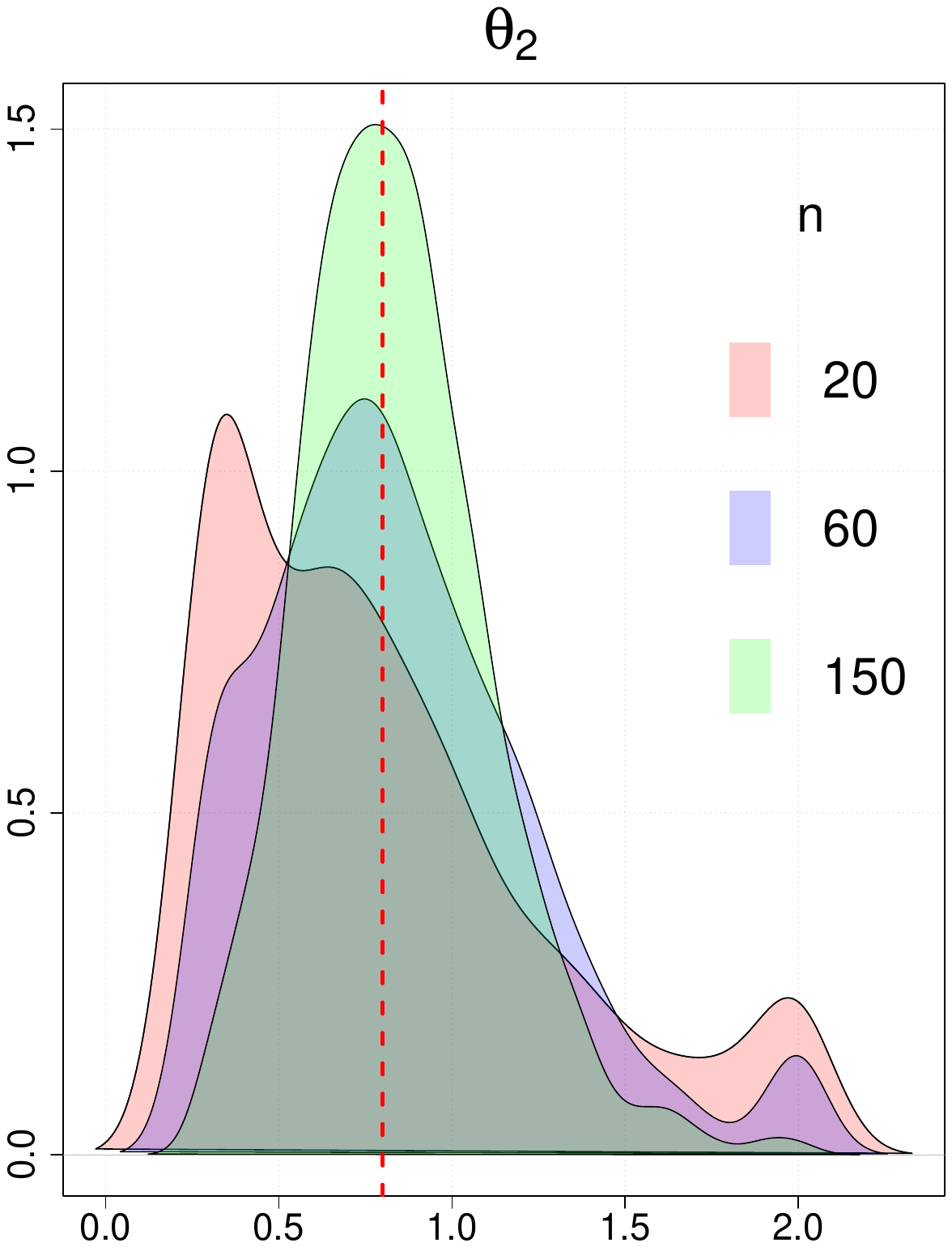}
		\includegraphics[height=3.6cm,width=3.6cm]{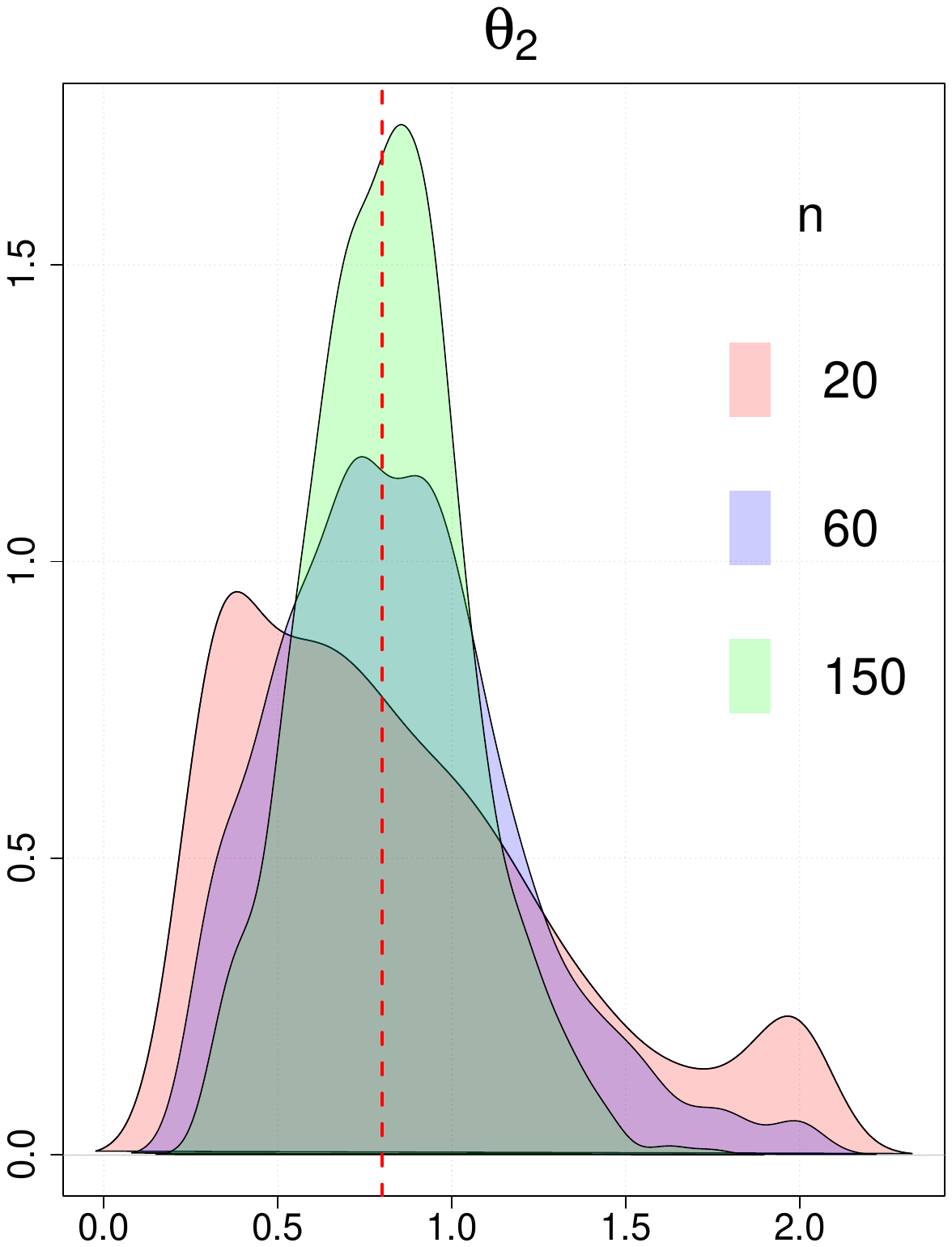}		
	\end{subfigure}
	\quad
	
	\begin{subfigure}
		\centering
		\includegraphics[height=3.6cm,width=3.6cm]{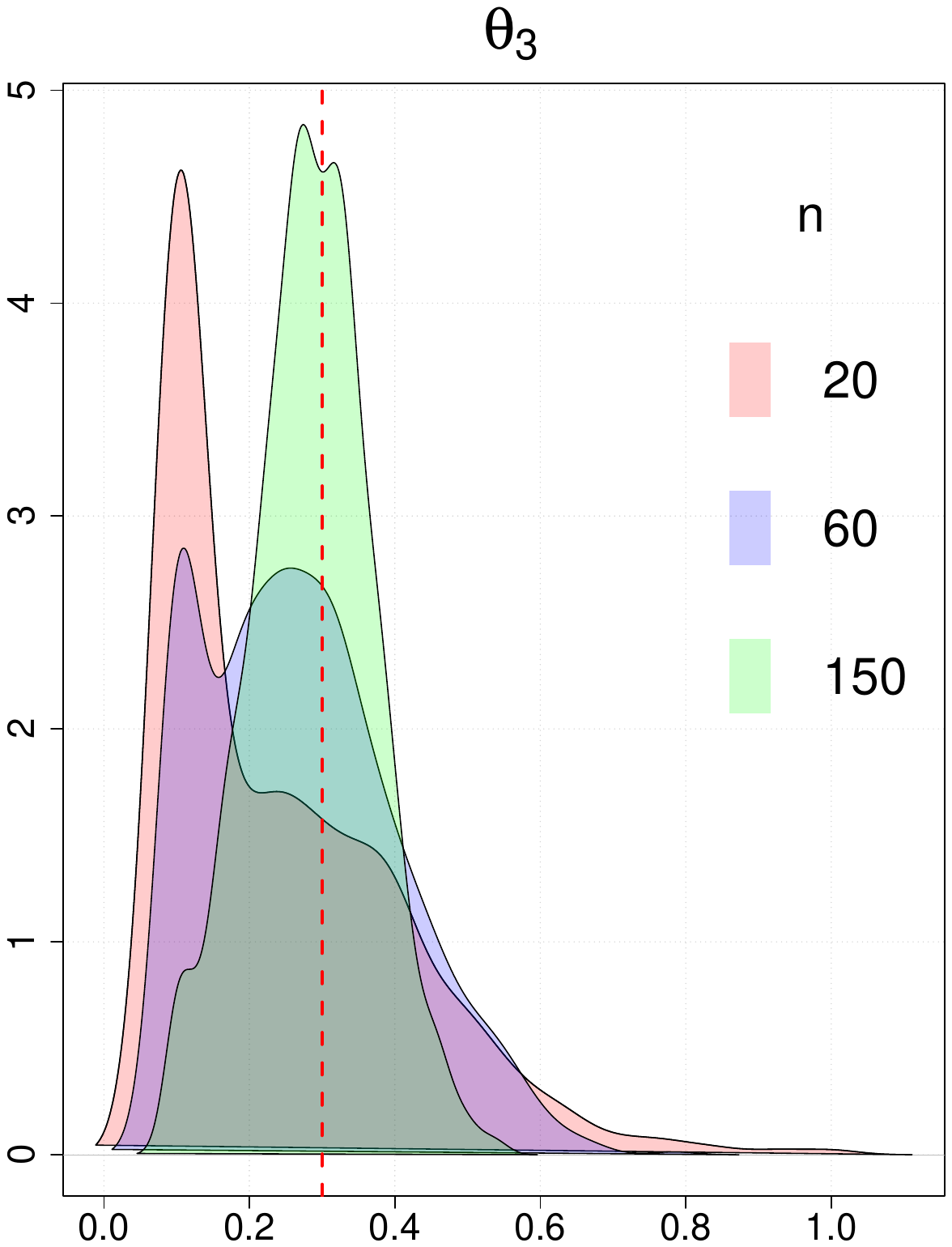}
		\includegraphics[height=3.6cm,width=3.6cm]{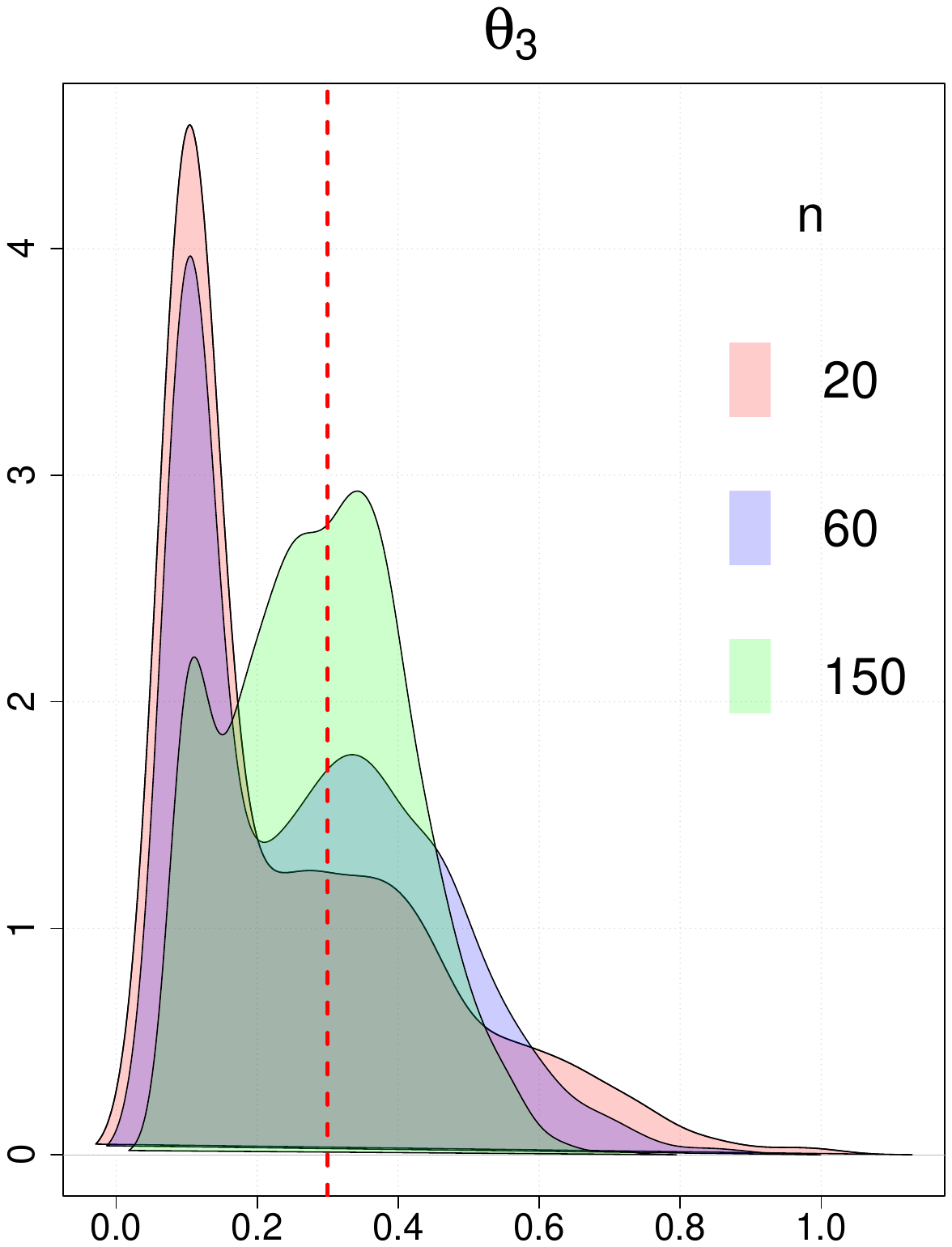}
		\includegraphics[height=3.6cm,width=3.6cm]{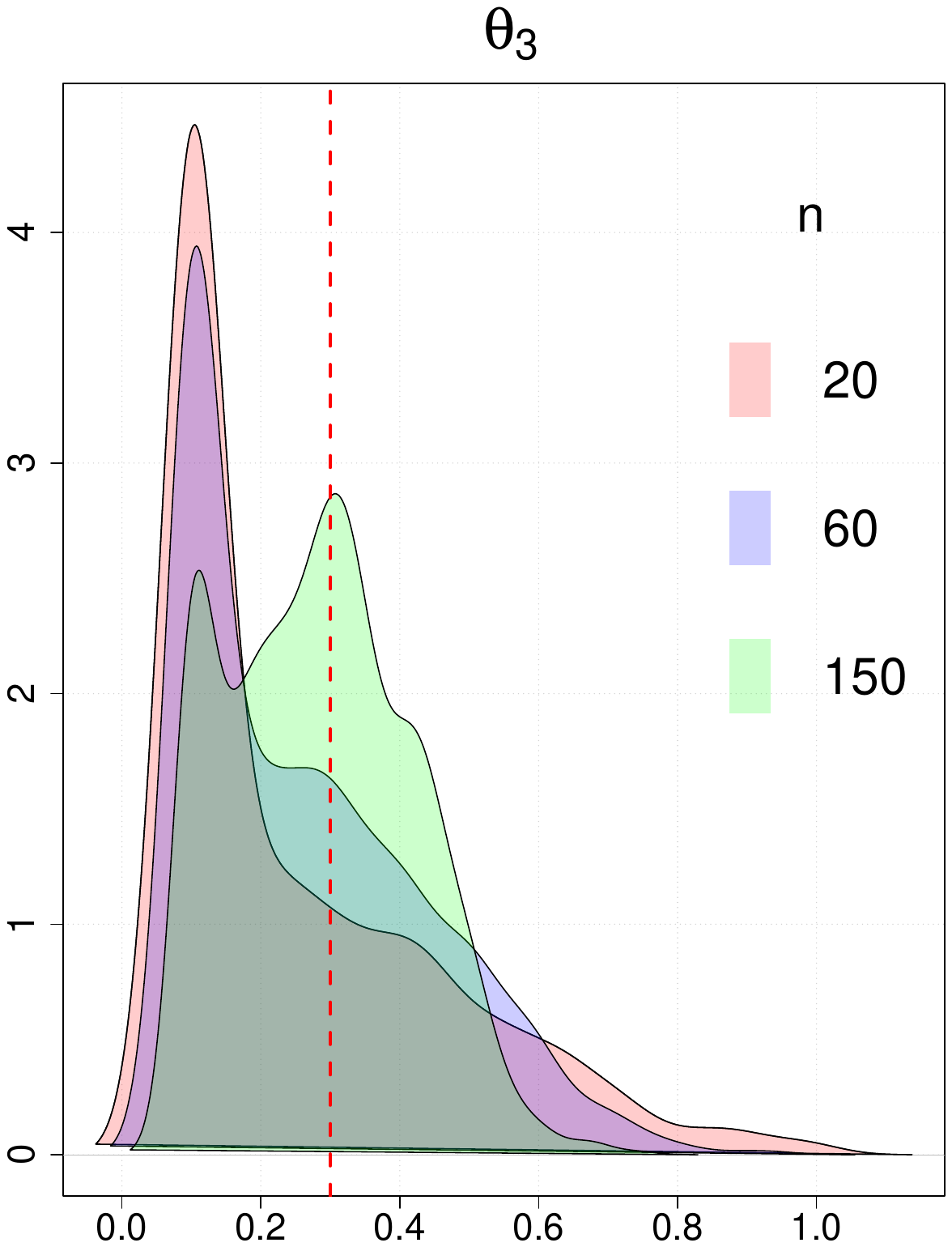}		
	\end{subfigure}
	\caption{Density of the coordinates of $\widehat{\theta}_n$ for the number of observations $n=20$ (in red), $n=60$ (in blue), $n=150$ (in green) with $\theta^*=(0.1,0.8,0.3)$ (represented by the red vertical line). We used the Kendall's tau distance, the Hamming distance and the Spearman's footrule distance from left to right.}
	\label{figdensi}
\end{figure}

\bigskip

\begin{figure}
\center
\includegraphics[height=4cm,width=4cm]{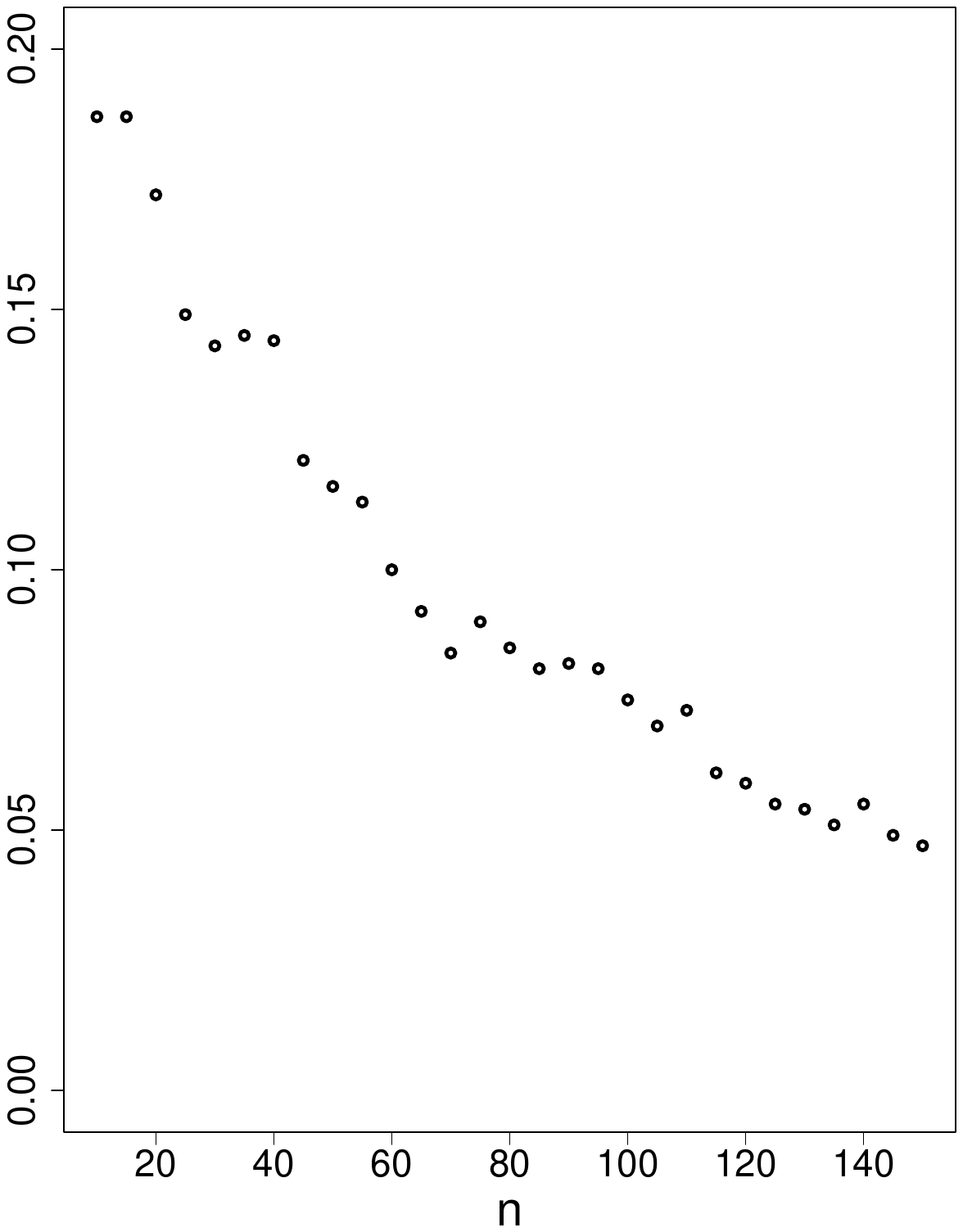}\includegraphics[height=4cm,width=4cm]{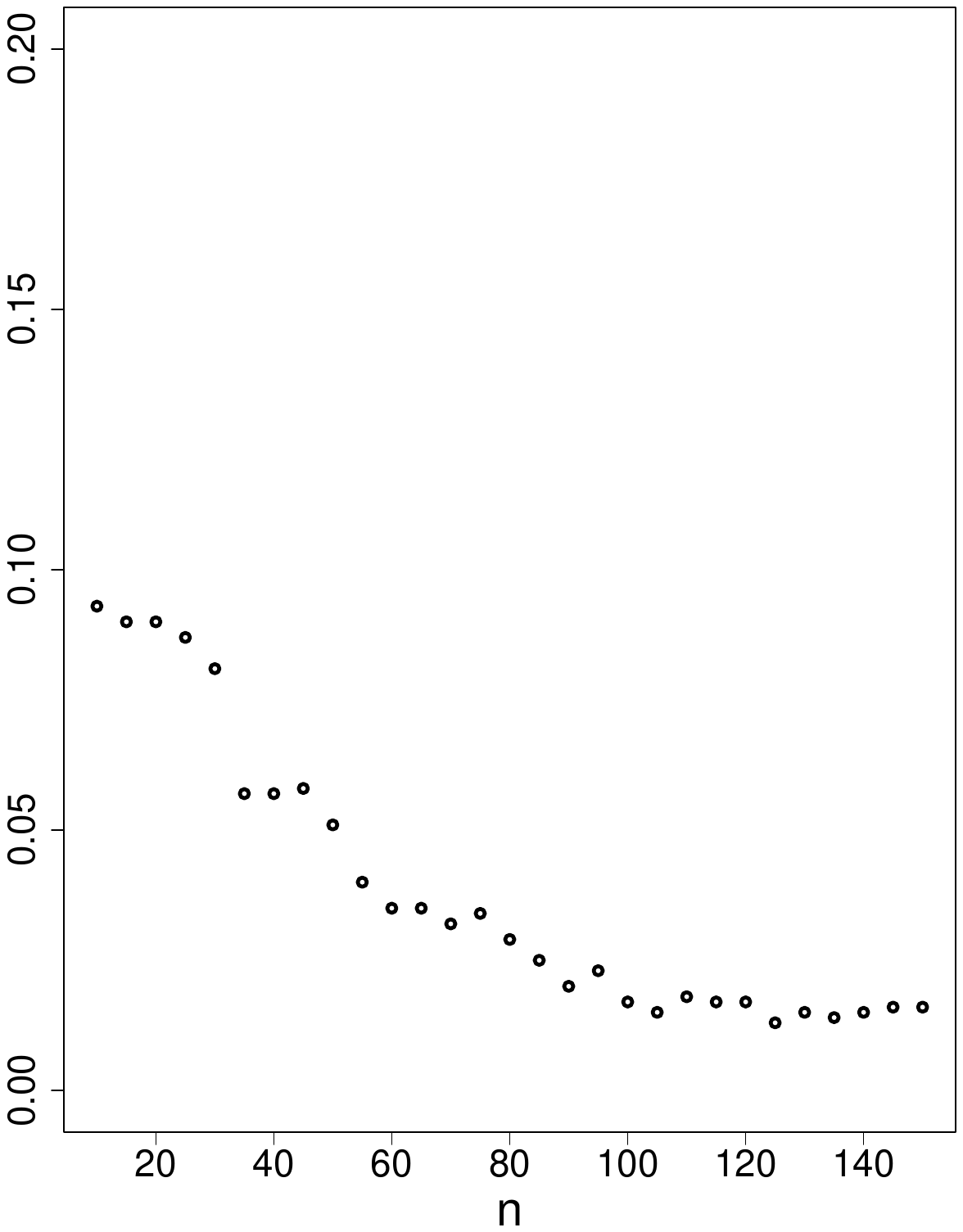}
\includegraphics[height=4cm,width=4cm]{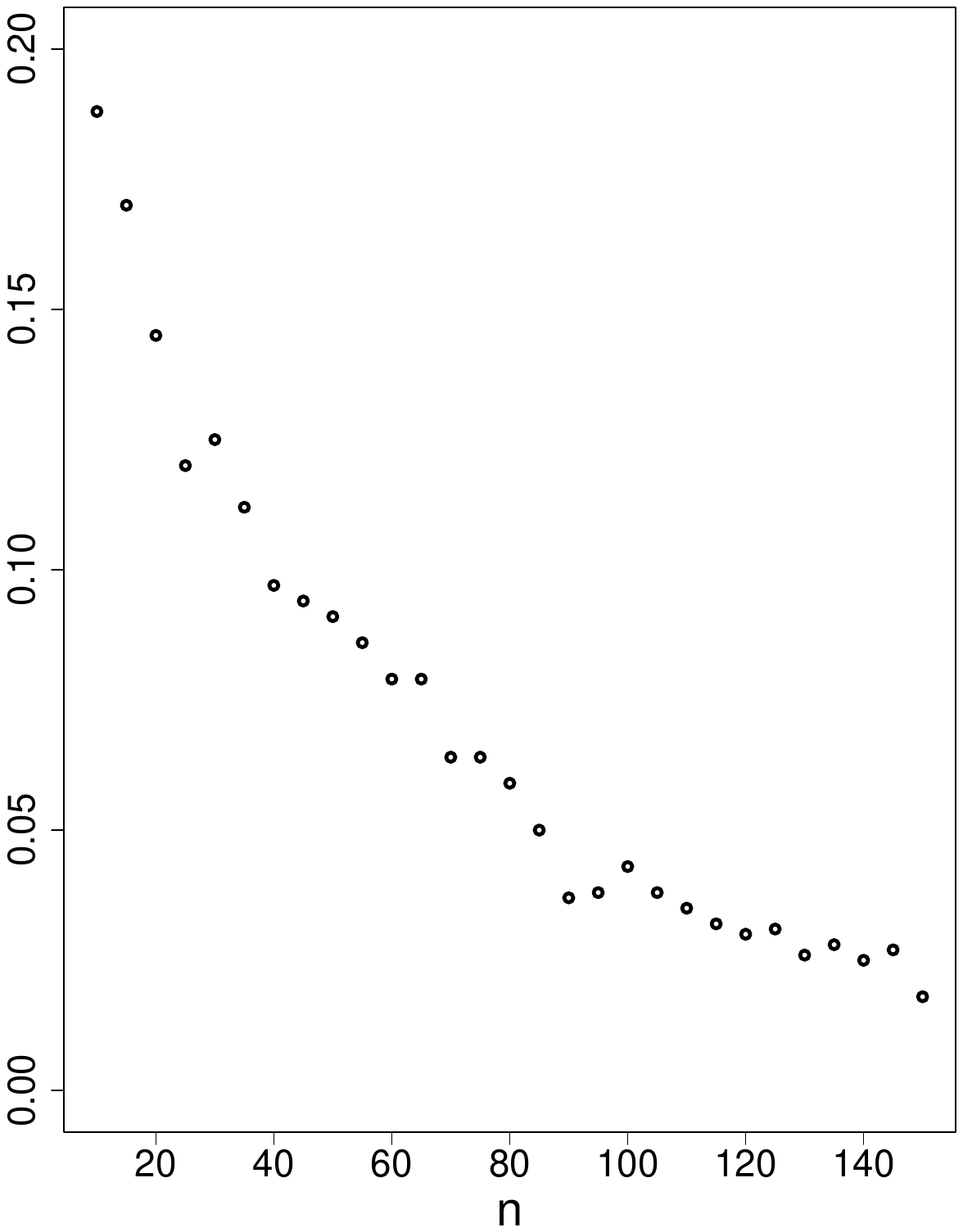}
\caption{Monte Carlo estimates of $\PP\left(\left|\widehat{Y}_{\widehat{\theta}_{n}}(\overline{\sigma}_n)-\widehat{Y}_{\theta^*}(\overline{\sigma}_n)\right|>0.3\right)$ for different values of $n$, the number of observations, with $\theta^*=(0.1,0.8,0.3)$, $\overline{\sigma}_n=(1\;4\;6)\in S_{n+3}$, and the Kendall's tau distance, the Hamming distance and the Spearman's footrule distance from left to right.}
\label{lastthrm}
\end{figure}

\bigskip

In Figure \ref{figdensi}, we display the density of the coordinates of the maximum likelihood estimator for different values of $n$ ranging from 20, 60 to 150. These densities have been estimated with a sample of 1000 values of the maximum likelihood estimator. We observe that the densities can be far from the true parameter for $n=20$ or $n=60$ but are quite close to it for $n=150$. Further, we see that for $n=150$, the Kendall's tau distance seems to give better estimates for $\theta_3^*$. However, the computation time of the distance matrix is much longer with the Kendall's tau distance than with the other distances.

\bigskip

In Figure \ref{lastthrm}, for a given $\overline{\sigma}_n$, we display estimates of the probability that the deviation between the prediction of $Y(\overline{\sigma}_n)$ given in \eqref{eq:pred} with the parameter $\hat{\theta}_{n}$ and the prediction of $Y(\overline{\sigma}_n)$ with the parameter $\theta^*$ exceeds $0.3$. Indeed, Theorem \ref{pred} ensures us that this probability converges to $0$ as $n\rightarrow +\infty$.

\subsection{Application to the optimization of Latin Hypercube Designs}\label{sec_LHD}
We consider here an application of Proposition \ref{pos} to find an optimal Latin Hypercube Design (LHD). A LHD is a design of experiments $(X_j)_{j\leq N}\in [0,1]^d$ where, for each component $i\in[1:d]$, the projections of $X_1,...,X_N$ on the component $i$ are equispaced in $[0,1]$ (see \cite{mckay_comparison_1979}). We will thus consider that each component of one $X_j$ is equal to $k/(N-1)$ for some $k\in [0:N-1]$. We also remark that we can always permute the variables so that the first component of $X_j$ is equal to $(j-1)\slash(N-1)$.  So, for each LHD $(X_j)_{j\leq N}$, there exist $\sigma_2,...,\sigma_d\in S_N$ such that for all $j\in [1:N]$, we have 
$$
X_j=\left(\frac{j-1}{N-1},\frac{\sigma_2(j)-1}{N-1},\cdots ,\frac{\sigma_d(j)-1}{N-1}\right).
$$
Hence, there is a bijection between the set of LHD with $N$ points and the set $S_N^{d-1}$.\\

Now, if $(X_j)_{j\leq N}$ is a LHD, we can define its measure of space filling quality as
$$
f((X_j)_{j\leq N})=\sup_{x\in [0,1]^d}\min_{j\in[1:N]}\|x-X_j\|,
$$
that is the largest distance of a point of $[0,1]^d$ to $(X_j)_{j\leq N}$. We remark that LHDs minimizing $f$ are called minimax \cite{santner_design_2003}. Our aim is to find a minimax LHD $(X_j^*)_{j\leq N}$ . However, given a LHD $(X_j)_{j\leq N}$, its quality $f((X_j)_{j\leq N})$ is  not an obvious quantity and its computation is expensive.

To estimate this quantity, we suggest to generate $N_{tot}$ random points $(x_l)_{l\leq N_{tot}}$ uniformly on $[0,1]^d$, to compute their distance to the LHD and to take the maximum value. This estimation is costly (because of the large number $N_{tot}$) and noisy (because of the randomness of the points $(x_l)_{l\leq N_{tot}}$). Thus, we suggest to use a Gaussian process model on $f$ and to apply the Expected Improvement (EI) strategy \cite{jones_efficient_1998}. Nevertheless, remark that $f$ is a positive function, whereas a Gaussian process realization can take negative values. In this case, different options are possible: firstly, we can ignore the information of the inequality constraint; secondly, we can use Gaussian process under inequality constraints (see \cite{bachoc2019maximum}); thirdly, we can use a transformation of the function to remove the inequality constraint. We choose here the third strategy and we model $\log(f)$ by a Gaussian process realization. We remark that $\log(f)$ can take positive and negative values.

We thus assume that the unknown function $\log(f)$ to minimize  is a realization of a Gaussian process. We have to find a positive definite kernel on $S_N^{d-1}$. Thanks to Proposition \ref{pos}, we have three positive definite kernels on $S_N$, thus on $S_N^{d-1}$ (taking the tensor product of these kernels). Thus, we apply the EI strategy with these three kernels to find the best LHD with $N_{\max}$ calls to the function $f$. The $N_{\max}\slash2$ first LHDs are generated uniformly on $S_{N}^{d-1}$ and the other ones are generated sequentially by following the EI strategy.

More precisely, for $ i \in [N_{\max}\slash 2 -1 : N_{\max}-1]$, let us explain how to choose the $i+1$-th observation, when we have observed the vectors $(\sigma_j^{(k)})_{j \in [ 2: d], k \in [1 : i ]}$ and the associated observations $\left[ \log\left( f\left((\sigma_{j}^{(k)})_{j \in [2:d]} \right) \right) \right]_{k \in [1:i]}$ (we remark that $f$ can be defined equivalently as a function $f(\sigma_2,\ldots,\sigma_d)$ of $d-2$ permutations or as a function $f( (X_j)_{j \leq N} )$ of a LHD). We model $\log(f)$ by a realization of a Gaussian process $Z$, with a conditional mean written $\widehat{Z}_i(\sigma_2,\cdots, \sigma_d)$ and a conditional variance written $\widehat{s}_i^2(\sigma_2,\cdots, \sigma_d)$, given
\begin{equation} \label{eq:observations:EI}
\{ Z( (\sigma_j^{(k)})_{j=2,\ldots,d} ) = \log(f( (\sigma_j^{(k)})_{j=2,\ldots,d} )) \}_{k = 1,\ldots,i}.
\end{equation}
Then, we let 
$$
(\sigma_2^{(i+1)},\cdots,\sigma_d^{(i+1)})\in \underset{\sigma_2,\cdots,\sigma_d \in S_N}{\mathrm{argmax}} \; EI(\sigma_2,\cdots,\sigma_d),
$$
where 
$$
EI(\sigma_2,\cdots,\sigma_d)= \E_i\left(\max \left(M_i - Z(\sigma_2,\cdots,\sigma_d) , 0\right) \right),
$$
where $M_i=\min_{k\in [1:i]} \log (f(\sigma_2^{(k)},\cdots,\sigma_d^{(k)}))$, and $\E_i$ is the expectation conditionally to the observations \eqref{eq:observations:EI}. We have an explicit expression of $EI$,
$$
EI=(M_i-\widehat{Z}_i) \Phi\left(\frac{M_i- \widehat{Z}_i}{\widehat{s}_i} \right)+\widehat{s}_i \phi \left(\frac{M_i- \widehat{Z}_i}{\widehat{s}_i} \right),
$$
where $\phi$ and $\Phi$ are the standard normal density and distribution functions. To choose $(\sigma_2^{(i+1)},\cdots,\sigma_d^{(i+1)})$, we thus solve an optimization problem for $EI$, which has a very small cost compared to evaluating $f$, since the computation of EI is instantaneous. We thus choose the set of permutations that maximizes $EI$ over 2000 sets of uniformly distributed permutations.

We refer to \cite{jones_efficient_1998} for more details on EI. The parameters of the covariance functions are estimated by maximum likelihood at each step.\\

We run an experiment where we compare the performances of the 5 following methods:
\begin{itemize}
\item Random sampling, to generate $N_{\max}$ LHDs of the form $\{(X_j^{(i)})_{j\leq N};\; i\leq N_{\max} \}$ by generating $\sigma_{2},...,\sigma_d$ uniformly and independently;
\item Simulated annealing, choosing that two LHDs $(\sigma_j)_{2\leq j\leq d}$ and  $(\sigma_j')_{2\leq j\leq d}$ are neighbours if there exist transpositions $\tau_2,...,\tau_d$ such that for all $j\in [2:d]$, we have $\sigma_j'=\sigma_j\tau_j$;
\item EI with Kendall distance;
\item EI with Hamming distance;
\item EI with Spearman distance.
\end{itemize}
For each method, the performance indicator is $\min_{i=1,...,N_{\max}}f((X_j^{(i)})_{j\leq N})$.
Here, we take $d=3$, $N=15$, $N_{\max}=200$ and $N_{tot}=27\times 10^6$.
\begin{figure}
\centering
\includegraphics[height=7cm,width=7cm]{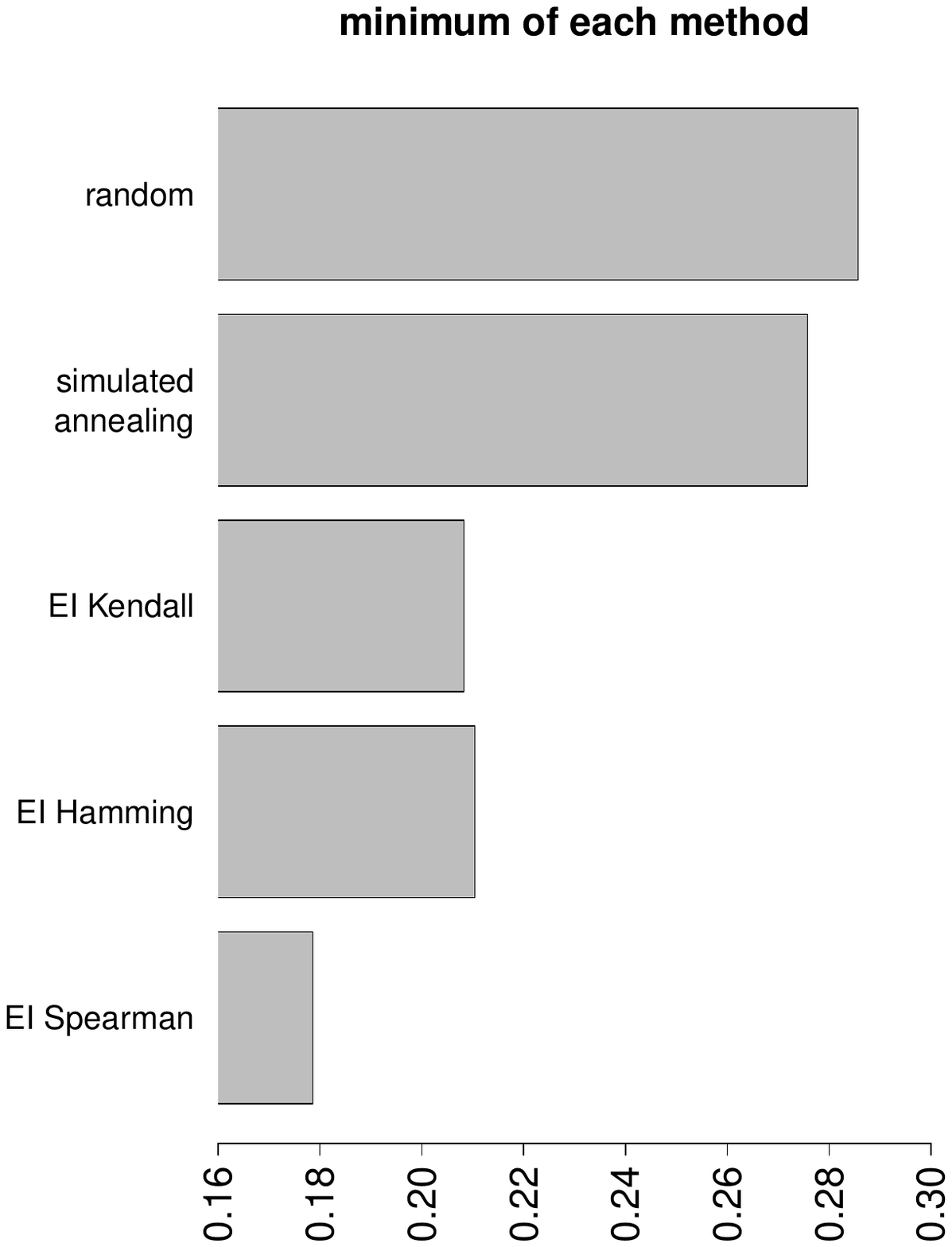}
\caption{Minimal quality of LHD found by the five methods.}
\label{min_EI}
\end{figure}

\begin{figure}
\centering
\includegraphics[height=8cm,width=8cm]{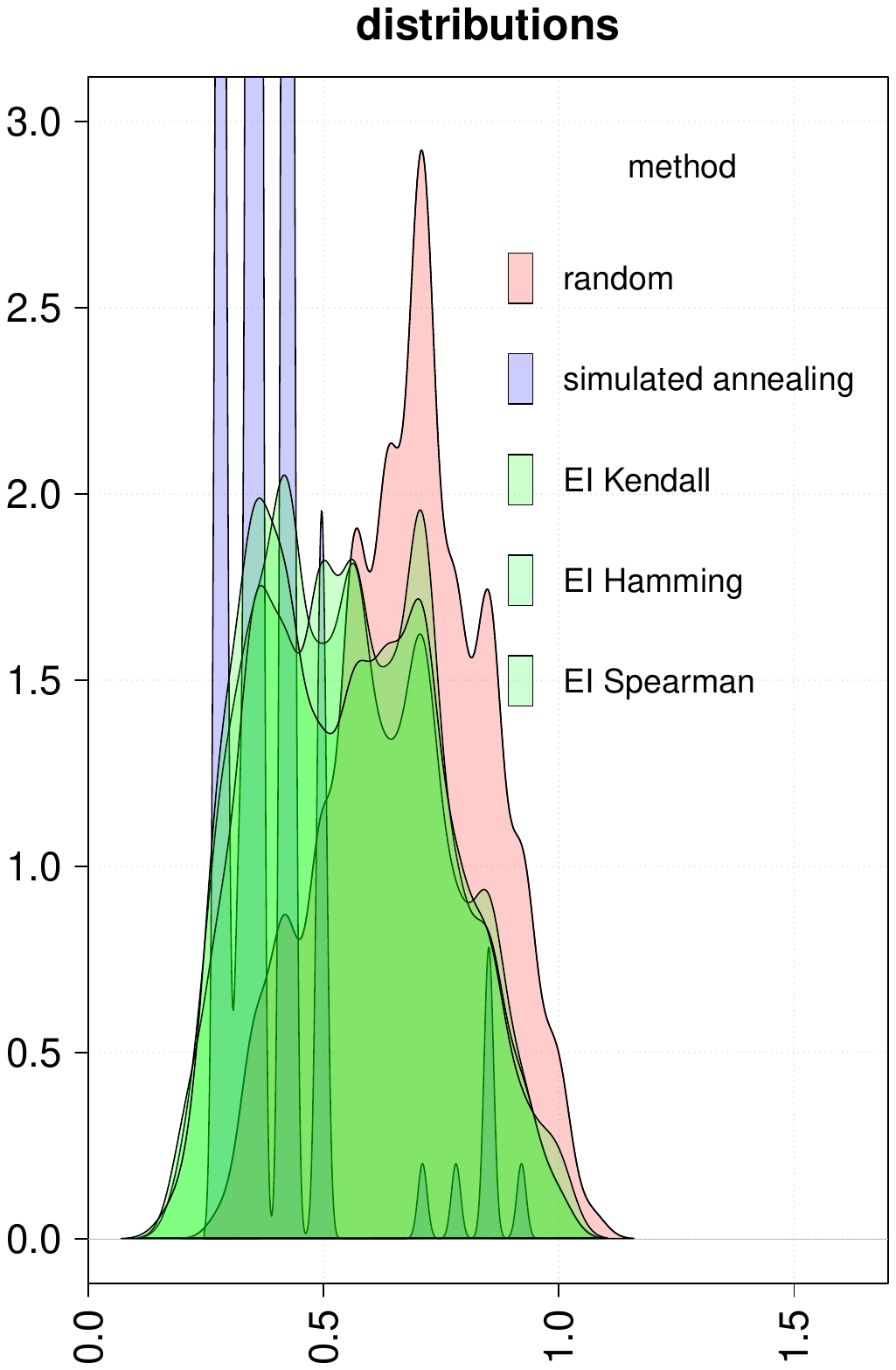}
\caption{Distributions of the quality of LHDs for the five methods.}
\label{distrib_EI}
\end{figure}

We can see in Figure \ref{min_EI} that the best LHDs are found by EI, particularly with the Spearman distance. The simulated annealing is slightly better than random sampling.

We display in Figure \ref{distrib_EI} the distributions of the qualities $\{f((X_j^{(i)})_{j\leq N});i\leq N_{\max} \}$ for the five methods. We can notice that the simulated annealing does not explore the set of all the LHDs and does not find the best minimum. EI performs minimisation and exploration to find better minima. We can then provide the best LHD of EI with the Spearman distance. This LHD is given by the permutations
\begin{eqnarray*}
\sigma_2&=& ( 5 ,   2 ,   1 ,   7 ,   6 ,   3  ,  4 ,   8,   11,    13,    12,     9 ,   10  ,  14 ,   15),\\
\sigma_3 &=&  (  3  ,  6,    1,    8 ,   4 ,   9,   15,    7,   12,     5 ,   13 ,   10  ,   2,    11,    14).
\end{eqnarray*}

To conclude, the kernels on permutations provided in Section \ref{s:kernrank} enable us to use EI that gives much better results than simulated annealing or random sampling to find the best LHD.

\section{Covariance model for partial ranking}\label{s:kernpartial}

\subsection{A new kernel on partial rankings}\label{s:newkernel}
In application, it can happen that partial rankings rather than complete rankings are observed. A partial ranking aims at giving an order of preference between different elements of $X$ without comparing all the pairs in $X$.  Hence, a partial ranking $R$ is a statement of the form 
\begin{equation}\label{partialrank}
X_1\succ X_2\succ\cdots\succ X_m,
\end{equation}
where $m<N$, and $X_1,\cdots,X_m$ are disjoint sets of $X=\{x_1,x_2,\cdots,x_N \}$. The partial ranking means that any element of $X_j$ is preferred to any element of $X_{j+1}$ but the elements of $X_j$ cannot be ordered. Given a partial ranking $R$, we consider the following subset of $S_N$ 
\begin{eqnarray}
E_R&:=&\left\{\sigma \in S_N:\; \sigma(i_1)<\sigma(i_2)<\cdots<\sigma(i_m)\right.\nonumber\\
& &
\left.\mbox{ for any choice of }(x_{i_1},\cdots,x_{i_m})\in X_1\times\cdots\times X_m\;
\right \}.
\label{ER}
\end{eqnarray}
In the statistical literature, there is  a  natural way to extend a positive definite kernel $K$ on $S_N$ to the set of partial rankings (see \cite{kondor2010ranking}, \cite{jiao2017kendall}). To do so, one considers for $R$ and $R'$ two partial rankings the following averaged kernel 
\begin{equation}\label{kernel1}
\mathcal{K}(R,R'):=\frac{1}{|E_R| |E_{R'}|}\sum_{\sigma \in E_R} \sum_{\sigma' \in E_{R'}} K(\sigma,\sigma').
\end{equation}
Here,  $|E_R|$ denotes the cardinal of the set $E_R$. 
Notice that, if  $K$ is a positive definite kernel on permutations, then $\mathcal{K}$  is also a positive definite kernel \cite{haussler_convolution_1999}. Indeed,  if $R_1,\cdots,R_n$ are partial rankings and if $(a_1,\cdots,a_n)\neq0$, then
\begin{equation}
\sum_{i,j=1}^n a_i a_j \mathcal{K}(R_i,R_j)=\sum_{\sigma,\sigma' \in S_N}b_{\sigma} b_{\sigma'} K(\sigma,\sigma'),
\end{equation}
where we set
\begin{equation}
b_\sigma:=\sum_{i,\;\sigma \in R_i}\frac{a_i}{|E_{R_i}|}.
\end{equation}
Observe that the computation of $\mathcal{K}$ is very costly. Indeed, we have to sum over $|E_R||E_{R'}|$ permutations. Several works aim to reduce the computation cost of this kernel (see \cite{kondor2010ranking, lebanon_nonparametric_2008,Lomeli2019}). However, its efficient computation remains an issue.

In the following, we provide another  way  to extend the kernels $K_{\theta_1,\theta_2,\theta_3}$ to partial rankings. We will provide computational simplifications for this extension. First, define the measure of dissimilarity $d_{\mbox{avg}}$ on partial rankings as the mean of distances $d(\sigma,\sigma')$  $(\sigma\in E_R,\sigma'\in E_{R'})$. That is 
\begin{equation}\label{def_davg}
d_{\mbox{avg}}(R,R'):=\frac{1}{|E_R||E_{R'}|}\sum_{\sigma \in E_R}\sum_{\sigma \in E_{R'}}d(\sigma,\sigma').
\end{equation}
Since $d_{\mbox{avg}}(R,R)\neq 0$ in general, we need to define $d_{\mbox{partial}}$ as follows 
\begin{equation}\label{dev_dpartial}
d_{\mbox{partial}}(R,R'):=d_{\mbox{avg}}(R,R')-\frac{1}{2}d_{\mbox{avg}}(R,R)-\frac{1}{2} d_{\mbox{avg}}(R',R').
\end{equation}
\begin{prop}\label{prop_distance_partial}
$d_{\mbox{partial}}^{\frac{1}{2}}$ is a pseudometric on partial rankings (i.e. it satisfies the positivity, the symmetry, the triangular inequality and is equal to 0 on the diagonal $\{ (R,R),\;R $ is a partial ranking$\}$).
\end{prop}

We remark that other metrics on partial rankings are defined in \cite{critchlow2012metric}, in particular the Hausdorff metrics and the fixed vector metrics (based on the group representation of $S_N$). These two metrics are different from the one defined in \eqref{dev_dpartial}. Our suggested metric $d_{\mbox{partial}}$ will enable us to define positive definite kernels in Proposition \ref{prop:positive_partial}. In future work, it would be interesting to study the construction of positive definite kernels based on the Hausdorff and fixed vector metrics.

We further define
\begin{equation}\label{eq:kernelpartial}
\mathcal{K}_{\theta_1,\theta_2,\theta_3}(R,R'):=\theta_2 \exp(-\theta_1 d_{\mbox{partial}}(R,R')^{\theta_3}).
\end{equation}
\noindent
The next proposition warrants that this last function is in fact a covariance kernel, which will later enable to define Gaussian processes on partial rankings.
\begin{prop}\label{prop:positive_partial}
$\mathcal{K}_{\theta_1,\theta_2,\theta_3}$ is a positive definite kernel for the Kendall's tau distance, the Hamming distance and the Spearman's footrule distance.
\end{prop}

\subsection{Kernel computation in partial ranking}\label{s:computationkernel}

At a first glance, the computation of the kernel $\mathcal{K}_{\theta_1,\theta_2,\theta_3}(R,R')$ on partial rankings may still appear very costly due to the evaluation of $d_{\mbox{partial}}$. Indeed, we have to sum $|E_R||E_{R'}|$ elements for $d_{\mbox{avg}}(R,R')$, $|E_R|^2$ elements for $d_{\mbox{avg}}(R,R)$ and $|E_{R'}|^2$ elements for $d_{\mbox{avg}}(R',R')$. However, this computation problem can be quite simplified. As we will show in this subsection,  the mean of the distances is much easier to compute than the mean of exponential of distances. We write $d_{\tau,\mbox{avg}}$ (resp. $d_{H,\mbox{avg}}$ and $d_{S,\mbox{avg}}$) for the average distance in \eqref{def_davg} when the distance on the permutations is $d_\tau$
(resp. $d_H$ and $d_S$).

\noindent
To begin with, let us consider the case of top-$k$ partial rankings. 
A top-$k$ partial ranking (or a top-$k$ list) is a partial ranking of the form
\begin{equation}
x_{i_1}\succ x_{i_2} \succ\cdots\succ x_{i_k}\succ X_{rest},
\end{equation} 
where $X_{rest}:=X\setminus \{x_{i_1},\cdots,x_{i_k} \}$. It can be seen as the "highest rankings".
In order to alleviate the notations, let just write $I=(i_1,\cdots,i_k)$ for this top-$k$ partial ranking. The following proposition shows that the computation cost to evaluate $d_{\mbox{avg}}$ (and so the kernel values) might be reduced when the partial rankings are  in fact top-$k$ partial rankings. Before stating this proposition let us define some more mathematical objects. Let $I:=(i_1,\cdots,i_k)$ and $I':=(i^{\prime}_1,\cdots,i^{\prime}_k)$ be two top-$k$ partial rankings. Let 
$$\{j_1,\cdots, j_p\}:= \{i_1,\cdots,i_k\}\cap  \{i_1^{\prime},\cdots,i_k^{\prime}\}$$ 
where $j_1<j_2<\cdots<j_p$ and $p$ is an integer no larger than $k$. Let, for $l=1,\cdots p $, $c_{j_l}$ (resp. $c^{\prime}_{j_l}$) denotes the rank of $j_l$ in $I$ (resp. in $I'$). Further, let $r:=k-p$  and define  $\tilde{I}$ (resp.
$\tilde{I'}$) as the complementary set of  $\{j_1,\cdots, j_p\}$ in
$\{i_1,\cdots,i_k\}$ (resp. in $\{i_1^{\prime},\cdots,i_k^{\prime}\}$).
Writing these two sets in ascending order, we may finally define for
$j=1,\cdots,r$, $u_j$ (resp. $u'_j$) as the rank in $I$ (resp $I'$) of the $j$-th element of $\tilde{I}$ (resp. $\tilde{I'}$).

\begin{ex}
Assume that $n=7$, $I=(3,2,1,4,5)$ and $I'=(3,5,1,6,2)$. We have $(j_1,j_2,j_3,j_4)= (1,2,3,5)$ (the items ranked by $I$ and $I'$, in increasing order). Thus, $c_{j_1}=3,\;c_{j_2}=2,\;c_{j_3}=1,\;c_{j_4}=5$ and $c'_{j_1}=3,\;c'_{j_2}=5,\;c'_{j_3}=1,\;c'_{j_4}=2$. Further,  $u_1=4$ and $u'_1=4$.
\end{ex}

\begin{prop}\label{prop:computation}
Let $I$ and $I'$ be two top $k$-partial rankings. 
Set $N':=N-k-1$ and $m:=N-|I\cup I'|$. Then,
\begin{eqnarray*}
d_{\tau,\mbox{avg}}(I,I') &=&\sum_{1\leq l <l' \leq p }\mathds{1}_{(c_{j_l}<c_{j_{l'}},c'_{j_l}>c'_{j_{l'}})\text{ or }(c_{j_l}>c_{j_{l'}},c'_{j_l}<c'_{j_{l'}})}+   r(2k+1-r)\\& & -\sum_{j=1}^r(u_j+u'_j)+r^2+ \begin{pmatrix}
N-k\\2
\end{pmatrix}-\frac{1}{2}\begin{pmatrix}
m\\2
\end{pmatrix}, \\
d_{H,\mbox{avg}}(I,I')&=&\sum_{l =1}^p\mathds{1}_{c_{j_l}\neq c_{j_l}'}+m\frac{N-k-1}{N-k}+2r,\\
d_{S,\mbox{avg}}(I,I')&=&\sum_{l=1}^p|c_{j_l}-c'_{j_l}|+r(N+k+1)-\sum_{j=1}^r(u_j+u'_j)\\ & & +mN'-\frac{mN'(2N'+1)}{3(N'+1)}.
\end{eqnarray*}
\end{prop}
Notice that the  sequences $(c_{j_l})$, $(c^{\prime}_{j_l})$ and 
$(u_{j})$, $(u^{\prime}_{j})$ are easily computable and so $d_{\mbox{avg}}(I,I')$ too. Let us discuss an easy example to handle the computation of the previous sequences.

\begin{ex}
Assume that $n=7$, $I=(3,2,1,4,5)$ and $I'=(3,5,1,6,2)$. Proposition \ref{prop:computation} 
leads to
$$
d_{\tau,\mbox{avg}}(I,I')=6,\;\;d_{S,\mbox{avg}}(I,I')=4.5,\;\;d_{S,\mbox{avg}}(I,I')=11.5.
$$
\end{ex}
\noindent
To compute the pseudometric $d_{\mbox{partial}}$ defined in \eqref{dev_dpartial}, we also need to compute $d_{\tau,\mbox{avg}}$ on the diagonal $\{ (I,I)|\; I$ is a top-$k$ partial ranking$\}$. The following corollary gives these computations.
\begin{corol}
Let $I$ be a top-$k$ partial ranking. Then,
\begin{eqnarray*}
d_{\tau,\mbox{avg}}(I,I)&=&\frac{1}{2}\begin{pmatrix}
N-k\\ 2
\end{pmatrix},\\
d_{H,\mbox{avg}}(I,I)&=&  N-k-1,\\
d_{S,\mbox{avg}}(I,I)&=&  (N-k)(N-k-1)+\frac{(N-k-1)(2N-2k-1)}{3}.
\end{eqnarray*}
\end{corol}
\noindent

\begin{rmk}
Similar results as Proposition \ref{prop:computation} are stated in Sections III.B and III.C of \cite{critchlow2012metric} for the Hausdorff metrics and the fixed vector metrics respectively.
\end{rmk}

In the case of the Hamming distance, we may step ahead and provide a simpler computational formula for the average distance between two partial rankings whenever their associated partitions share the same number of members (see Proposition \ref{prop::computationH} below). More precisely 
let $R_1$ and $R_2$ be two partial rankings such that
\begin{equation}
R_i=X_1^i\succ\cdots\succ X_k^i,\; i=1,2,
\end{equation}
assume also that for $j=1,\cdots,k$,  $|X_j^1|=|X_j^2|$ and denote by $\gamma_j$ this integer. Obviously, $N=\sum_{j=1}^k \gamma_j$ so that $\gamma:=(\gamma_j)_j$ is an integer partition of $n$.
Further, when $1=\gamma_1=\gamma_2=\cdots=\gamma_{k-1}$ and $\gamma_k=N-k+1$ one is in the  top-$(k-1)$ partial ranking case.
For $j=1,\cdots, k$, let  $\Gamma_j$ be the set of all integers lying in $\left[\sum_{l=1}^{j-1}\gamma_l+1,\sum_{l=1}^j\gamma_l\right]$. Set further, 
$$
S_\gamma :=S_{\Gamma_1}\times S_{\Gamma_2}\times \cdots\times S_{\Gamma_k},
$$
where $S_{\Gamma_i}$ is the set of permutations on $\Gamma_i$.
Notice that $S_\gamma$ is nothing more than the subgroup of $S_n$ letting invariant the sets ${\Gamma_j}$ ($j=1,\cdots, k$). So that, for $i=1,2$,  we can write $E_{R_i}$ as a right coset $R_i=S_\gamma \pi_i$ for some $\pi_i \in E_{R_i}$.  With these extra notations and definitions, we are now able to compute $d_{H,\mbox{avg}}(R_1,R_2)$.
\begin{prop}\label{prop::computationH}
In the previous setting, we have
\begin{equation}\label{eq_H_partial}
d_{H,\mbox{avg}}(R_1,R_2)=|\{i,\; \Gamma(\pi_1(i))\neq \Gamma(\pi_2(i))\}|+\sum_{j=1}^k\frac{\gamma_j}{N}(\gamma_j-1),
\end{equation}
where, for $1\leq l\leq N$, $\Gamma(l)$ is the integer $j$ such that $l \in \Gamma_j$.
\end{prop}
Note that in \eqref{eq_H_partial}, the term $|\{i,\; \Gamma(\pi_1(i))\neq \Gamma(\pi_2(i))\}|$ counts the number of item $i\in [1:N]$ that are ranked differently in $R_1$ and $R_2$.

\subsection{Numerical experiments}\label{section_partial_num}
We have proposed in Section \ref{s:newkernel} a new kernel $\mathcal{K}_{\theta_1,\theta_2,\theta_3}$ defined by \eqref{eq:kernelpartial} on partial rankings. We show in Section \ref{s:computationkernel} that in several cases (for example with top-$k$ partial rankings), we can reduce drastically the computation of this kernel. Another direction is given in \cite{jiao2017kendall} by considering the averaged Kendall kernel and reducing the computation of this kernel on top-$k$ partial rankings. This kernel is available on the R package { \verb kernrank }. We write $\mathcal{K}$ the averaged Kendall kernel, and we define $\mathcal{K}_{\theta_1}:=\theta_1 \mathcal{K}$.

In this section, we compare our new kernel $\mathcal{K}_{\theta_1,\theta_2,\theta_3}$ with the averaged Kendall kernel $\mathcal{K}_{\theta_1}$ in a numerical experiment where an objective function indexed by top-$k$ partial rankings is predicted, by Kriging. We take $N=10$ and for simplicity, we take the same value $k=4$ for all the top-$k$ partial rankings.
For a top-$k$ partial ranking $I=(i_1,i_2,i_3,i_4)$, the objective function to predict is $
f(I):=2i_1+i_2-i_3-2i_4$. We make 500 noisy observations $(y_i)_{i\leq 500}$ with $y_i=f(I_i)+\varepsilon_i$, where $(I_i)_{i\leq 500}$ are i.i.d. uniformly distributed top-k partial rankings and $(\varepsilon_i)_{i\leq 500}$ are i.i.d. $\mathcal{N}(0,\lambda^2)$, with $\lambda=\frac{1}{2}$. As in Section \ref{s:GPrank}, we estimate $(\theta,\lambda)$ by maximum likelihood. Then, we compute the predictions $(\widehat{y}_i')_{i\leq 500}$ of $y'=(y_i')_{i\leq 500}$, with $y'$ the observations corresponding to 500 other test points $(I_i')_{i\leq 500}$, that are i.i.d. uniform top-k partial rankings.

For the four kernels (our kernel $\mathcal{K}_{\theta_1,\theta_2,\theta_3}$ with the 3 distances and the averaged Kendall kernel $\mathcal{K}_{\theta_1}$), we provide the rate of test points that are in the 90$\%$ confidence interval together with the coefficient of determination $R^2$ of the predictions of the test points.
Recall that
$$
R^2:=1-\frac{\frac{1}{500}\sum_{i=1}^{500}\left( y_i' -\widehat{y}_i'\right)^2}{\frac{1}{500}\sum_{i=1}^{500}\left( y_i' - \overline{y'}\right)^2},
$$
where $\overline{y'}$ is the average of $y'$. The results are provided in Table \ref{table}.
\begin{table}
\centering
\begin{tabular}{|c|c|c|c|c|}
 \hline 
 kernel & $\mathcal{K}_{\theta_1,\theta_2,\theta_3}^\tau$ & $\mathcal{K}_{\theta_1,\theta_2,\theta_3}^H$  &$\mathcal{K}_{\theta_1,\theta_2,\theta_3}^S$ & $\mathcal{K}_{\theta_1}$ \\
  \hline
rate & 0.902 &  0.904 &  0.912 &  0.928 \\
$R^2$ & 0.887 & 0.996 &  0.996 & 0.070 \\
  \hline
\end{tabular}
\caption{Rate of test points that are in the 90$\%$ confidence interval and coefficient of determination for the four kernels.}
\label{table}
\end{table}

The rate of test points that are in the 90$\%$ confidence interval is close to $90\%$ for the four kernels. We can deduce that the parameters $(\theta,\lambda)$ are well estimated by maximum likelihood, even for the averaged Kendall kernel $\mathcal{K}_{\theta_1}$.

However, we can see that the coefficient of determination of the averaged Kendall kernel $\mathcal{K}_{\theta_1}$ is close to 0. The predictions given by the averaged Kendall kernel $\mathcal{K}_{\theta_1}$ are nearly as bad as predicting with the empirical mean. In the opposite way the coefficient of determination of our kernels is larger than $0.9$ for the Kendall distance, and larger than $0.99$ for the Hamming distance and the Spearman distance. That means that the prediction given by our kernels are much better than the empirical mean. \\

To conclude, we provide a class of positive definite kernels $\mathcal{K}_{\theta_1,\theta_2,\theta_3}$ which seems to be significantly more efficient than the averaged Kendall kernel $\mathcal{K}_{\theta_1}$, in the case of Gaussian process models on partial rankings.

\section{Conclusion}\label{conclu}
In this paper, we provide a Gaussian process model for permutations. Following the recent works of \cite{jiao2017kendall} and \cite{mania_kernel_2016}, we propose kernels to model the covariance of such processes and show the relevance of such choices. Based on the three distances on the set of permutations,  Kendall's tau, Hamming distance and Spearman's footrule distance, we obtain parametric families of  relevant covariance models. To show the practical efficiency of these parametric families, we apply them to the optimization of Latin Hypercube Designs. In this framework, we prove under some assumptions on the set of observations, that the parameters of the model can be estimated and the process can be forecasted using linear combinations of the observations, with asymptotic efficiency. Such results enable to extend the well-known properties of Kriging methods to the case where the process is indexed by ranks and tackle a large variety of problems. We remark that our asymptotic setting corresponds to the increasing domain asymptotic framework for Gaussian processes on the Euclidean space. It would be interesting to extend our results to more general sets of permutations under designs that do not necessarily satisfy Conditions 1 and 2.

We also show that the Gaussian process framework can be extended to the case of partially observed ranks. This corresponds to many practical cases. We provide new kernels on partial rankings, together with results that significantly simplify their computation. We show the efficiency of these kernels in simulations. We leave a specific asymptotic study of Gaussian processes indexed by partial rankings open for further research.

As highlighted in \cite{marden_analyzing_2014}, data consisting of rankings arise from many different fields. Our suggested kernels on total rankings and partial rankings could lead to different applications to real ranking data. We treated the case of regression in Sections \ref{s:appli} and \ref{section_partial_num}. In Section \ref{sec_LHD}, we used these kernels for an optimization problem. One could also use our suggested kernels in classification, as it is done in \cite{jiao2017kendall}, in \cite{mania_kernel_2016} or in \cite{kondor2010ranking}, and also using Gaussian process based classification \cite{GPML} with ranking inputs.

\section*{Acknowledgement} 
We are grateful to Jean-Marc Martinez for suggesting us the Latin Hypercube Design application. We are indebted to an associate editor and to three anonymous reviewers, for their comments and suggestions, that lead to an improved revision of the manuscript.

\bibliographystyle{abbrv}
\bibliography{references}


\clearpage

\appendix
\section{Proofs for Sections \ref{s:kernrank} and \ref{s:kernpartial}}

{\bf Proof of Proposition~\ref{defpos}}
\begin{proof}
We show that $K_{\theta_1,\theta_2}$ is a strictly positive definite kernel on $S_n$. It suffices to prove that, if $\nu>0$, the map $K$ defined by
\begin{equation}
K(\sigma,\sigma'):=e^{-\nu d(\sigma,\sigma')}
\end{equation}
is a strictly positive definite kernel.

\paragraph{Case of the Kendall's tau distance.}
 It has been shown in Theorem 5 of \cite{mania_kernel_2016} that $K$ is a strictly positive definite kernel on $S_N$ for the Kendall's tau distance. Nevertheless, we  provide here an other shorter and easier proof. The idea is to write $K(\sigma_1,\sigma_2)$ as $M(\Phi(\sigma_1),\Phi(\sigma_2))$, for an application $\Phi$ defined below, for a function $M$ defined below and for $\sigma_1,\sigma_2 \in S_N$. We will then show that $M$ is strictly positive definite and which will imply that $K$ also is.
 
 Let
$$
\Phi:\begin{array}{ccl}S_N & \longrightarrow & \{0,1\}^{\frac{N(N-1)}{2}}\\
\sigma & \longmapsto & (\mathds{1}_{\sigma(i)<\sigma(j)})_{1\leq i<j\leq N}.
\end{array}
$$ 
Further, define
$$
M:\begin{array}{ccl}
\{0,1\}^{\frac{N(N-1)}{2}} \times \{0,1\}^{\frac{N(N-1)}{2}} & \longrightarrow & \R\\
\left( (a_{i,j})_{i,j},(b_{i,j})_{i,j}\right) & \longmapsto & \exp \left( - \nu \sum_{i<j} |a_{i,j}-b_{i,j}| \right).
\end{array}
$$
Remark that for all $\sigma, \sigma'$, we have 
\begin{equation*}\label{eq_K_M(Phi)}
K(\sigma,\sigma')=M(\Phi(\sigma),\Phi(\sigma')).
\end{equation*}
Now, assume that $M$ is a strictly positive definite kernel. Let $n \in \N$ and let $\sigma_1, \cdots ,\sigma_n \in S_N$ such that $\sigma_i\neq \sigma_j$ if $i \neq j$. As $\Phi$ is injective, we have $\Phi(\sigma_i)\neq \Phi(\sigma_j)$ if $i\neq j$, and so $(K(\sigma_i,\sigma_j))_{1\leq i ,j \leq n}=\left(M(\Phi(\sigma_i),\Phi(\sigma_j))\right)_{1\leq i,j\leq n}$ is a symmetric positive definite matrix. Thus, $K$ is a strictly positive definite kernel. 

It remains to prove that $M$ is a strictly positive kernel.
For all $k\in \N^*$, we index the elements of $\{0,1\}^{k}$ using the following bijective map
$$
N_k:\begin{array}{ccl}
\{0,1\}^k & \longrightarrow & [1:2^k]\\
(a_i)_{i\leq k} & \longmapsto & 1+\sum_{i=1}^k a_i 2^{i-1}.
\end{array}
$$
With this indexation, we let $\tilde{M}$ be the square matrix of size $2^{\frac{N(N-1)}{2}}$ defined by 
$$
\tilde{M}_{i,j}:=M(N_{\frac{N(N-1)}{2}}^{-1}(i),N_{\frac{N(N-1)}{2}}^{-1}(j)).$$
By induction on $k$, we show that the $2^k \times 2^k$ matrix $M^{(k)}$ defined by
$$
M_{i,j}^{(k)}:=\exp\left( - \nu \sum_{l=1}^{k} | N_k^{-1}(i)_l-N_k^{(-1)}(j)_l| \right),\;\;\;(i,j\in [1:2^k]),
$$
is the Kronecker product of $k$ matrices $A_\nu$ defined by
$$
A_\nu:=\begin{pmatrix}
1 & e^{-\nu} \\ e^{-\nu} & 1
\end{pmatrix},\;\;\;(\nu>0).
$$
This is obvious for $k=1$. Assume that this is true for some $k$. Thus, for all $i\leq 2^k$ and $j\leq 2^k$, we have
\begin{eqnarray*}
(A_\nu\otimes M^{(k)})_{i,j}&=&1 M^{(k)}_{i,j}\\
&=&\exp\left( - \nu \sum_{l=1}^{k} | N_k^{-1}(i)_l-N_k^{(-1)}(j)_l| \right)\\
&=&\exp\left( - \nu \sum_{l=1}^{k+1} | N_{k+1}^{-1}(i)_l-N_{k+1}^{(-1)}(j)_l| \right)\\
&=&M^{(k+1)}_{i,j}.
\end{eqnarray*}
With the same computation, we have
$$
(A_\nu\otimes M^{(k)})_{i+2^k,j+2^k}=M^{(k+1)}_{i+2^k,j+2^k}.
$$
We also have
\begin{eqnarray*}
(A_\nu\otimes M^{(k)})_{i+2^k,j}&=&e^{-\nu} M^{(k)}_{i,j}\\
&=&\exp\left( - \nu \left[1+\sum_{l=1}^{k} | N_k^{-1}(i)_l-N_k^{(-1)}(j)_l| \right] \right)\\
&=&\exp\left( - \nu \sum_{l=1}^{k+1} | N_{k+1}^{-1}(i)_l-N_{k+1}^{(-1)}(j)_l| \right)\\
&=&M^{(k+1)}_{i+2^k,j},
\end{eqnarray*}
and with the same computation,
$$
(A_\nu\otimes M^{(k)})_{i,j+2^k}=M^{(k+1)}_{i,j+2^k}.
$$
So we conclude the induction. Using this result with $k=\frac{N(N-1)}{2}$, we have that the matrix $\tilde{M}$ is the Kronecker product of positive definite matrices, thus it is positive definite and so, $M$ is a strictly positive definite kernel.

\begin{rmk}
We could have showed that $M$ is a positive definite kernel using Example 21.5.1 and Property 21.5.8 of \cite{rachev2013methods} (it is a straightforward consequence of these example and property). However, these example an property do not prove the strict positive definiteness of $M$. 
\end{rmk}

\paragraph{Case of the other distances.}
For the Hamming distance and the Spearman's footrule distance, we show that the kernel $K$ is strictly positive definite on the set $F$ of the functions from $[1:N]$ to $[1:N]$. Indeed, if "for all $n\in \N$ and all $f_1,\cdots,f_n\in F$ such that $f_i\neq f_j$ if $i\neq j$, $(K(f_i,f_j))_{1 \leq i,j\leq n}$ is a symmetric positive definite matrix", then "for all $n\in \N$ and all $\sigma_1,\cdots,\sigma_n \in S_N \subset F$ such that $\sigma_i\neq \sigma_j$ if $i\neq j$, $(K(\sigma_i,\sigma_j))_{1 \leq i,j\leq n}$ is a symmetric positive definite matrix". Now, to prove the strict positive definiteness of $K$ on $F$, it suffices to index the elements of $F$ by $f_1,\cdots , f_{N^N}$ and to prove that the matrix $\tilde{M}:=(K(f_i,f_j))_{1 \leq i,j \leq N^N}$ is symmetric positive definite. We index the elements of $F$ using the following bijective map
$$
J_N:\begin{array}{ccl}
F & \longrightarrow & [1:N^N]\\
f & \longmapsto & 1+\sum_{i=1}^N N^{i}(f(i)-1).
\end{array}
$$
Thus, it suffices to show that the $N^N\times N^N$ matrices $\tilde{M}$ defined by
$$
\tilde{M}_{i,j}:=K\left(J_N^{-1}(i),J_N^{-1}(j)\right),
$$
are positive definite matrices for these three distances. Straightforward computations show that
\begin{itemize}
\item For the Hamming distance, $\tilde{M}$ is the Kronecker product of $N$ matrices, all equal to $\left(\exp(-\nu\mathds{1}_{i\neq j})\right)_{i,j\in [1:N]}$.
\item For the Spearman Footrule distance, $\tilde{M}$ is the Kronecker product of $N$ matrices, all equal to $\left(\exp(-\nu|i-j|)\right)_{i,j\in [1:N]}$.
\end{itemize}
In all cases, $\tilde{M}$ is a Kronecker product of positive definite matrices thus is also a positive definite matrix.

\end{proof}

\begin{lm}\label{lm1}
For all the three distances, there exist constants $d_N\in \N^*$, $C_N\in \R$ and a function $\Phi:S_N\rightarrow \R^{d_N}$ such that $d(\sigma,\sigma')=C_N-\langle\Phi(\sigma),\Phi(\sigma')\rangle$. Here $\langle \cdot,\cdot\rangle$ denotes the standard scalar product on $\R^{d_N}$.
\end{lm}

\begin{proof}
\begin{itemize}
\item $\frac{N(N-1)}{4}-d_\tau(\sigma,\sigma')=\frac{1}{2}\sum_{i<j}\mathds{1}_{\sigma(i)<\sigma(j),\;\sigma'(i)<\sigma'(j)}+\mathds{1}_{\sigma(i)>\sigma(j),\;\sigma'(i)>\sigma'(j)}- \frac{1}{2}\sum_{i<j}\mathds{1}_{\sigma(i)<\sigma(j),\;\sigma'(i)>\sigma'(j)}+\mathds{1}_{\sigma(i)>\sigma(j),\;\sigma'(i)<\sigma'(j)}=\langle\Phi(\sigma),\Phi(\sigma') \rangle$ where $\Phi(\sigma)\in \R^{\frac{N(N-1)}{2}}$ is defined by
$
\Phi(\sigma)_{i,j}:=\frac{1}{\sqrt{2}}(\mathds{1}_{\sigma(i)>\sigma(j)}-\mathds{1}_{\sigma(i)<\sigma(j)}),
$
for all $1\leq i<j\leq N$.
\item $N-d_H(\sigma,\sigma')=\sum_{i=1}^N\mathds{1}_{\sigma(i)=\sigma(j)}=\langle\Phi(\sigma),\Phi(\sigma') \rangle$ where $\Phi(\sigma) \in \mathcal{M}_N(\R)$ is defined by  $ \Phi(\sigma):=(\mathds{1}_{\sigma(i)=j})_{i,j}$,
\item $N^2-d_S(\sigma,\sigma')=\sum_{i=1}^N \min(\sigma(i),\sigma'(i))+ N- \max(\sigma(i),\sigma'(i))=\langle\Phi(\sigma),\Phi(\sigma') \rangle$ where $\Phi(\sigma) \in \mathcal{M}_N(\R)^2$ is defined by
\begin{eqnarray*}
\Phi(\sigma)_{i,j,1}:=\left\{ \begin{array}{ll}
1 & \text{if }j\leq \sigma(i)\\
0 & \text{otherwise},
\end{array} \right. \;\;\;\;\;\; \Phi(\sigma)_{i,j,2}:=\left\{ \begin{array}{ll}
0 & \text{if }j< \sigma(i)\\
1 & \text{otherwise}.
\end{array} \right.
\end{eqnarray*}
\end{itemize}
\end{proof}

{\bf Proof of Proposition \ref{pos}}
\begin{proof}
Let us prove that $d$ is a definite negative kernel, that is, for all $c_1,...,c_k\in \R$ such that $\sum_{i=1}^k c_i=0$, we have $\sum_{i,j=1}^k c_i c_j d(\sigma_i,\sigma_j)\leq 0$.
Let $c_1,...,c_k\in \R$ such that $\sum_{i=1}^k c_i=0$ and let $\sigma_1,...,\sigma_k\in S_N$. We have
\begin{eqnarray*}
\sum_{i,j=1}^k c_i c_j d(\sigma_i,\sigma_j)&=& C_N\sum_{i,j=1}^k c_i c_j - \sum_{i,j=1}^k c_i c_j \langle\Phi(\sigma_i),\Phi(\sigma_j)\rangle \leq 0,
\end{eqnarray*}
as $ C_N\sum_{i,j=1}^k c_i c_j=C_N\left(\sum_{i=1}^Nc_i\right)^2$ is equal to 0.
So, $d$ is a negative definite kernel. Hence $d^{\theta_3}$ is a definite negative kernel for all $\theta_3\in [0,1]$ (see for example Property 21.5.9 in \cite{rachev2013methods}).
The function $F:t\mapsto \theta_2\exp(-\theta_1 t)$ is completely monotone, thus, using Schoenberg's theorem (see \cite{berg84harmonic} for the definitions of these notions and Schoenberg's theorem), $K_{\theta_1,\theta_2,\theta_3}$ is a positive definite kernel.
\end{proof}

{\bf Proof of Proposition \ref{prop_distance_partial}}
\begin{proof}
Let us write, with the notation of Lemma \ref{lm1},
\begin{equation}
\Phi_{\mbox{avg}}:R\longmapsto \frac{1}{|E_R|}\sum_{\sigma \in E_R}\Phi(\sigma).
\end{equation}
Then, 
\begin{eqnarray*}
C_N-d_{\mbox{avg}}(R,R')&=&C_N-\frac{1}{|E||E'|}\sum_{\sigma \in E_R}\sum_{\sigma \in E_{R'}}d(\sigma,\sigma')\\
&=&\frac{1}{|E_R||E_{R'}|}\sum_{\sigma \in E_R}\sum_{\sigma \in E_{R'}}C_N-d(\sigma,\sigma')\\
&=&\frac{1}{|E_R||E_{R'}|}\sum_{\sigma \in E_R}\sum_{\sigma \in E_{R'}}\langle \Phi(\sigma), \Phi(\sigma')\rangle \\
&=&\langle \Phi_{\mbox{avg}}(R),\Phi_{\mbox{avg}}(R') \rangle.
\end{eqnarray*}
Thus,
\begin{eqnarray*}
d_{\mbox{partial}}(R,R')&=& d_{\mbox{avg}}(R,R')-\frac{1}{2}d_{\mbox{avg}}(R,R)-\frac{1}{2} d_{\mbox{avg}}(R',R')\\
&=&\frac{1}{2}\left[ \left( C_N-d_{\mbox{avg}}(R,R)\right)+ \left( C_N-d_{\mbox{avg}}(R',R')\right)-2 \left( C_N-d_{\mbox{avg}}(R,R')\right) \right]\\
&=&\frac{1}{2}\left( \|\Phi_{\mbox{avg}}(R)\|^2+\|\Phi_{\mbox{avg}}(R')\|^2-2 \langle \Phi_{\mbox{avg}}(R),\Phi_{\mbox{avg}}(R')\rangle \right)\\
&=&\|\Phi_{\mbox{avg}}(R)-\Phi_{\mbox{avg}}(R')\|^2.
\end{eqnarray*}
\end{proof}

{\bf Proof of Proposition \ref{prop:positive_partial}}
\begin{proof}
Let us prove that $d_{\mbox{partial}}$ is a definite negative kernel.
We define
\begin{equation}
D_{\mbox{avg}}(R,R'):=\Phi_{\mbox{avg}}(R)^T\Phi_{\mbox{avg}}(R').
\end{equation}
Let $(c_1,...,c_k)\in \R^k$ such that $\sum_{i=1}^k c_i=0$.
We have
\begin{eqnarray*}
\sum_{i,j=1}^k c_i c_j d_{\mbox{partial}}(R_i,R_j)&=& \sum_{i,j=1}^k c_i c_j \left[d_{\mbox{avg}}(R_i,R_j)-\frac{1}{2}d_{\mbox{avg}}(R_i,R_i)-\frac{1}{2} d_{\mbox{avg}}(R_j,R_j)\right]\\
&=& \sum_{i,j=1}^k c_i c_j d_{\mbox{avg}}(R_i,R_j)-\frac{1}{2}\sum_{i=1}^k c_i d_{\mbox{avg}}(R_i,R_i) \sum_{j=1}^k c_j\\
&& -\frac{1}{2}\sum_{j=1}^k c_j d_{\mbox{avg}}(R_j,R_j) \sum_{i=1}^k c_i\\
&=&\sum_{i,j=1}^k c_i c_j d_{\mbox{avg}}(R_i,R_j)\\
&=&\sum_{i,j=1}^k c_i c_j\left[ C_N-D_{\mbox{avg}}(R_i,R_j)\right]\\
&=&-\sum_{i,j=1}^k c_i c_jD_{\mbox{avg}}(R_i,R_j)\\
& \leq & 0.
\end{eqnarray*}
So, $d_{\mbox{partial}}$ is a definite negative kernel, and we may conclude as in the proof of Proposition \ref{pos}.

\end{proof}

{\bf Proof of Proposition \ref{prop:computation}}
\begin{proof}
Assume that $\sigma$ (resp. $ \sigma'$) is a uniform random variable of $E_I$ (resp. $E_{I'}$). We have to compute $\E(d(\sigma,\sigma'))=d_{\mbox{avg}}(I,I')$ for the three distances: Kendall's tau, Hamming and Spearman's footrule.

First, we compute $\E(d_\tau(\sigma,\sigma'))$. Following the proof of Lemma 3.1 of \cite{fagin_comparing_2003}, we have 
\begin{equation*}
\E(d_{\tau}(\sigma,\sigma'))=\sum_{a<b}\E(K_{a,b}(\sigma,\sigma')),
\end{equation*}
with
\begin{equation*}
K_{a,b}(\sigma,\sigma')=\mathds{1}_{(\sigma(a)<\sigma(b),\sigma'(a)>\sigma'(b)) \text{ or } (\sigma(a)>\sigma(b),\sigma'(a)<\sigma'(b))}.
\end{equation*}
We now compute $\E(K_{a,b}(\sigma,\sigma'))$ for $(a,b)$ in different cases. Let us write $J:=\{j_1,\cdots,j_p\}$ and we keep the notation $I$ (resp. $I'$) for the set $\{i_1,...,i_k\}$ (resp. $\{i_1',...,i_k'\}$). In this way, we have $I=J\sqcup \tilde{I}$ and $I'=J\sqcup \tilde{I'}$.
\begin{enumerate}
\item Consider the case where $a$ and $b$ are in $J$. There exists $l$ and $l' \in [1:p]$ such that $a=j_l$ and $b=j_{l'}$. Then
\begin{equation*}
K_{a,b}(\sigma,\sigma')=\mathds{1}_{(c_{j_l}<c_{j_{l'}},c'_{j_l}>c'_{j_{l'}})\text{ or }(c_{j_l}>c_{j_{l'}},c'_{j_l}<c'_{j_{l'}})}.
\end{equation*} 
Thus, the total contribution of the pairs in this case is 
\begin{equation*}
\sum_{1\leq l<l' \leq p }\mathds{1}_{(c_{j_l}<c_{j_{l'}},c'_{j_l}>c'_{j_{l'}})\text{ or }(c_{j_l}>c_{j_{l'}},c'_{j_l}<c'_{j_{l'}})}.
\end{equation*} 
\item Consider the case where $a$ and $b$ both appear in one top-$k$ partial ranking (say $I$) and exactly one of $i$ or $j$, say $i$ appear in the other top-$k$ partial ranking. Let us call $P_2$ the set of $(a,b)$ such that $a<b$ and $(a,b)$ is in this case. We have
\begin{eqnarray*}
\sum_{ (a,b)\in P_2}K_{a,b}(\sigma,\sigma')=\sum_{\substack{a\in J, \\ b \in \tilde{I}}}K_{a,b}(\sigma,\sigma')+\sum_{\substack{a\in J,\\ b \in \tilde{I'}}}K_{a,b}(\sigma,\sigma')
\end{eqnarray*}
Let us compute the first sum. Recall that $\tilde{I}=\{i_{u_1},...,i_{u_r}\}$.
\begin{eqnarray*}
\sum_{\substack{a\in J, \\ b \in \tilde{I}}}K_{a,b}(\sigma,\sigma')&=& \sum_{b \in \tilde{I}} \sum_{a \in J}K_{a,b}(\sigma,\sigma')\\
&=&\sum_{b \in \tilde{I}} \# \{ a \in J,\;\sigma(a)>\sigma(b)\} \\
&=&\sum_{l=1}^r\#\{a \in J,\;\sigma(a)>\sigma(i_{u_l}) \}
\end{eqnarray*}
We order $u_1,\cdots,u_r$ such that $u_1<\cdots<u_r$. Let $l\in [1:r]$. Remark that $\sigma(i_{u_l})=u_l$. We have $\#\{a \in I,\;\sigma(a)>u_l \}=k-u_l$ and $\#\{a \in \tilde{I},\;\sigma(a)>u_l\}=r-l$, thus $\#\{a \in J,\;\sigma(a)>u_l \}=k-u_l-r+l$. Then,
\begin{equation*}
\sum_{\substack{a\in J, \\ b \in \tilde{I}}}K_{a,b}(\sigma,\sigma')= r\left(k+\frac{1-r}{2}\right)-\sum_{l=1}^r u_l.
\end{equation*}
Likewise, we have
\begin{equation}
\sum_{\substack{a\in J, \\ b \in \tilde{I'}}}K_{a,b}(\sigma,\sigma')= r\left(k+\frac{1-r}{2}\right)-\sum_{l=1}^r u'_l.
\end{equation}
Finally, the total contribution of the pairs in this case is 
\begin{equation*}
r(2k+1-r)-\sum_{j=1}^r(u_j+u'_j).
\end{equation*}
\item Consider the case where $a$, but not $b$, appears in one top-$k$ partial ranking (say $I$), and $b$, but not $a$, appears in the other top-$k$ partial ranking ($I'$). Then $K_{a,b}(\sigma,\sigma')=1$ and the total contribution of these pairs is $r^2$.
\item Consider the case where $a$ and $b$ do not appear in the same top-$k$ partial ranking (say $I$). It is the only case where $K_{a,b}(\sigma,\sigma')$ is a non constant random variable. First, we show that in this case, $\E(K_{a,b}(\sigma,\sigma'))=1\slash2$. Assume for example that $I$ does not contain $a$ and $b$. Let $(a\;b)$ be the transposition which exchanges $a$ and $b$ and does not change the other elements. We have
\begin{equation*}
\{\pi \in E_I,\pi(a)<\pi(b)\}=(a\;b)\{\pi \in E_I,\pi(a)>\pi(b) \}.
\end{equation*} 
Thus, there are as many $\pi \in E_I$ such that $\pi(a)<\pi(b)$ as there are $\pi \in E_I$ such that $\pi(a)>\pi(b)$. That proves that $\E(K_{a,b}(\sigma,\sigma'))=1\slash2$.

Then, the total distribution of the pairs in this case is
\begin{equation*}
\frac{1}{2}\left[
\begin{pmatrix} |I^c|\\
2 \end{pmatrix}
+
\begin{pmatrix} |I^{\prime c}|\\2 \end{pmatrix}
-
\begin{pmatrix}
|I^c\cap I^{\prime c}| \\2
\end{pmatrix}
\right]= \begin{pmatrix} N-k\\2 \end{pmatrix}-\frac{1}{2}\begin{pmatrix}
m\\2
\end{pmatrix}.
\end{equation*}
\end{enumerate}
That concludes the computation for the Kendall's tau distance.\\

To compute $\E(d_H(\sigma,\sigma'))$, it suffices to see that
\begin{eqnarray*}
\E(d_H(\sigma,\sigma'))&=&\E\left(\sum_{i=1}^n \mathds{1}_{\sigma(i)\neq \sigma'(i)}\right)\\
&=&\sum_{l=1}^p\mathds{1}_{c_{j_l}\neq c'_{j_l}}+\E\left( \sum_{i\neq I\cup I'} \mathds{1}_{\sigma(i)\neq \sigma'(i)} \right)\\
&&+\E\left( \sum_{j =1}^r \mathds{1}_{u_j\neq \sigma'(i_{u_j})} \right)+\E\left( \sum_{j =1}^r \mathds{1}_{\sigma(i_{u'_j})\neq u'_j} \right)\\
&=&\sum_{l =1}^p\mathds{1}_{c_{j_l}\neq c'_{j_{l}}}+m\frac{N-k-1}{N-k}+ 2r.
\end{eqnarray*}

Finally, let compute $\E(d_S(\sigma,\sigma'))$. First, we define
\begin{itemize}
\item $ A_c:=\sum_{j=1}^p|c_j-c'_j| $ 
\item $A_u(\sigma'):=\sum_{j=1}^r|u_j-\sigma'(i_{u_j})|$
\item $A_{u'}(\sigma):=\sum_{j=1}^r |\sigma(i'_{u'_j})-u'_j| $
\item $ R(\sigma,\sigma'):=\sum_{i\neq I\cup I'}|\sigma(i)-\sigma'(i)|. $
\end{itemize}
We have
\begin{eqnarray*}
\E(d_S(\sigma,\sigma'))=\E(A_c)+\E(A_u(\sigma'))+\E(A_{u'}(\sigma))+\E( R(\sigma,\sigma')).
\end{eqnarray*}
It remains to compute all the expectations appearing here.
\begin{enumerate}
\item $\E(A_c)=A_c$. 
\item $\E(A_u(\sigma'))=\sum_{j=1}^r \E(|u_j-\sigma'(i_{u_j})|)$. If $\sigma'$ is uniform on $E_{I'}$, then $\sigma'(i_{u_j})$ is uniform on $[k+1:N]$ so:
\begin{eqnarray*}
\E(|u_j-\sigma'(i_{u_j})|)=\E(\sigma'(i_{u_j})-u_j)=\frac{N+k+1}{2}-u_j.
\end{eqnarray*}
Finally,
\begin{equation}
\E(A_u(\sigma'))=r\frac{N+k+1}{2}-\sum_{j=1}^r u_j.
\end{equation}
\item $\E(A_{u'}(\sigma))=r\frac{N+k+1}{2}-\sum_{j=1}^r u'_j$.
\item $\E(R(\sigma,\sigma'))=\sum_{i\neq I\cup I'}\E(|\sigma(i)-\sigma'(i)|). $ $\sigma(i)$ and $\sigma'(i)$ are independent uniform random variables on $[k+1:N]$. 
\begin{eqnarray*}
\E(|\sigma(i)-\sigma'(i)|)&=&\sum_{j=1}^{N-k-1} j \PP(|\sigma(i)-\sigma'(i)|=j)\\
&=&\sum_{j=1}^{N-k-1} j 2\frac{N-k-j}{(N-k)^2}.
\end{eqnarray*}
Then
\begin{eqnarray*}
\E(R(\sigma,\sigma'))&=&\frac{2m}{(N'+1)^2}\sum_{j=1}^{N'}j(N'+1-j)\\
&=& \frac{2m}{(N'+1)^2}\left(\frac{N'(N'+1)^2}{2}  -\frac{N'(N'+1)(2N'+1)}{6}\right)\\
&=&mN'-\frac{mN'(2N'+1)}{3(N'+1)}.
\end{eqnarray*}
\end{enumerate}
That concludes the proof of Proposition \ref{prop:computation}. 
\end{proof}

{\bf Proof of Proposition \ref{prop::computationH}}
\begin{proof}
We define
\begin{eqnarray*}
a_j^\gamma(\sigma,\sigma') &:=& |\{i \in [1:N],\; \sigma(i)\in \Gamma_j,\sigma'(i)\in \Gamma_j,\; \sigma(i)\neq \sigma'(i) \}|,\\
b_{j,l}^\gamma(\sigma,\sigma')&:=& |\{i \in [1:N],\; \sigma(i)\in \Gamma_j,\sigma'(i)\in \Gamma_l,\;j \neq l  \}|.
\end{eqnarray*}
Now, assume that $\sigma,\sigma'\sim \mathcal{U}(S_\gamma)$ and $\sigma_j,\sigma'_j\sim \mathcal{U}(S_{\gamma_j})$. We have
\begin{eqnarray*}
\E\left(d_{H}(\sigma,\sigma')  \right)&=& \E \left( \sum_{j,l=1}^k b_{j,l}^\gamma(\sigma\pi_1,\sigma'\pi_2)+\sum_{j=1}^k a_j^\gamma(\sigma\pi_1,\sigma' \pi_2)   \right)\\
&=&\sum_{j,l=1}^k b_{j,l}^\gamma(\pi_1,\pi_2)+\sum_{j=1}^k | \{i,\; \pi_1(i), \pi_2(i)\in \Gamma_j\}| \frac{\gamma_j-1}{\gamma_j} \\
&=&|\{i,\; \Gamma(\pi_1(i))\neq \Gamma(\pi_2(i))\}|+\sum_{j=1}^k\frac{\gamma_j}{n}(\gamma_j-1).
\end{eqnarray*}
\end{proof}

\section{Proofs for Section \ref{s:GPrank}}\label{s:AB}
In the following, let us write $\|.\|$ for the operator norm (for a linear mapping of $\R^n$ with the Euclidean norm) of a squared matrix of size $n$, $\|.\|_F$ for its Frobenius norm defined by $\|M\|_F^2 := \sum_{i,j=1}^n m_{ij}^2$ for $M= (m_{ij})_{1 \leq i,j \leq n} \in \mathcal{M}_n(\R)$, and let us define the norm $|\cdot |$ by $|M|^2:=\frac{1}{n}\|M\|_F^2$. 
We remark that, when $M$ is a symmetric positive definite matrix, $\|M\|$ is its largest eigenvalue. In this case, we may also write $\|M\|=\lambda_{\max}(M)$, where $\lambda_{\max}(M)$ has been defined in Section \ref{s:asymtoresults2} and is the largest value of $M$.
For a vector $u$ of $\R^d$, for $d\in \R$, recall that $\|u\|$ is the Euclidean norm of $u$.

The proofs of Theorems \ref{consis}, \ref{gaussien} and \ref{pred} are given in Appendix \ref{section_proof_1}, \ref{section_proof_2} and \ref{section_proof_3} respectively. These proofs are based on Lemmas \ref{V.4} to \ref{V.8}, that are stated and proved in Appendix \ref{section_lemmas}.
The proofs of these lemmas are new.
Then, having at hand the lemmas, the proof of the theorems follows \cite{bachoc_gaussian_2017}. 
We write all the proofs to be self-contained. 
\subsection{Lemmas} \label{section_lemmas}
The following Lemmas are useful for the proofs of Theorems \ref{consis}, \ref{gaussien} and \ref{pred}.

\begin{lm}\label{V.4}
The eigenvalues of $R_\theta$ are lower-bounded by $\theta_{3,\min}>0$ uniformly in $n$, $\theta$ and $\Sigma$.
\end{lm}

\begin{proof}
$R_\theta$ is the sum of a symmetric positive matrix and $\theta_3 I_n$. Thus, the eigenvalues are lower-bounded by $\theta_{3,\min}$.
\end{proof}

\begin{lm}\label{A.5}
For all $\alpha = (\alpha_1,\alpha_2,\alpha_3) \in \N^3$, with $|\alpha| = \alpha_1+\alpha_2 + \alpha_3$ and with $\partial \theta^\alpha = \partial \theta_1^{\alpha_1} \partial \theta_2^{\alpha_2} \partial \theta_3^{\alpha_3}$, the eigenvalues of $\frac{\partial^{|\alpha|} R_\theta}{\partial \theta^\alpha}$ are upper-bounded uniformly in $n$, $\theta$ and $\Sigma$.
\end{lm}

\begin{proof}
It is easy to prove when $\alpha_1=\alpha_2=0$. Indeed:
\begin{enumerate}
\item If $\alpha_3=0$, then $\lambda_{\max}\left( R_{\theta} \right)\leq \lambda_{\max}\left( (K_{\theta_1,\theta_2}(\sigma_i,\sigma_j))_{i,j}  \right)+\theta_{3,\max}$ and we show that $\lambda_{\max} \left( K_{\theta_1,\theta_2}(\sigma_i,\sigma_j)_{i,j} \right)$ is uniformly bounded using Gershgorin circle theorem (\cite{izvestija/gerschgorin31}).
\item If $\alpha_3=1$, then $\frac{\partial^{|\alpha|} R_\theta}{\partial \theta^\alpha}=I_n$.
\item If $\alpha_3>1$, then $\frac{\partial^{|\alpha|} R_\theta}{\partial \theta^\alpha}=0$.
\end{enumerate}
Then, we suppose that $(\alpha_1,\alpha_2)\neq (0,0)$. Thus, 
$$
\frac{\partial^{|\alpha|} R_\theta}{\partial \theta^\alpha}=\frac{\partial^{|\alpha|} \left(K_{\theta_1,\theta_2}(\sigma_i,\sigma_j)_{i,j}\right)}{\partial \theta^\alpha}.
$$
It does not depend on $\alpha_3$ so we can assume that $\alpha\in \N^2$. We have
\begin{equation}
\left|\frac{\partial^{|\alpha|} K_{\theta_1,\theta_2}(\sigma,\sigma')}{\partial \theta^\alpha}\right|\leq \max(1,\theta_{2,\max})d(\sigma,\sigma')^{\alpha_1} e^{-\theta_{1,\min} d (\sigma,\sigma')}.
\end{equation}
We conclude using Gershgorin circle theorem \cite{izvestija/gerschgorin31}.
\end{proof}

\begin{lm} \label{lem:extra}
Uniformly in $\Sigma$,
\begin{equation}\label{V.5}
\forall \alpha>0,\;\;\underset{n\rightarrow+\infty}{\lim \inf} \inf_{\|\theta-\theta^*\|\geq \alpha}\frac{1}{n}\sum_{i,j=1}^n(R_{\theta,i,j}-R_{\theta^*,i,j})^2>0.
\end{equation}
\end{lm}

\begin{proof}
Let $N$ be the norm on $\R^3$ defined by
\begin{equation}
N(x):=\max(4c\theta_{2,\max} |x_1|,2|x_2|,|x_3|),
\end{equation}
with $c$ as in Condition 2.
Let $\alpha>0$. We want to find a positive lower-bound over $\theta\in \Theta\setminus B_N(\theta^*,\alpha)$, where $B_N(\theta^*,\alpha)$ is the ball with the norm $N$ of center $\theta^*$ and radius $\alpha$, of
\begin{equation}
\frac{1}{n}\sum_{i,j=1}^n(R_{\theta,i,j}-R_{\theta^*,i,j})^2.
\end{equation}
Let $\theta\in \Theta\setminus B_N(\theta^*,\alpha)$.
\begin{enumerate}
\item Consider the case where $|\theta_1-\theta_1^*|\geq \alpha\slash(4c\theta_{2,\max})$. Let $k_\alpha\in \N$ be the first integer such that
\begin{equation}
k_\alpha^\beta \geq 4c\theta_{2,\max}\frac{2+ \ln(\theta_{2,\max})-\ln(\theta_{2,\min})}{\alpha}.
\end{equation}
Then, for all $i\in \N^*$,
\begin{eqnarray*}
\left| \frac{(\theta_1^*-\theta_1 )d(\sigma_i,\sigma_{i+k_\alpha})+\ln(\theta_2)-\ln(\theta_2^*)}{2} \right|\geq 1.
\end{eqnarray*} 
For all $n\geq k_\alpha$,
\begin{eqnarray*}
&&\frac{1}{n}\sum_{i,j=1}^n(R_{\theta,i,j}-R_{\theta^*,i,j})^2\\
&\geq & \frac{1}{n}\sum_{i=1}^{n-k_\alpha}(R_{\theta,i,i+k_\alpha}-R_{\theta^*,i,i+k_\alpha})^2\\
& \geq & \frac{1}{n}\sum_{i=1}^{n-k_\alpha} e^{-2\theta_{1,\max}ck_\alpha+2\ln(\theta_{2,\min})}4\sinh^2\left( \frac{(\theta_1^*-\theta_1 )d(\sigma_i,\sigma_{i+k_\alpha})+\ln(\theta_2)-\ln(\theta_2^*)}{2} \right)\\
&\geq & C_{1,\alpha}\frac{n-k_\alpha}{n},
\end{eqnarray*}
where we write $C_{1,\alpha}=e^{-2\theta_{1,\max}ck_\alpha+2\ln(\theta_{2,\min})}4\sinh^2(1)$.
\item  Consider the case where $|\theta_1-\theta_1^*|\leq \alpha\slash(4c\theta_{2,\max})$.
\begin{enumerate}
\item If $|\theta_2-\theta_2^*|\geq \alpha\slash 2$, we have
\begin{eqnarray*}
\frac{|\theta_1-\theta_1^*|}{2}d(\sigma_i,\sigma_{i+1})& < &\frac{\alpha}{8\theta_{2,\max}}\\
&=&\frac{\alpha}{4\theta_{2,\max}}-\frac{\alpha}{8\theta_{2,\max}}\\
& \leq & \frac{|\ln(\theta_2^*)-\ln(\theta_2)|}{2}-\frac{\alpha}{8\theta_{2,\max}}.
\end{eqnarray*}
Thus,
\begin{equation}
\left| \frac{(\theta_1^*-\theta_1 )d(\sigma_i,\sigma_{i+1})+\ln(\theta_2)-\ln(\theta_2^*)}{2} \right|\geq \frac{\alpha}{8\theta_{2,\max}},
\end{equation}
and we have
\begin{eqnarray*}
&&\frac{1}{n}\sum_{i,j=1}^n(R_{\theta,i,j}-R_{\theta^*,i,j})^2\\
&\geq & \frac{1}{n}\sum_{i=1}^{n-1}(R_{\theta,i,i+1}-R_{\theta^*,i,i+1})^2\\
&\geq & \frac{1}{n}\sum_{i=1}^{n-1}e^{-2\theta_{1,\max}c+2\ln(\theta_{2,\min})}4\sinh^2\left( \frac{\alpha}{8\theta_{2,\max}} \right)\\
&=& C_{2,\alpha}\frac{n-1}{n},
\end{eqnarray*}
where we write $C_{2,\alpha}:=e^{-2\theta_{1,\max}c+2\ln(\theta_{2,\min})}4\sinh^2\left( \frac{\alpha}{8\theta_{2,\max}} \right)$.
\item  If $|\theta_2-\theta_2^*|< \alpha\slash 2$, we have $|\theta_3-\theta_3^*|\geq \alpha$. Thus,
\begin{eqnarray*}
&&\frac{1}{n}\sum_{i,j=1}^n(R_{\theta,i,j}-R_{\theta^*,i,j})^2\\
 &\geq &\frac{1}{n} \sum_{i=1}^{n}(R_{\theta,i,i}-R_{\theta^*,i,i})^2\\
 &=&\frac{1}{n} \sum_{i=1}^{n}(\theta_2+\theta_3-\theta_2^*-\theta_3^*)^2\\
 & \geq & \frac{\alpha^2}{4}.
\end{eqnarray*}
\end{enumerate}
\end{enumerate}
Finally, if we write
\begin{equation}
C_{\alpha}:=\min\left(C_{1,\alpha},C_{2,\alpha},\frac{\alpha^2}{2} \right),
\end{equation}
we have
\begin{equation}
 \inf_{N(\theta-\theta^*)\geq \alpha}\frac{1}{n}\sum_{i,j=1}^n(R_{\theta,i,j}-R_{\theta^*,i,j})^2
\geq 
 \frac{n-k_\alpha}{n} C_\alpha.
\end{equation}
To conclude, by equivalence of norms in $\R^3$, there exists $h>0$ such that $\|.\|_2\leq h N(.)$, thus
\begin{eqnarray}
\underset{n\rightarrow+\infty}{\lim \inf} \inf_{\|\theta-\theta^*\|\geq \alpha}\frac{1}{n}\sum_{i,j=1}^n(R_{\theta,i,j}-R_{\theta^*,i,j})^2 \geq C_{\alpha \slash h}>0.
\end{eqnarray}
\end{proof}

\begin{lm}\label{V.8}
$\forall (\lambda_1,\lambda_2,\lambda_3)\neq (0,0,0)$, uniformly in $\sigma$,
\begin{equation}
\underset{n\rightarrow +\infty}{\liminf}\frac{1}{n}\sum_{i,j=1}^n \left( \sum_{k=1}^3 \lambda_i\frac{\partial}{\partial \theta_k}R_{\theta^*,i,j}   \right)^2>0.
\end{equation}
\end{lm}

\begin{proof}
We have
\begin{eqnarray*}
\frac{\partial}{\partial \theta_1}R_{\theta^*,i,j}&=&- \theta_2^* d(\sigma_i,\sigma_j) e^{-\theta_1^* d(\sigma_i,\sigma_j)},\\
\frac{\partial}{\partial \theta_2}R_{\theta^*,i,j}&=&e^{-\theta_1^* d(\sigma_i,\sigma_j)} ,\\
\frac{\partial}{\partial \theta_3}R_{\theta^*,i,j}&=&\mathds{1}_{i=j}.
\end{eqnarray*}
Let $(\lambda_1,\lambda_2,\lambda_3)\neq (0,0,0)$. We have
\begin{eqnarray*}
&&\frac{1}{n}\sum_{i,j=1}^n \left( \sum_{k=1}^3 \lambda_k\frac{\partial}{\partial \theta_k}R_{\theta^*,i,j} \right)^2\\
&=&\frac{1}{n}\sum_{i\neq j=1}^n \left( \sum_{k=1}^2 \lambda_k\frac{\partial}{\partial \theta_k}R_{\theta^*,i,j} \right)^2+(\lambda_2+\lambda_3)^2 \\
&=&\frac{1}{n}\sum_{i\neq j=1}^n e^{- 2\theta_1^* d(\sigma_i,\sigma_j)} \left( \lambda_2 - \lambda_1 \theta_2^* d(\sigma_i,\sigma_j)  \right)^2+(\lambda_2+\lambda_3)^2.
\end{eqnarray*}
If $\lambda_1 \neq 0$, then for conditions 1 and 2, we can find $\epsilon >0 , \tau >0, k \in \mathbb{Z}$ so that for $|i - j| = k$, we have $\left( \lambda_2 - \lambda_1  d(\sigma_i,\sigma_j)  \right)^2 \geq \epsilon$ and $e^{- 2\theta_1^* d(\sigma_i,\sigma_j)} \geq \tau$. This concludes the proof in the case $\lambda_1 \neq 0$. The proof in the case $\lambda_1 = 0$ can then be obtained by considering the pairs $(j,j+1)$ in the above display.

\end{proof}

With these lemmata we are ready to prove the main asymptotic results.

\subsection{Proof of Theorem \ref{consis}} \label{section_proof_1}
\begin{proof}
\underline{Step 1:}
It suffices to prove that, uniformly in $\Sigma$ where we recall that $\Sigma=(\sigma_1,\cdots,\sigma_n)\in S_{N_n}$,
\begin{equation}\label{E1}
\mathbb{P} \left(
\left.
\sup_{\theta\in \Theta} 
\left|
 (L_\theta-L_{\theta^*})-(\E(L_\theta | \Sigma)-\E(L_{\theta^*} | \Sigma) )
  \right| 
 \geq \epsilon 
\right|
 \Sigma 
 \right)
 \to_{n \to \infty} 0,
\end{equation}
and that there exists $a > 0$ such that
\begin{equation}\label{E2}
\E(L_\theta|\Sigma)-\E(L_{\theta^*}|\Sigma)\geq a \frac{1}{n}\sum_{i,j=1}^n(K_\theta(\sigma_i,\sigma_j)-K_{\theta^*}(\sigma_i,\sigma_j))^2.
\end{equation}
Indeed, by contradiction, assume that we have \eqref{E1}, \eqref{E2} but not the consistency of the maximum likelihood estimator.
We will use a subsequence argument and thus we explicit here the dependence on $n$ of the likelihood function (resp. the estimated parameter) writing it $L_{n,\theta}$ (resp. $\widehat{\theta}_n$). Then,
\begin{equation}
\exists \epsilon>0,\;\exists \alpha >0,\;\forall n \in \N,\; \exists m_n\geq n,\;\PP(\|\widehat{\theta}_{m_n}-\theta^*\|\geq \epsilon)\geq \alpha.
\end{equation}
Thus, with probability at least $\alpha$, we have, for all $n$:\\
$\|\widehat{\theta}_{m_n}-\theta^*\|\geq \epsilon$ thus $\inf_{\|\theta-\theta^*\|\geq  \epsilon}L_{m_n,\theta}\leq L_{m_n,\widehat{\theta}_{m_n}}$.\\
However, by definition of $\widehat{\theta}_{m_n}$, we have $ L_{m_n,\widehat{\theta}_{m_n}}\leq L_{m_n,\theta^*}$.\\
Thus: $\inf_{\|\theta-\theta^*\|\geq  \epsilon}L_{m_n,\theta}\leq  L_{m_n,\theta^*}$.\\
Finally, with probability at least $\alpha$:
\begin{eqnarray*}
 0&\geq &\inf_{\| \theta-\theta^* \|\geq \epsilon}\left( L_{m_n,\theta}-L_{m_n,\theta^*} \right)\\
 &\geq & \inf_{\| \theta-\theta^* \|\geq \epsilon}\E \left(L_{m_n,\theta}-L_{m_n,\theta^*}\middle| \Sigma \right)\\
 & & -\sup_{\| \theta-\theta^* \|\geq \epsilon}\left|(L_{m_n,\theta}-L_{m_n,\theta^*})-(\E(L_{m_n,\theta}-L_{m_n,\theta^*}|\Sigma )\right|\\
 & \geq &  \inf_{\| \theta-\theta^* \|\geq \epsilon} a |R_\theta-R_{\theta^*}|^2 -\sup_{\| \theta-\theta^* \|\geq \epsilon}\left|(L_{m_n,\theta}-L_{m_n,\theta^*})-(\E(L_{m_n,\theta}-L_{m_n,\theta^*}|\Sigma )\right|, 
\end{eqnarray*}
using \eqref{E2}, which is contradicted using \eqref{E1} and recalling Lemma \ref{lem:extra}. In the above display, we recall that the norm $| \cdot |$ for matrices is defined at the beginning of Appendix \ref{s:AB}.
It remains to prove \eqref{E1} and \eqref{E2}. \\

\underline{Step 2:} We prove \eqref{E1}.\\
For all $\sigma \in (S_{N_n})^n$ satisfying Conditions 1 and 2, recalling that $\|\cdot \|_F^2$ and $\|\cdot \|$ are defined at the beginning of Appendix \ref{s:AB},
\begin{eqnarray*}
\V(L_\theta| \Sigma=\sigma)&=&\V\left(\frac{1}{n}\det(R_\theta)+\frac{1}{n}y^T R_\theta^{-1} y   | \Sigma=\sigma\right)\\
&=&\frac{2}{n^2}\mbox{Tr}(R_{\theta^*} R_{\theta}^{-1}R_{\theta^*} R_{\theta}^{-1})\\
&= &\frac{2}{n^2}\left \| R_{\theta^*}^{\frac{1}{2}}R_{\theta^*}^{-1}R_{\theta^*}^{\frac{1}{2}}\right\|_F^2.
\end{eqnarray*}
The previous display holds true because, with $R_{\theta^*}^{\frac{1}{2}}$, the unique matrix square root of $R_{\theta^*}$, we have
$$
\mbox{Tr}(R_{\theta^*} R_{\theta}^{-1}R_{\theta^*} R_{\theta}^{-1})=\mbox{Tr}\left[ \left( R_{\theta^*}^{\frac{1}{2}}R_{\theta^*}^{-1}R_{\theta^*}^{\frac{1}{2}}\right)^T  \left( R_{\theta^*}^{\frac{1}{2}}R_{\theta^*}^{-1}R_{\theta^*}^{\frac{1}{2}}\right) \right]=\left \| R_{\theta^*}^{\frac{1}{2}}R_{\theta^*}^{-1}R_{\theta^*}^{\frac{1}{2}}\right\|_F^2.
$$
Then, we have the relation $\| AB\|_F^2\leq \|A\|^2 \|B\|_F^2$.
Thus, we have
\begin{eqnarray*}
\V(L_\theta| \Sigma=\sigma)&\leq & \frac{2}{n^2} \left \| R_{\theta^*}^{\frac{1}{2}}R_{\theta^*}^{-1}R_{\theta^*}^{\frac{1}{2}}\right\|_F^2\\
& \leq & \frac{2}{n^2} \left \| R_{\theta^*}^{\frac{1}{2}}\right\|^2 \left\|R_{\theta^*}^{-1}\right\|_F^2 \left\|R_{\theta^*}^{\frac{1}{2}}\right\|^2\\
 & \leq & \frac{2}{n^2}\|R_{\theta^*}^{\frac{1}{2}}\|^4  n\| R_{\theta}^{-1}\|^2\\
&\leq & \frac{2}{n} \|R_{\theta^*}\|^2  \| R_{\theta}^{-1}\|^2.
\end{eqnarray*}
Hence, we have $$
\V(L_\theta| \Sigma=\sigma) \leq \frac{C}{n},
$$
where $C>0$ is some constant independent on $n,\theta$ and $\Sigma$, using Lemmas \ref{V.4} and \ref{A.5} (Lemmas \ref{V.4} to \ref{V.8} are stated and proved in Appendix \ref{section_lemmas}).
Thus, for all $\sigma$,
$$\V(L_\theta| \Sigma=\sigma)=\E\left( (L_\theta-\E(L_\theta|\Sigma=\sigma))^2|\Sigma=\sigma   \right) \leq  \frac{C}{n},$$
so
$$\E\left( (L_\theta-\E(L_\theta|\Sigma=\sigma))^2 \right) \leq  \frac{C}{n},$$
thus $L_\theta - \E(L_\theta|\Sigma)=o_\PP(1).$
Let us write $z:=R_\theta^{-\frac{1}{2}}y$. For $i \in \{1,2,3\}$,
\begin{eqnarray*}
\sup_{\theta\in \Theta} \left|\frac{\partial L_\theta }{\partial \theta_i}  \right| &=&   \sup_{\theta\in \Theta}  \frac{1}{n}\left( \mbox{Tr}\left(R_{\theta}^{-1}  \frac{\partial R_\theta}{\partial \theta_i}\right)+z^T R_{\theta^*}^\frac{1}{2} R_\theta^{-1} \frac{\partial R_\theta}{\partial \theta_i}R_\theta^{-1}  R_{\theta^*}^\frac{1}{2} z    \right)\\
&\leq &\sup_{\theta\in \Theta} \left(\max\left(\|R_\theta^{-1} \|\left\|   \frac{\partial R_\theta}{\partial \theta_i}  \right\|,\| R_{\theta^*} \| \|R_\theta^{-2}\|  \left\|  \frac{\partial R_\theta}{\partial \theta_i}  \right\|  \right)  \right) \left(1+\frac{1}{n}\|z\|^2  \right).
\end{eqnarray*}
Here, we have used $z^TAz \leq \|z\|^2 \|A\|$ for a symmetric positive definite matrix $A$ , the fact that $\|A B\| \leq \|A \| \|B\|$ for matrices $A$ and $B$, and the fact that, by Cauchy-Schwarz,
$$
\mbox{Tr}(AB) \leq \|A\|_F \|B\|_F \leq n \|A \| \|B\|.
$$
Hence, $\sup_{\theta\in \Theta} \left|\frac{\partial L_\theta }{\partial \theta_i}  \right|$ is bounded in probability conditionally to $\Sigma=\sigma$, uniformly in $\sigma$. Indeed $z\sim \mathcal{N}(0,I_n)$ thus $1\slash n \;\|z\|^2$ is bounded in probability, conditionally to $\Sigma$ and uniformly in $\Sigma$. 

Then $\sup_{i\in [1:3],\theta\in \Theta} \left|\frac{\partial L_\theta }{\partial \theta_i}  \right|$ is bounded in probability.\\
Thanks to the pointwise convergence and the boundedness of the derivatives, we have
\begin{equation}\label{eq_step1}
\sup_{\theta \in \Theta} |L_\theta - \E(L_\theta)|=:r_1,
\end{equation}
where $r_1$ depends on $\Sigma$ and, for all $\varepsilon>0$, $\PP(|r_1|>\varepsilon)\underset{n\to +\infty}{\longrightarrow}0$ uniformly in $\Sigma$.
Hence,
$$
\sup_{\theta \in \Theta} |L_\theta - \E(L_\theta|\Sigma)|+| L_{\theta^*} - \E(L_{\theta^*}|\Sigma)|=:r_2,
$$
where $r_2$ depends on $\Sigma$ and, for all $\varepsilon>0$, $\PP(|r_2|>\varepsilon)\underset{n\to +\infty}{\longrightarrow}0$ uniformly in $\Sigma$.
Now, let us write $D_{\theta,\theta^*}:=\E(L_\theta|\Sigma)-\E(L_{\theta^*}|\Sigma)$. Thanks to \eqref{eq_step1},
\begin{equation}
\sup_{\theta \in \Theta}|L_{\theta}-L_{\theta^*}-D_{\theta,\theta^*}| \leq \sup_\theta |L_\theta - \E(L_\theta|\Sigma)|+| L_{\theta^*} - \E(L_{\theta^*}|\Sigma)|.
\end{equation}
Thus
$$
\sup_{\theta \in \Theta}|L_{\theta}-L_{\theta^*}-D_{\theta,\theta^*}| =:r_3,
$$
where $r_3$ depends on $\Sigma$ and, for all $\varepsilon>0$, $\PP(|r_3|>\varepsilon)\underset{n\to +\infty}{\longrightarrow}0$ uniformly in $\Sigma$.\\

\underline{Step 3:} We prove \eqref{E2}.\\
We have
\begin{eqnarray*}
\E(y^TR_\theta y|\Sigma)=\E(\mbox{Tr}(y^TR_\theta y)|\Sigma)=\E(\mbox{Tr}(R_\theta y y^T)|\Sigma))=\mbox{Tr}(R_\theta \E(y^Ty)).
\end{eqnarray*}
Thus
\begin{equation}
\E(L_\theta|\Sigma)=\frac{1}{n} \ln(\det(R_\theta))+\frac{1}{n}\mbox{\mbox{Tr}}(R_{\theta}^{-1}R_{\theta^*}),
\end{equation}
Let us write $\phi_1(M),\cdots,\phi_n(M)$ the eigenvalues of a symmetric $n\times n$ matrix $M$. We have
\begin{eqnarray*}
D_{\theta,\theta^*}&=&\frac{1}{n}\ln(\det(R_\theta))+\frac{1}{n}\mbox{Tr}(R_\theta^{-1}R_{\theta^*})-\frac{1}{n}\ln(\det(R_{\theta^*}))-1\\
&=&  \frac{1}{n}\left(-\ln\left((\det(R_\theta^{-1})\det(R_{\theta^*})\right)+ \mbox{Tr}(R_{\theta}^{-1}R_{\theta^*})-1  \right)  \\
&=& \frac{1}{n}\left(-\ln\left((\det(R_{\theta^*}^\frac{1}{2} R_{\theta}^{-1}R_{\theta^*}^\frac{1}{2})\right)+ \mbox{Tr}(R_{\theta^*}^\frac{1}{2} R_{\theta}^{-1}R_{\theta^*}^\frac{1}{2})-1  \right)  \\
&=&\frac{1}{n}\sum_{i=1}^n\left(-\ln \left[\phi_i\left(R_{\theta^*}^\frac{1}{2}R_\theta^{-1}R_{\theta^*}^\frac{1}{2}  \right)  \right] +\phi_i\left(R_{\theta^*}^\frac{1}{2}R_\theta^{-1}R_{\theta^*}^\frac{1}{2}  \right)-1   \right).
\end{eqnarray*}
Thanks to Lemmas \ref{A.5} and \ref{lem:extra}, the eigenvalues of $R_{\theta}$ and $R_{\theta}^{-1}$ are uniformly bounded in $\theta$ and $\Sigma$. Thus, there exist $a>0$ and $b>0$ such that for all $\sigma$, $n$ and $\theta$, we have
$$
\forall i,\; a<\phi_i\left(R_{\theta^*}^\frac{1}{2}R_\theta R_{\theta^*}^\frac{1}{2} \right)<b.$$
Let us define $f(t):=-\ln(t)+t-1$. The function $f$ is minimal in $1$ and $f'(1)=0$ and $f''(1)=1$. So there exists $A>0$ such that for all $t\in [a,b]$, $f(t)\geq A(t-1)^2$. Finally:
\begin{eqnarray*}
D_{\theta,\theta^*} & \geq & \frac{A}{n}\sum_{i=1}^n\left(1-\phi_i(R_{\theta^*}^\frac{1}{2}R_\theta^{-1}R_{\theta^*}^\frac{1}{2})  \right)^2\\
& = &  \frac{A}{n}\mbox{Tr}\left[\left(I_n-R_{\theta^*}^\frac{1}{2}R_\theta^{-1}R_{\theta^*}^\frac{1}{2}  \right)^2\right]\\
&=&\frac{A}{n}\mbox{Tr}\left[\left(  R_\theta^{-\frac{1}{2}}(R_\theta-R_{\theta^*}) R_\theta^{-\frac{1}{2}}   \right)^2   \right]\\
&=& \frac{A}{n} \left\| R_\theta^{-\frac{1}{2}}(R_\theta-R_{\theta^*}) R_\theta^{-\frac{1}{2}}   \right\|_F^2,
\end{eqnarray*}
where we have used $\mbox{Tr}(AA^T)=\|A\|_F^2$ for a square matrix $A$. Furthermore, with $\lambda_{\min}(A)$ the smallest eigenvalue of a symmetric matrix $A$, for any squared matrix $B$, we have $\|A B\|_F^2 \geq \lambda_{\min}^2(A)\|B\|^2$. This yields 
\begin{eqnarray*}
D_{\theta,\theta^*} &\geq & \frac{A}{n} \left\| R_\theta-R_{\theta^*} \right\|_F^2 \lambda_{\min}^2\left( R_\theta^{-\frac{1}{2}}\right) \lambda_{\min}^2\left( R_\theta^{-\frac{1}{2}}\right)\\
& \geq & a |R_\theta-R_{\theta^*} |^2,
\end{eqnarray*}
by Lemma \ref{V.4}, writing $a=A \theta_{3,\max}^{-2}$, and recalling that $|A|^2=\frac{1}{n}\|A\|_F^2$ for a matrix $A$.
\end{proof}

\subsection{Proof of Theorem \ref{gaussien}}\label{section_proof_2}

\begin{proof}
First, we prove \eqref{22}. For all $(\lambda_1,\lambda_2,\lambda_3) \in \R^3$ such that $\| (\lambda_1,\lambda_2,\lambda_3)\|=1$, we have
\begin{eqnarray*}
\sum_{i,j=1}^3 \lambda_i \lambda_j(M_{ML})_{i,j}&=&\frac{1}{2n}\mbox{Tr}\left(R_{\theta^*}^{-1}\left( \sum_{i=1}^3 \lambda_i \frac{\partial R_{\theta^*}}{\partial \theta_i}\right)  R_{\theta^*}^{-1}\left( \sum_{j=1}^3 \lambda_j \frac{\partial R_{\theta^*}}{\partial \theta_j}\right) \right)\\
&=&\frac{1}{2n}\mbox{Tr}\left(R_{\theta^*}^{-\frac{1}{2}}\left( \sum_{i=1}^3 \lambda_i \frac{\partial R_{\theta^*}}{\partial \theta_i}\right) R_{\theta^*}^{-\frac{1}{2}}R_{\theta^*}^{-\frac{1}{2}}\left( \sum_{j=1}^3 \lambda_j \frac{\partial R_{\theta^*}}{\partial \theta_j}\right)R_{\theta^*}^{-\frac{1}{2}} \right)\\
&=& \frac{1}{2n}\left\| R_{\theta^*}^{-\frac{1}{2}}\left( \sum_{i=1}^3 \lambda_i \frac{\partial R_{\theta^*}}{\partial \theta_i}\right) R_{\theta^*}^{-\frac{1}{2}} \right\|_F^2,
\end{eqnarray*}
where we have used $\mbox{Tr}(AA^T)=\|A\|_F^2$ for a square matrix $A$. Furthermore, using $\|A B\|_F^2 \geq \lambda_{\min}^2(A)\|B\|^2$ when $A$ is symmetric, we obtain 
\begin{eqnarray*}
\sum_{i,j=1}^3 \lambda_i \lambda_j(M_{ML})_{i,j}&\geq & \frac{1}{2n}\lambda_{\min}^2\left( R_{\theta^*}^{-\frac{1}{2}}\right)  \left\|\left( \sum_{i=1}^3 \lambda_i \frac{\partial R_{\theta^*}}{\partial \theta_i}\right)\right\|_F^2  \lambda_{\min}^2\left( R_{\theta^*}^{-\frac{1}{2}} \right)\\
&=& \frac{1}{2 \theta_{3,\max}^2 }\left| \left( \sum_{i=1}^3 \lambda_i \frac{\partial R_{\theta^*}}{\partial \theta_i}\right) \right|^2,
\end{eqnarray*}
using Lemma \ref{V.4} and where we recall that $\frac{1}{n}\| \cdot \|_F^2=| \cdot |$, see the beginning of Appendix \ref{s:AB}.
Hence, from Lemma \ref{V.8}, there exists $C_{\min}>0$ such that
\begin{equation}
\liminf_{n\rightarrow \infty}\lambda_{\min}(M_{ML})\geq C_{\min}.
\end{equation}
Moreover, we have, using similar manipulations of norms on matrices above, and using $|\mbox{Tr}(AB)|\leq \|A\|_F \|B\|_F$ from Cauchy-Schwarz,
\begin{eqnarray*}
|(M_{ML})_{i,j}|&=&\left|\frac{1}{2n}\mbox{Tr}\left( R_{\theta^*}^{-1}\frac{\partial R_{\theta^*}}{\partial \theta_i}  R_{\theta^*}^{-1}\frac{\partial R_{\theta^*}}{\partial \theta_j} \right) \right|\\
& \leq & \frac{1}{2n}\left\| R_{\theta^*}^{-1}\frac{\partial R_{\theta^*}}{\partial \theta_i}\right\|_F \left\| R_{\theta^*}^{-1}\frac{\partial R_{\theta^*}}{\partial \theta_j}\right\|_F\\
&\leq &\frac{1}{2}\left\| R_{\theta^*}^{-1}\frac{\partial R_{\theta^*}}{\partial \theta_i}\right\| \left\| R_{\theta^*}^{-1}\frac{\partial R_{\theta^*}}{\partial \theta_j}\right\|\\
&\leq & \frac{1}{2}\|R_{\theta^*}^{-1}\|^2 \left\|\frac{\partial R_{\theta^*}}{\partial \theta_i}\right\| \left\| \frac{\partial R_{\theta^*}}{\partial \theta_j} \right\| \\
&\leq & C_{\max},
\end{eqnarray*}
for some $C_{\max}<\infty$, from Lemmas \ref{V.4} and \ref{A.5}.
Using Gershgorin circle theorem \cite{izvestija/gerschgorin31}, we obtain
\begin{equation}
\limsup_{n\rightarrow \infty}\lambda_{\max}(M_{ML})<+\infty,
\end{equation}
that concludes the proof of \eqref{22}.
\bigskip

By contradiction, let us now assume that
\begin{equation}
\sqrt{n} M_{ML}^{\frac{1}{2}}\left(\widehat{\theta}_{ML}-\theta^* \right)\; \mathrel{\cancel{\overset{\mathcal{L}}{\underset{n\rightarrow +\infty}{\longrightarrow}}}}\mathcal{N}(0,I_3).
\end{equation}
Then, there exists a bounded measurable function $g:\R^3\rightarrow \R$, $\xi>0$ such that, up to extracting a subsequence, we have:
\begin{equation}\label{nonconv}
\left| \E\left[g\left(\sqrt{n} M_{ML}^{\frac{1}{2}}(\widehat{\theta}_{ML} -\theta^*   \right)  \right]-\E(g(U))   \right|  \geq \xi,
\end{equation}
with $U\sim \mathcal{N}(0,I_3)$. The rest of the proof consists in contradicting \eqref{nonconv}.

As $0<C_{\min}\leq\lambda_{\min} (M_{ML}) \leq \lambda_{\max}(M_{ML})\leq C_{\max}$, up to extracting another subsequence, we can assume that:
\begin{equation}\label{Minfty}
M_{ML}\underset{n\rightarrow \infty}{\longrightarrow}M_\infty,
\end{equation}
with $\lambda_{\min}(M_\infty)>0$.
\bigskip

We have:
\begin{equation}
\frac{\partial }{\partial \theta_i}L_\theta=\frac{1}{n}\left( \mbox{Tr}\left(R_\theta^{-1} \frac{\partial R_\theta}{\partial \theta_i}   \right)-y^TR_\theta^{-1}\frac{\partial R_\theta}{\partial \theta_i} R_\theta^{-1}y \right).
\end{equation}
Let $\lambda=(\lambda_1\; \lambda_2\;\lambda_3)^T\in \R^3$. For a fixed $\sigma$, denoting $\sum_{k=1}^3 \lambda
_k R_{\theta^*}^{-\frac{1}{2}}\frac{\partial R_{\theta^*}}{\partial \theta_k}R_{\theta^*}^{-\frac{1}{2}}=P^TDP$ with $P^TP=I_n$ and $D$ diagonal, $z_\sigma=PR_{\theta^*}^{-\frac{1}{2}}y$ (which is a vector of i.i.d. standard Gaussian variables, conditionally to $\Sigma=\sigma$), we have, letting $\phi_1(A),\cdots,\phi_n(A)$ be the eigenvalues of a $n\times n$ symmetric matrix A,
\begin{eqnarray*}
\sum_{k=1}^3\lambda_k \sqrt{n}\frac{\partial }{\partial \theta_k}L_{\theta^*}&=&\frac{1}{\sqrt{n}}\left[\mbox{Tr}\left(\sum_{k=1}^3\lambda_k  R_{\theta^*}^{-1} \frac{\partial R_{\theta^*}}{\partial \theta_k}   \right)-\sum_{i=1}^n\phi_i\left(\sum_{k=1}^3\lambda_k  R_{\theta^*}^{-\frac{1}{2}}\frac{\partial R_{\theta^*}}{\partial \theta_k}R_{\theta^*}^{-\frac{1}{2}} \right)z_{\sigma,i}^2   \right]\\
&=&\frac{1}{\sqrt{n}}\left[\sum_{i=1}^n\phi_i\left(\sum_{k=1}^3\lambda_k  R_{\theta^*}^{-\frac{1}{2}}\frac{\partial R_{\theta^*}}{\partial \theta_k}R_{\theta^*}^{-\frac{1}{2}} \right)(1-z_{\sigma,i}^2)   \right].
\end{eqnarray*}
Hence, we have
\begin{eqnarray*}
\V\left(\sum_{k=1}^3\lambda_k \sqrt{n}\frac{\partial }{\partial \theta_k}L_{\theta^*} \middle|\Sigma \right)&=&\frac{2}{n}\sum_{i=1}^n\phi_i^2\left(\sum_{k=1}^3\lambda_k  R_{\theta^*}^{-\frac{1}{2}}\frac{\partial R_{\theta^*}}{\partial \theta_k}R_{\theta^*}^{-\frac{1}{2}} \right)\\
&=&\frac{2}{n}\sum_{k,l=1}^3\lambda_k \lambda_l \mbox{Tr}\left(\frac{\partial R_{\theta^*}}{\partial \theta_k}R_{\theta^*}^{-1}\frac{\partial R_{\theta^*}}{\partial \theta_l}R_{\theta^*}^{-1}  \right)\\
&=&\lambda^T(4M_{ML})\lambda \underset{n\rightarrow \infty}{\longrightarrow} \lambda^T(4 M_\infty)\lambda.
\end{eqnarray*}
Hence, for almost every $\sigma$, we can apply the Lindeberg-Feller criterion to the variables\\ $\frac{1}{\sqrt{n}}\phi_i\left(\sum_{k=1}^3 \lambda_k R_{\theta^*}^{-\frac{1}{2}}\frac{\partial R_{\theta^*}}{\partial \theta_k}R_{\theta^*}^{-\frac{1}{2}} \right)(1-z_{\sigma,i}^2)$ to show that, conditionally to $\Sigma=\sigma$, $\sqrt{n}\frac{\partial }{\partial \theta}L_{\theta^*}$ converges in distribution to $\mathcal{N}(0,4M_\infty)$.\\

Then, using the dominated convergence theorem on $\Sigma$, we show that:
\begin{equation}
\E\left( \exp\left(i\sum_{k=1}^3 \lambda_k   \sqrt{n}\frac{\partial }{\partial \theta_k}L_{\theta^*} \right)  \right) \underset{n\rightarrow \infty}{\longrightarrow} \exp\left( -\frac{1}{2}\lambda^T (4M_\infty)\lambda  \right).
\end{equation}
Finally,
\begin{equation}\label{deriv_gaussien}
\sqrt{n}\frac{\partial }{\partial \theta}L_{\theta^*} \overset{\mathcal{L}}{\underset{n\rightarrow\infty}{\longrightarrow}}\mathcal{N}(0,4M_\infty).
\end{equation}
Let us now compute
\begin{eqnarray*}
\frac{\partial^2}{\partial \theta_i \partial \theta_j}L_{\theta^*}&=&\frac{1}{n}\mbox{Tr}\left( -R_{\theta^*}^{-1}\frac{\partial R_{\theta^*}}{\partial \theta_i}R_{\theta^*}^{-1}\frac{\partial R_{\theta^*}}{\partial \theta_j}+R_{\theta^*}^{-1}\frac{\partial^2 R_{\theta^*}}{\partial \theta_i \partial \theta_j}\right)\\
&&+\frac{1}{n}y^T\left( 2R_{\theta^*}^{-1}\frac{\partial R_{\theta^*}}{\partial \theta_i}R_{\theta^*}^{-1} \frac{\partial R_{\theta^*}}{\partial \theta_j}R_{\theta^*}^{-1}-R_{\theta^*}^{-1}\frac{\partial^2 R_{\theta^*}}{\partial \theta_i \partial \theta_j}R_{\theta^*}^{-1} \right)y.
\end{eqnarray*}
Thus, we have,
\begin{equation}
\E\left(\frac{\partial^2}{\partial \theta_i \partial \theta_j}L_{\theta^*}\right)\underset{n\rightarrow+\infty}{\longrightarrow}(2M_\infty)_{i,j},
\end{equation}
and, using Lemmas \ref{V.4} and \ref{A.5}, and proceeding similarly as in the proof of Theorem \ref{consis},
\begin{equation}
\V\left( \frac{\partial^2}{\partial \theta_i \partial \theta_j}L_{\theta^*}\middle|\Sigma \right)\underset{n\rightarrow +\infty}{\longrightarrow}0.
\end{equation}
Hence, a.s.
\begin{equation}\label{ddproba}
\frac{\partial^2}{\partial \theta_i \partial \theta_j}L_{\theta^*}\overset{\PP_{|\Sigma}}{\underset{n\rightarrow +\infty}{\longrightarrow}} 2(M_{\infty})_{i,j}.
\end{equation}
Moreover, $\frac{\partial^3}{\partial \theta_i \partial \theta_j \partial \theta_k}L_\theta$ can be written as
\begin{equation}
\frac{1}{n}\mbox{Tr}(A_\theta)+\frac{1}{n}y^TB_\theta y,
\end{equation}
where $A_\theta$ and $B_\theta$ are sums and products of the matrices $R_\theta^{-1}$ or $\frac{\partial ^{|\beta|}}{\partial \theta^\beta}$ with $\beta \in [0:3]^3$. Hence, from Lemmas \ref{V.4} and \ref{A.5}, we have
\begin{equation}\label{ddd}
\sup_{\theta \in \Theta}\left\|\frac{\partial^3}{\partial \theta_i \partial \theta_j \partial \theta_k} L_\theta \right\|=O_{\PP_{|\Sigma}}(1).
\end{equation}
We know that, for $k\in \{1,2,3\}$, from a Taylor expansion,
$$
0=\frac{\partial}{\partial \theta_k}L_{\widehat{\theta}_{ML}}=\frac{\partial}{\partial \theta_k} L_{\theta^*}+\left( \frac{\partial}{\partial \theta}\frac{\partial}{\partial \theta_k}L_{\theta^*}\right)^T(\widehat{\theta}_{ML}-\theta^*)+r_k
$$
with some random $r_k$, such that
$$
|r_k|\leq C \sup_{\theta\in \Theta,i,j}\left| \frac{\partial^3 L_\theta}{\partial \theta_i \partial\theta_j\partial \theta_k} \right| \|\widehat{\theta}_{ML}-\theta^*\|^2,
$$
where $C$ is a finite constant that come from the equivalence of norms for $3\times 3$ matrices.
Hence, from \eqref{ddd}, $r_k=o_{\PP_{|\Sigma}}(|\widehat{\theta}_{ML}-\theta^*|)$. We then have, with $\frac{\partial^2}{\partial \theta^2} L_{\theta^*}$ the $3\times 3$ Hessian matrix of $L_{\theta}$ at $\theta^*$,
$$
-\frac{\partial}{\partial \theta} L_{\theta^*}=\left[\left( \frac{\partial^2}{\partial \theta^2}L_{\theta^*}\right)^T+o_{\PP_{|\Sigma}}(1) \right]\left(  \widehat{\theta}_{ML}-\theta^* \right),
$$
an so
\begin{equation}
\left(  \widehat{\theta}_{ML}-\theta^* \right)=-\left[\left( \frac{\partial}{\partial \theta}\frac{\partial}{\partial \theta}L_{\theta^*}\right)^T+o_{\PP_{|\Sigma}}(1) \right]^{-1}\frac{\partial}{\partial \theta_k} L_{\theta^*}.
\end{equation}
Hence, using Slutsky lemma, \eqref{ddproba} and \eqref{deriv_gaussien}, a.s.
\begin{equation}
\sqrt{n}\left(  \widehat{\theta}_{ML}-\theta^* \right)\overset{\mathcal{L}_{|\Sigma}}{\underset{n\rightarrow +\infty}{\longrightarrow}}\mathcal{N}\left(0,(2M_\infty)^{-1} (4M_\infty)(2M_\infty)^{-1}\right)=\mathcal{N}\left(0,M_\infty^{-1}\right).
\end{equation}
Moreover, using \eqref{Minfty}, we have
\begin{equation}
\sqrt{n} M_{ML}^{\frac{1}{2}}\left(\widehat{\theta}_{ML}-\theta^* \right)\; {\overset{\mathcal{L}_{|\Sigma}}{\underset{n\rightarrow +\infty}{\longrightarrow}}}\mathcal{N}(0,I_3).
\end{equation}
Hence, using dominated convergence theorem on $\Sigma$, we have
\begin{equation}
\sqrt{n} M_{ML}^{\frac{1}{2}}\left(\widehat{\theta}_{ML}-\theta^* \right)\; {\overset{\mathcal{L}}{\underset{n\rightarrow +\infty}{\longrightarrow}}}\mathcal{N}(0,I_3).
\end{equation}
To conclude, we have found a subsequence such that, after extracting,
\begin{equation}\label{ouiconv}
\left| \E\left[g\left(\sqrt{n} M_{ML}^{\frac{1}{2}}(\widehat{\theta}_{ML} -\theta^*  \right)  \right]-\E(g(U))   \right|  \underset{n\to +\infty}{\longrightarrow} 0,
\end{equation}
which is in contradiction with \eqref{nonconv}.
\end{proof}

\subsection{Proof of Theorem \ref{pred}}\label{section_proof_3}
\begin{proof}
Let $\overline{\sigma}_n \in S_{N_n}$. We have:
\begin{equation}
\left|\widehat{Y}_{\widehat{\theta}_{ML}}(\overline{\sigma}_n)-\widehat{Y}_{\theta^*}(\overline{\sigma}_n)\right|\leq  \sup_{\theta \in \Theta}\left\| \frac{\partial }{\partial \theta}\widehat{Y}_\theta (\overline{\sigma}_n)\right\|\;\left\| \widehat{\theta}_{ML}-\theta^* \right\|.
\end{equation}
From Theorem \ref{consis}, it is enough to show that, for $i \in \{1,2,3\}$
\begin{equation}
\left| \sup_{\theta \in \Theta} \frac{\partial }{\partial \theta_i} \hat{Y}_\theta(\overline{\sigma}_n)  \right|=O_\PP(1).
\end{equation}
From a version of Sobolev embedding theorem ($W^{1,4}(\Theta)\hookrightarrow L^\infty(\Theta)$, see Theorem
4.12, part I, case A in \cite{adams_sobolev_2003}), there exists a finite constant $A_\Theta$ depending only on $\Theta$ such that
\begin{eqnarray*}
\sup_{\theta \in \Theta}\left| \frac{\partial}{\partial \theta_i} \widehat{Y}_\theta(\overline{\sigma}_n) \right| & \leq & A_\Theta \int_\Theta \left| \frac{\partial}{\partial \theta_i} \widehat{Y}_\theta(\overline{\sigma}_n) \right|^4 d\theta+A_\Theta \sum_{j=1}^3 \int_\Theta \left| \frac{\partial^2}{\partial \theta_j \partial \theta_i}\widehat{Y}_\theta(\overline{\sigma}_n) \right|^4 d\theta.
\end{eqnarray*}
The rest of the proof consists in showing that these integrals are bounded in probability. We have to compute the derivatives of
$$
\widehat{Y}_\theta(\overline{\sigma}_n)=r_\theta^T(\overline{\sigma}_n)R_\theta^{-1}y
$$
with respect to $\theta$. Thus, we can write these first and second derivatives as weighted sums of $w_\theta^T(\overline{\sigma}_n)W_\theta y$, where $w_\theta(\overline{\sigma}_n)$ is of the form $r_\theta(\overline{\sigma}_n)$ or $\frac{\partial}{\partial \theta_i}r_\theta(\overline{\sigma}_n)$ of $\frac{\partial^2}{\partial \theta_j\theta_i}r_\theta(\overline{\sigma}_n)$ and $W_\theta$ is product of the matrices $R_\theta^{-1}$, $\frac{\partial}{\partial \theta_i}R_\theta$ and $\frac{\partial^2}{\partial \theta_j \theta_i}R_\theta$. It is sufficient to show that
\begin{equation}
\int_\Theta \left|w_\theta^T(\overline{\sigma}_n)W_\theta y\right|^4 d_\theta=O_\PP(1).
\end{equation}
From Fubini-Tonelli Theorem (see \cite{billingsley2013convergence}), we have
$$
\E\left( \int_\Theta \left|w_\theta^T(\overline{\sigma}_n)W_\theta y\right|^4 d_\theta \middle| \Sigma \right)=\int_\Theta \E\left(\left|w_\theta^T(\overline{\sigma}_n)W_\theta y\right|^4\middle| \Sigma \right) d_\theta.
$$
There exists a constant $c$ so that for $X$ a centred Gaussian random variable
$$
\E\left(|X|^4\right)=c\V(X)^{2},
$$
hence
\begin{eqnarray*}
\E\left( \int_\Theta \left|w_\theta^T(\overline{\sigma}_n)W_\theta y\right|^4 d_\theta | \Sigma \right)&=&c\int_\Theta\V \left(w_\theta^T(\overline{\sigma}_n)W_\theta y | \Sigma\right)^2 d_\theta\\
&=&c\int_\Theta\left( w_\theta^T(\overline{\sigma}_n)W_\theta R_\theta^* W_\theta(\overline{\sigma}_n)w_\theta(\overline{\sigma}_n) \right)^2 d_\theta.
\end{eqnarray*}
From Lemma \ref{A.5}, there exists $B<\infty$ such that, a.s.
$$
\sup_{\theta\in \Theta} \|W_\theta R_{\theta^*}W_\theta \|<B.
$$
Thus
\begin{equation}
\E\left( \int_\Theta \left|w_\theta^T(\overline{\sigma}_n)W_\theta y\right|^4 d_\theta \middle| \Sigma\right)\leq B^2 c \int_\Theta \|w_\theta^T(\overline{\sigma}_n)\|^2 d_\theta.
\end{equation}
Finally, for some $\alpha \in [0:2]^3$ such that $|\alpha| \leq 2$, we have
\begin{eqnarray*}
\sup_{\theta \in \Theta} \| w_\theta^T(\overline{\sigma}_n)\|^2&=&\sup_\theta \sum_{i=1}^n \left( \frac{\partial^{|\alpha|}}{\partial \theta^\alpha}K_\theta(\overline{\sigma}_n,\sigma_i)\right)^2.
\end{eqnarray*}
Thus, it suffices to bound this term. Using the proof of Lemma \ref{A.5}, there exists $A<+\infty ,a>0$ such that 
$$
\sup_\theta \left( \frac{\partial^{|\alpha|}}{\partial \theta^\alpha}K_\theta(\overline{\sigma}_n,\sigma_i)\right)^2\leq A \exp\left( -a d(\overline{\sigma}_n,\sigma_i)\right).
$$
Yet, choosing $i^*\in [1:n]$ such that $d(\overline{\sigma}_n,\sigma_{i^*})\leq d(\overline{\sigma}_n,\sigma_i)$ for all $i\in [1:n]$, we have
$$
d(\overline{\sigma}_n,\sigma_i)\geq \frac{1}{2}d(\sigma_i,\sigma_{i^*}).
$$
Thus, we have
\begin{eqnarray*}
\sup_\theta \sum_{i=1}^n \left( \frac{\partial^{|\alpha|}}{\partial \theta^\alpha}K_\theta(\overline{\sigma}_n,\sigma_i)\right)^2 & \leq & A\sum_{i=1}^n \exp\left( -\frac{a}{2}d(\sigma_i,\sigma_{i^*}) \right)\\
& \leq & A\sum_{i=1}^n \exp\left( -\frac{a}{2}|i-i^*|^\beta\right)\\
& \leq & 2 A\sum_{i=0}^{+\infty} \exp\left( -\frac{a}{2}i^\beta\right)\\
& \leq & C.
\end{eqnarray*}
That concludes the proof.
\end{proof}

\end{document}